\documentclass[aos]{imsart}
\usepackage{amsthm,amsmath,amsfonts,amssymb,stackengine,comment}
\usepackage{amsfonts,mathrsfs}
\usepackage{color}
\usepackage{enumerate}
\usepackage[textwidth=8em,textsize=scriptsize]{todonotes}
\usepackage[authoryear, round]{natbib} %% uncomment this for author-year bibliography
\RequirePackage[colorlinks,citecolor=blue,urlcolor=blue]{hyperref}
\usepackage{graphicx}
\usepackage{caption,subcaption,float,booktabs,multirow,lscape}
\usepackage{xr}
\usepackage{comment}

\usepackage{algorithm,algpseudocode} %,algorithm2e}

\theoremstyle{plain}
\newtheorem{theorem}{Theorem}
\newtheorem{lemma}{Lemma}
\newtheorem{corollary}{Corollary}
\newtheorem{definition}{Definition}
\newtheorem{assumption}{Assumption}

\newtheorem{prop}{Proposition}

\numberwithin{equation}{section}

\def\p{\frac{\partial}{\partial\tau}}
\def\pdim{{\rm Pdim}}
\def\vdim{{\rm VCdim}}
\def\s{{\rm sign}}

\begin{document}

\begin{frontmatter}
		\title{Estimation of Non-Crossing Quantile Regression Process
with Deep ReQU Neural Networks}
		%\title{A sample article title with some additional note\thanksref{t1}}
		\runtitle{Deep Non-Crossing Quantile Regression}
		%\thankstext{T1}{A sample additional note to the title.}
%	\begin{comment}
		\begin{aug}
	%%%%%%%%%%%%%%%%%%%%%%%%%%%%%%%%%%%%%%%%%%%%%%
	%%Only one address is permitted per author. %%
	%%Only division, organization and e-mail is %%
	%%included in the address.                  %%
	%%Additional information can be included in %%
	%%the Acknowledgments section if necessary. %%
	%%%%%%%%%%%%%%%%%%%%%%%%%%%%%%%%%%%%%%%%%%%%%%
	\author[A]{\fnms{Guohao} \snm{Shen}\thanks{Guohao Shen and Yuling Jiao contributed equally to this work.}\ead[label=e1,mark]{guohao.shen@polyu.edu.hk}}
	\author[B]{\fnms{Yuling} \snm{Jiao}$^*$
	%{Guohao Shen and Yuling Jiao contributed equally to this work}
	\ead[label=e2,mark]{yulingjiaomath@whu.edu.cn}}
	\author[C]{\fnms{Yuanyuan} \snm{Lin}\ead[label=e3,mark]{ylin@sta.cuhk.edu.hk}}\\
	
	\author[D]{\fnms{Joel L.} \snm{Horowitz}\ead[label=e4,mark]{joel-horowitz@northwestern.edu}}
	\and
	\author[E]{\fnms{Jian} \snm{Huang}\ead[label=e5,mark]{j.huang@polyu.edu.hk}}
	%%%%%%%%%%%%%%%%%%%%%%%%%%%%%%%%%%%%%%%%%%%%%%
	%% Addresses                                %%	%%%%%%%%%%%%%%%%%%%%%%%%%%%%%%%%%%%%%%%%%%%%%%
	\address[A]{Department of Applied Mathematics, The Hong Kong Polytechnic University,  Hong Kong, China.
	\printead{e1}}	
	
	\address[B]{School of Mathematics and Statistics,
	Wuhan University, Wuhan, %Hubei Province,
	China  %, 430072
	\printead{e2}}
	
	\address[C]{Department of Statistics, The Chinese University of Hong Kong,  Hong Kong, China.
	\printead{e3}}	
	
	\address[D]{Department of Economics, Northwestern University, Evanston, IL 60208, USA.
	\printead{e4}}	
	
	\address[E]{Department of Applied Mathematics, The Hong Kong Polytechnic University,
Hong Kong, China
	\printead{e5}}
	\end{aug}

%	\end{comment}

\begin{abstract}
We propose a penalized nonparametric approach to estimating the quantile regression process (QRP) in a nonseparable model using rectifier quadratic unit (ReQU)  activated deep neural networks and introduce a novel penalty function to enforce non-crossing of quantile regression curves.
We establish the non-asymptotic excess risk bounds for the estimated QRP
and derive the mean integrated squared error for the estimated QRP under mild smoothness and regularity conditions.
To establish these non-asymptotic risk and estimation error bounds, we also develop a new error bound for approximating $C^s$ smooth functions with $s >0$ and their derivatives
using ReQU activated neural networks.  This is a new approximation result
for ReQU networks and is of independent interest and may be useful in other problems.
%This new approximation result for ReQU networks.
Our numerical experiments demonstrate that the proposed method is competitive with or outperforms two existing methods, including methods using reproducing kernels and random forests, for nonparametric quantile regression.
\end{abstract}



\end{frontmatter}

%\maketitle

\section{Introduction}
\label{intro}

Consider a nonparametric regression model
\begin{equation}\label{model}
		Y=f_0(X, U),
\end{equation}
where $Y \in \mathbb{R}$ is a response variable, $X \in\mathcal{X} \subset\mathbb{R}^d$ is a $d$-dimensional vector of predictors, $U$ is an unobservable random variable following the uniform distribution on $(0,1)$ and independent of $X$.  The function $f_0:\mathcal{X}\times (0,1)\to \mathbb{R}$ is an unknown regression function, and $f_0$ is increasing in its second argument.
This is a non-separable quantile regression model, in which
the specification  $U\sim {\rm Unif}(0,1)$ is a normalization but not a restrictive assumption
\citep{cin2007, hs2007}. Nonseparable quantile regression models are important in empirical economics (see, e.g., \citet{rhp2017}).
%Blundell, Horowitz, and Parey (2017)).
Based on (\ref{model}), it can be seen that for any $\tau\in(0,1)$, the conditional $\tau$-th quantile $Q_{Y|x}(\tau)$ of $Y$ given $X=x$ is
\begin{equation}\label{qmodel}
	Q_{Y|x}(\tau)=f_0(x,\tau).
\end{equation}
We refer to $f_0 = \{f_0(x, \tau): (x, \tau) \in \mathcal{X} \times (0, 1)\}$ as a quantile regression process (QRP).
%We are interesting in estimating the this process.
A basic property of QRP is that it is nondecreasing with respect to $\tau$ for
any given $x$, often referred to as the non-crossing property.
We propose a novel penalized nonparametric method for estimating $f_0$
on a random discrete grid
of quantile levels in $(0, 1)$ simultaneously, with the penalty designed to ensure the non-crossing property.
% and a new penalty function that promotes the non-crossing property.
%penalized approach to nonparametric estimation of $f_0$ using deep neural networks and %respecting the non-crossing constraint.

%We consider the problem of nonparametric quantile process estimation for the $f_0$ with %non-crossing constraints. To approximate the target quantile function $f_0$, we study the %penalized nonparametric quantile process estimator using deep neural networks.

%Although model (\ref{model}) is seemingly different from the usual quantile regression model %with an additive error term,  it

Quantile regression \citep{koenker1978} is an important method for modeling
the relationship between a response $Y$ and a predictor $X$.
Different from  least squares regression that estimates the conditional
mean of $Y$ given $X$, quantile regression models the conditional
quantiles of $Y$ given $X$, so it  fully describes
the conditional distribution of $Y$ given $X.$
The non-separable model (\ref{model})  can be transformed into a familiar quantile regression model with an additive error.
%Specifically,
For any $\tau \in (0, 1)$, we have $P\{Y-f_0(X,\tau)\le0\}=\tau$ under (\ref{model}). If we define $\epsilon=Y-f_0(X,\tau)$, then model (\ref{model}) becomes
\begin{align}
\label{model2}
Y=g_0(X)+\epsilon,
\end{align}
where $g_0(X)=f_0(X,\tau)$ and $P(\epsilon\le 0\mid X=x) =\tau$ for any $x\in\mathcal{X}$.
An attractive feature of the nonseparable model (\ref{model}) is that it explicitly includes the quantile level as a second argument of $f_0$, which makes it possible to construct a single objective function for estimating the whole quantile process simultaneously.

A general nonseparable quantile regression model that allows a vector random disturbance $U$
 was proposed by \citet{ch2005}. The model (\ref{model}) in the presence of endogeneity was considered by \citet{cin2007}, who gave local identification conditions for the quantile regression function $f_0$  and provided sufficient condition under which a series estimator is consistent. The convergence rate of  the series estimator is unknown.
The relationship between the nonseparable quantile regression model (\ref{model})
and the usual separable quantile regression model
was discussed in \citet{hs2007}. A study of nonseparable bivariate quantile regression
for nonparametric demand estimation using splines under shape constraints was given in \citet{rhp2017}.

There is a large body of literature on separable linear quantile regression in the fixed-dimension setting \citep{koenker1978, koenker2005} and in the high-dimensional settings
 \citep{
 %li2008,
 belloni2011, wang2012, zheng2015globally}.
 %, belloni2019}.
%zheng2018high
Nonparametric estimation of separable quantile regressions has also been considered.  Examples include the methods using shallow neural networks \citep{white1992nonparametric},  smoothing splines \citep{koenker1994, he1994convergence, he1999Qsplines}  and reproducing kernels \citep{takeuchi06a, sangnier2016joint}.
Semiparametric quantile regression has also been considered in the literature \citep{cvc2016, bccf2019}. A popular semiparametric quantile regression model is
\begin{align}
\label{Qmodel}
Q_{Y|x}(\tau) = Z(x)^\top \beta(\tau).
\end{align}
where $Q_{Y|x}(\tau )$
is defined in (\ref{qmodel})
and $Z(x) \in \mathbb{R}^m$ is usually
a series representation of the predictor $x$. The goal is to estimate the
coefficient process $\{\beta(\tau): \tau \in (0, 1)\}$ and derive the asymptotic distribution
of the estimators.  Such results can be used for conducting statistical inference about
$\beta(\tau)$. However, they hinge on the model assumption (\ref{Qmodel}). If this assumption is not satisfied, estimation and inference results based on a misspecified model can be misleading.

%[[[(Some discussion and references on quantile crossing in this paragraph.)
Quantile regression curves satisfy a monotonicity condition. At the population level, it holds that $f_0(x, \tau_2) \ge f_0(x; \tau_1)$ for any $0 <  \tau_1 < \tau_2<1$ and every $x \in \mathcal X$.
However, for an estimator $\hat f$ of $f_0$, there can be values of $x$ for which the quantile
curves cross, that is, $\hat f(x, \tau_2) <  \hat f (x; \tau_1)$  due to finite sample size and sampling variation. Quantile crossing makes it challenging to interpret the estimated quantile curves \citep{he1997}. Therefore, it is desirable to avoid it in practice.
Constrained optimization methods have been used to obtain non-crossing conditional quantile estimates in linear quantile regression and nonparametric quantile regression with a scalar covariate \citep{he1997,brw2010}. A method proposed by \citet{cfg2010} uses sorting to rearrange the original estimated non-monotone quantile curves into monotone curves without crossing. It is also possible to apply the isotonization method for qualitative constraints \citep{mammen1991} to the original estimated quantile curves to obtain quantile curves without crossing.  \citet{brando2022deep} proposed a deep learning algorithm for estimating conditional quantile functions that ensures  quantile monotonicity.
%To ensure the monotonicity,
They first restrict the output of a deep neural network to be positive as the estimator of the derivative of the conditional quantile function, then by using truncated Chebyshev polynomial expansion, the estimated derivative is integrated and the estimator of  conditional quantile function is obtained.

Recently, there has been active research on
nonparametric least squares regression using deep neural networks  \citep{%lee1996efficient,
bauer2019deep,
%schmidt2019deep,
schmidt2020nonparametric,
chen2019nonparametric,
kohler2019estimation, nakada2019adaptive, farrell2021deep, jiao2021deep}.
%	Despite the fact
These studies show that, under appropriate conditions,  least squares regression with neural networks can achieve the optimal rate of convergence up to a logarithmic factor for estimating a conditional mean regression function.
%\cite{stone1982optimal}.
% under certain conditions.
%However, the methods and results are not applicable to the nonseparable quantile regression problem %considered in this paper.
Since the quantile regression problem considered in this work is quite different from the least squares regression, different treatments are needed in the present setting.

%With similar goal, but different method from \citet{brando2022deep},
We propose a penalized nonparametric approach for estimating the nonseparable quantile regression model (\ref{model})
using rectified quadratic unit (ReQU) activated deep neural networks.  We introduce a penalty function for the derivative of the QRP with respect to the quantile level to avoid quantile crossing, which does not require
%posit positive constraint on the derivative and avoids
% the %possibly unstable
numerical integration as in \citet{brando2022deep}.
%}

Our main contributions are as follows.

\begin{enumerate}
\item
We propose a novel loss function that is the expected quantile loss function with respect to a
distribution over $(0, 1)$ for the quantile level, instead of the quantile loss function at a
single quantile level as in the usual quantile regression. An appealing feature of the
proposed loss function is that it can be used to estimate quantile regression functions at an arbitrary number of quantile levels simultaneously. %{\color{red} which reduces training time, evaluation time, and storage requirements compared to the separate estimations at these quantile levels. Besides,  estimating the quantiles jointly greatly alleviates the undesired phenomena of crossing quantiles.}

\item We propose a new penalty function to enforce the non-crossing property for quantile curves at different quantile levels. This is achieved by encouraging the derivative of the quantile regression function $f(x,\tau)$ with respect to $\tau$ to be positive. The use of ReQU activation ensures that the derivative exists. This penalty is easy to implement and  computationally feasible for high-dimensional predictors.

\item We establish  non-asymptotic excess risk bounds for the estimated QRP
and derive the mean integrated squared error for the estimated QRP under
the assumption that the underlying quantile regression process belongs to the $C^s$ class of
functions on $\mathcal{X}\times (0,1).$
 .

\item %We study the approximation properties of  ReQU activated deep neural networks.
 {\color{black} We derive  novel approximation error bounds  for  $C^s$ smooth functions with a positive smoothness index $s$ and their derivatives  using ReQU activated deep neural networks. The error bounds hold not only for the target function, but also its derivatives.} This is a new approximation result for ReQU networks and is of independent interest and may be useful in other problems.

\item We conduct simulation studies to evaluate the finite sample performance of the proposed QRP estimation method and demonstrate that it is competitive or outperforms two existing nonparametric quantile regression methods, including kernel based quantile regression and quantile regression forests.

\end{enumerate}

The remainder of the paper is organized as follows. In Section \ref{sec2} we describe the proposed method for nonparametric estimation of QRP with a novel penalty function for avoiding quantile crossing. In Section \ref{sec3} we state the main results of the paper, including  bounds for  the non-asymptotic excess risk  and
the mean integrated squared error for the proposed QRP estimator.
In Section \ref{stoerr-sec} we derive the stochastic error for the QRP estimator.
In Section \ref{apperr} we establish a novel approximation error bound for approximating $C^s$ smooth functions and their derivatives using ReQU activated neural networks.
Section \ref{computation} describes computational implementation of the proposed method. In Section \ref{sec_num} we conduct numerical studies to evaluate the performance of the QRP estimator. Conclusion remarks are given Section \ref{conclusion}. Proofs and technical details are provided in the Appendix.

\section{Deep quantile regression process estimation with non-crossing constraints}
\label{sec2}
In this section, we describe the proposed approach for estimating the quantile regression process using deep neural networks with a novel penalty for avoiding non-crossing.

%present the basic setup of nonparametric regression.
%We describe the structure of the feedforward neural networks to be used in the
%estimation and define the compositional structure for the target conditional quantile function.

%Problem formulation}
\subsection{The standard quantile regression}
We first recall the standard quantile regression method with the check loss function \citep{koenker1978}.
For a given quantile level $\tau\in(0,1)$,  the quantile check loss function is defined by $$\rho_\tau(x)=x\{\tau-I(x\leq0)\}, \ x \in \mathbb{R}.$$
For any
%a possibly random function
$f:\mathcal{X}\times(0,1)\to\mathbb{R}$
%, let $(X,Y)$ be a random vector independent of $f$. For any
and $\tau\in(0,1)$, the $\tau$-risk of $f$ is defined by
\begin{equation}\label{tau_risk}
	\mathcal{R}^\tau(f)=\mathbb{E}_{X,Y}\{\rho_\tau(Y-f(X,\tau))\}.
\end{equation}
Clearly, by the  model  assumption in  (\ref{qmodel}), for each given $\tau\in(0,1)$, the function $f_0(\cdot, \tau)$ is the minimizer of $\mathcal{R}^\tau(f)$ over all the measurable functions from $\mathcal{X}\times(0,1)\to\mathbb{R}$, i.e., for
%the function $f_0^\tau(\cdot):=f_0(\cdot,\tau):\mathcal{X}\to\mathbb{R}$ is the minimizer of $\mathcal{R}^\tau(f)$ over all measurable functions $f^\tau(\cdot):=f(\cdot,\tau):\mathcal{X}\to\mathbb{R}$, i.e.,
$$f_*^\tau=\arg\min_{f} \mathcal{R}^\tau(f) =\arg\min_{f}\mathbb{E}_{X,Y}\{\rho_\tau(Y-f(X,\tau))\},$$
we have $f_*^\tau\equiv f_0(\cdot, \tau)$ on $\mathcal{X}\times\{\tau\}$. This is the basic identification result for the
standard quantile regression, where only a single conditional quantile function $f_0(\cdot, \tau)$
at a given quantile level $\tau$ is estimated.

\subsection{Expected
check loss with non-crossing constraints}
Our goal is to estimate the whole quantile regression process $\{f_0(\cdot,\tau): \tau \in (0, 1)\}$.
Of course, computationally we can only obtain an estimate of this process on a discrete grid of quantile levels. For this purpose, we propose an objective function and estimate the process
$\{f_0(\cdot,\tau): \tau \in (0, 1)\}$ on a grid of random quantile levels that are increasingly dense as
the sample size $n$ increases. We will achieve this by constructing a randomized objective function as follows.

Let  $\xi$ be a random variable supported on $(0,1)$ with density function $\pi_\xi:(0,1)\to\mathbb{R}^+$. Consider the following randomized version of the check loss function
$$\rho_\xi(x)=x\{\xi-I(x\leq0)\}, \ x \in \mathbb{R}.$$
For a measurable function $f:\mathcal{X}\times(0,1)\to\mathbb{R}$,
%let $(X,Y,\xi)$ be a random vector independent of $f$.
%, and we can
%We
define the risk of $f$ by
\begin{equation}\label{risk}
	\mathcal{R}(f)=\mathbb{E}_{X,Y,\xi}\{\rho_\xi(Y-f(X,\xi))\}=\int_0^1 \mathcal{R}^t(f)\pi_\xi(t) dt.
\end{equation}
%With the defined risk $\mathcal{R}$,
{{\color{black}  At the population level,
let $f^*: \mathcal{X}\times(0,1)\to \mathbb{R}$ be a measurable function  satisfying
$$f^* \in\arg\min_{f} \mathcal{R}(f) =\arg\min_{f}\int_0^1 \mathcal{R}^t(f)\pi_\xi(t) dt.$$
Note that $f^*$ may not be uniquely defined
%in the sense of almost everywhere
if $(X,\xi)$ has zero density on some set $A_0\subseteq \mathcal{X}\times(0,1)$ with positive Lebesgue measure. In this case, $f^*(x,\xi)$ can take any value for $(x,\xi)\in A_0$ since it does not affect  the risk.  Importantly, since the target quantile function $f_0(\cdot, \tau)$
defined in (\ref{model})  minimizes the $\tau$-risk $\mathcal{R}^\tau$ for each $\tau\in(0,1)$, $f_0$ is also the risk minimizer of $\mathcal{R}$ over all measurable functions. Then we have $f_0\equiv f^*$ on $\mathcal{X}\times(0,1)$ almost everywhere given that $(X,\xi)$ has nonzero density on $\mathcal{X}\times(0,1)$ almost everywhere.

 %Note that the definition of
In addition, the risk $\mathcal{R}$ depends on the distribution of $\xi$.
Different distributions of $\xi$ may lead to different  $\mathcal{R}$'s.  However, the target quantile process $f_0$ is still the risk minimizer of $\mathcal{R}$ over all measurable functions, regardless of the distribution  of $\xi$.  We state this property in the following proposition, whose proof is given in the Appendix.

\begin{prop}\label{prop1}
For any random variable $\xi$ supported on $(0,1)$, the target function $f_0$ minimizes the risk $\mathcal{R}(\cdot)$ defined in (\ref{risk}) over all measurable functions, i.e.,
$$f_0\in \arg\min_{f} \mathcal{R}(f)=\arg\min_{f}\mathbb{E}_{X,Y,\xi}\{\rho_{\xi}(Y-f(X,\xi))\}.$$
Furthermore, if $(X,\xi)$ has non zero density almost everywhere on $\mathcal{X}\times(0,1)$ and the probability measure of $(X,\xi)$ is absolutely continuous with respect to Lebesgue measure, then $f_0$ is the unique minimizer of $\mathcal{R}(\cdot)$ over all measurable functions in the sense of almost everywhere(almost surely), i.e.,
$$f_0=\arg\min_{f} \mathcal{R}(f)=\arg\min_{f}\mathbb{E}_{X,Y,\xi}\{\rho_{\xi}(Y-f(X,\xi))\},$$
up to a negligible set with respect to the probability measure of $(X,\eta)$ on $\mathcal{X}\times(0,1)$.
\end{prop}
}

{\color{black} The loss function in (\ref{risk}) can be viewed as a weighted quantile check loss function, where the distribution of $\xi$
%plays a role as the weighting scheme on the
weights the importance of different quantile levels in the estimation.
%To be exact, the density function $\pi_{\xi}(\tau)$ assigns weights to  $\mathcal{R}^\tau(\cdot)$ at quantile level $\tau\in(0,1)$.
%as part of the integrated total risk $\mathcal{R}(\cdot)=\int_0^1\mathcal{R}^\tau(\cdot)\pi_{\xi}(\tau)d\tau$.
Proposition  \ref{prop1} implies that, though different distributions of $\xi$ may result in different estimators with finite samples,  these estimators can be shown to be consistent for the target function $f_0$ under mild conditions.

A natural and simple choice of the distribution of $\xi$ is the uniform distribution over $(0,1)$ with density function  $\pi_\xi(t)\equiv1$ for all $t\in(0,1)$. In this paper we focus on the case that $\xi$ is uniformly distributed on $(0, 1)$, but we emphasize that the theoretical results presented
in Section 5 hold for different choices of the distribution of $\xi$.
%{\color{red}!! By choosing uniform distributed $\xi$, %\citet{tagasovska2019single} proposed ``Simultaneous Quantile Regression" %(SQR), which can be seen as an special case of our estimation.
%In this paper, without loss of generality, we focus on the case that $\xi$ is %uniformly distributed over $(0,1)$ and give theoretical guarantees on such %estimators.  But we wish to emphasize that the theoretical results  presented %in section \ref{sec3} are valid for different choices of the distribution of %$\xi$.!!}
% and the convergence properties of the resulting estimators are generally %invariant regardless of the choice of $\xi$.
}
 %A detailed discussion can be found in Remark \ref{remark1}.

In  applications,  only a  random sample $\{(X_i,Y_i)\}_{i=1}^n$ is available. Also, the integral with respect to $\pi_{\xi}$ in (\ref{risk}) does not have an explicit expression. We can approximate it using a random sample $\{\xi_i\}_{i=1}^n$ from the uniform distribution on $(0,1).$  The empirical risk corresponding to the population risk $R(f)$ in (\ref{risk}) is
\begin{equation}
	\label{er1}
	\mathcal{R}_n(f)=\frac{1}{n} \sum_{i=1}^{n} \rho_{\xi_i}(Y_i-f(X_i,\xi_i)). %\qquad {\color{blue} \mathcal{R}_n(f)=\frac{1}{n} \sum_{i=1}^{n} \int_{0}^1\rho_{\tau}(Y_i-f(X_i,\tau))d\tau }.
\end{equation}
%Our goal is to construct an estimator of $f_0$ within a certain class of functions $\mathcal{F}_n$ %by minimizing the empirical risk, that is,
Let $\mathcal{F}_n$  be a class of
deep neural network (DNN) functions defined on $\mathcal X \times (0, 1).$
We define the
%quantile regression process
QRP estimator as the empirical risk minimizer
\begin{equation}\label{erm}
	\hat{f}_n\in\arg\min_{f\in\mathcal{F}_n}\mathcal{R}_n(f).
\end{equation}
%where $\hat{f}_n$ is called the empirical risk minimizer (ERM). We will give a detailed %description of DNNs below. However, due to the sampling variation, $\hat{f}_n$ may not satisfy %the non-crossing property of the quantile regression process.
%In quantile regression, a common concern is the quantile-crossing problem.  In the estimations,
The estimator $\hat{f}_n$ contains estimates of the quantile curves
$\{\hat f_n(x, \xi_1), \ldots, \hat f_n(x, \xi_n)\}$  at  the quantile levels
$\xi_1, \ldots, \xi_n.$  An attractive feature of this approach is that it estimates all these quantile curves simultaneously.
%For convenience, we refer to $\hat f_n$ as a deep quantile regression process estimator (QRPE).

By the basic properties of quantiles, the underlying quantile regression function $f_0(x, \tau)$ satisfies
\[
f_0(x, \xi_{(1)}) \le \cdots \le f_0(x; \xi_{(n)}), \ x \in \mathcal{X},
\]
where $\xi_{(1)} < \cdots < \xi_{(n)}$ are the ordered values of
$\xi_1, \ldots, \xi_n.$ It is desirable that the estimated quantile function also possess this monotonicity property.
%Given $0<\tau_1<\tau_2<1$, the underlying quantile functions satisfy the monotonicity condition %$f_0(x,\tau_2) \ge f_0(x;\tau_1).$
However, with finite samples and due to sampling variation, the estimated quantile function $\hat{f}_n(x,\tau)$
%and $\hat{f}_n(x,\tau_2)$
may violate this monotonicity property
and cross for some values of $x$,
%that is, $\hat{f}_n(\cdot,\tau_1)$ and $\hat{f}_n(\cdot,\tau_2)$ may cross with each other,
leading to an improper distribution for the predicted response. To avoid quantile crossing, constraints are required in the estimation process.
%generally posited in the quantile estimations.
However, it is not a simple matter to impose monotonicity constraints directly on
regression quantiles.
%In model (\ref{model}),

We use the fact that a regression quantile function $f_0(x, \tau)$ is
nondecreasing in its second argument $\tau$  if
%$f_0$ has partial derivative with respect to its second argument,
 its partial derivative with respect to $\tau$  is non-negative.
 For a quantile regression function $f:\mathcal{X}\times(0,1)\to\mathbb{R}$ with first order partial derivatives, we let $\partial f/\partial \tau$ denote the partial derivative operator for
$f$ with respect to its second argument.
A natural way to impose the monotonicity on
%a quantile function
$f(x, \tau)$ with respect to $\tau$  is to constrain its partial derivative with respect to $\tau$ to be nonnegative. So it is natural to consider ways to constrain the derivative of $f(x;\tau)$ with respect to $\tau$ to be nonnegative.

We propose a penalty function based on  the ReLU activation function,  $\sigma_1(x)=\max\{x,0\} $ $x \in \mathbb{R}$,  as follows,
% to promote the nondecreasing property is to use the penalty function in the following way,
\begin{equation}
\label{penp}\kappa(f)=\mathbb{E}_{X, \xi} \sigma_1\Big(-\frac{\partial}{\partial\tau}f(X,\xi)\Big)
=\mathbb{E}_{X, \xi}\Big[\max\Big\{-\frac{\partial}{\partial\tau}f(X,\xi),0\Big\}\Big].
\end{equation}
Clearly, this penalty function encourages $\frac{\partial}{\partial\tau}f(x,\xi) \ge 0$.
The empirical version of $\kappa$ is
%This reminds us that a straightforward constraint is to restrict our estimator $\hat{f}_n$ to be %monotonic increasing in its second argument. Alternatively, we can penalize the partial derivative %of the estimator.
%For notational simplicity,  for any function $f$ we let
\begin{align}
\label{pene}
%\kappa(f):=\mathbb{E}\Big[\max\{-\frac{\partial}{\partial\tau}f(X,\xi),0\}\Big]
%\quad {\rm and}\quad
\kappa_n(f):=\frac{1}{n}\sum_{i=1}^n\Big[\max\Big\{-\frac{\partial}{\partial\tau}f(X_i,\xi_i),0\Big\}\Big].
\end{align}
%where $\{(X_i,\xi_i)\}_{i=1}^n$ is a random sample.

%For any function $f:\mathcal{X}\times(0,1)\to\mathbb{R}$ with first order partial derivatives, we %let $\partial /\partial \tau$ denote the partial derivative operator of a function with respect to its %second argument.
Based on the above discussion and combining (\ref{risk}) and (\ref{penp}),  we
%define the
propose the following population level penalized risk  for the regression quantile functions
\begin{equation}\label{prisk}
	\mathcal{R}^\lambda(f)=
%=\mathcal R (f) + \lambda \kappa(f)
\mathbb{E}_{X,Y,\xi}\Big[ \rho_{\xi}(Y-f(X,\xi))+\lambda\max\Big\{-\frac{\partial}{\partial\tau}f(X,\xi),0\Big\}\Big],
\end{equation}
where $\lambda \ge 0$ is a tuning parameter.
%And in this paper,
%We assume
Suppose that the partial derivative of the target quantile function $f_0$ with respect to its second argument exists.
It then follows that $\frac{\partial}{\partial\tau} f_0(x,u)\ge0$ for any $(x,u)\in\mathcal{X}\times(0,1)$, and thus $f_0$ is also the risk minimizer of $\mathcal{R}^\lambda(f)$ over all measurable functions on $\mathcal{X}\times(0,1)$.

The empirical risk corresponding to (\ref{prisk})
 for estimating the regression quantile functions is
\begin{equation}\label{perisk}
	\mathcal{R}^\lambda_n(f)=\frac{1}{n} \sum_{i=1}^{n}\Big[ \rho_{\xi_i}(Y_i-f(X_i,\xi_i))+\lambda \max\Big\{-\frac{\partial}{\partial\tau}f(X_i,\xi_i),0\Big\}\Big].
\end{equation}
%where $\lambda\ge 0$ is a tuning parameter.
%Here the penalty term $\max\{-\frac{\partial}{\partial\tau}f(X_i,\xi_i),0\}$ encourages
%$\frac{\partial}{\partial\tau}f(X_i,\xi_i)$ to be nonnegative.
The penalized empirical risk minimizer over a class of functions $\mathcal{F}_n$ is given by
%defined by
\begin{equation}\label{perm}
	\hat{f}^\lambda_n\in\arg\min_{f\in\mathcal{F}_n}\mathcal{R}^\lambda_n(f),
\end{equation}
We refer to $\hat{f}^\lambda_n$ as a penalized deep quantile regression process (DQRP)
estimator. The function class $\mathcal{F}_n$ plays an important role in (\ref{perm}).
%We will take $\mathcal{F}_n$ to be a class of deep neural networks (DNN).
 Below we give a detailed description of $\mathcal{F}_n.$
% below.

%The corresponding penalized risk of the penalized DQRP
%quantile process estimator is
%\begin{equation}\label{prisk}
%	\mathcal{R}^\lambda(f)=\mathbb{E}_{X,Y,\xi}\Big[ %\rho_{\xi}(Y-f(X,\xi))+\lambda\max\{-\frac{\partial}{\partial\tau}f(X,\xi),0\}\Big].
%\end{equation}

\subsection{ReQU activated neural networks}

%In recent years,  deep neural network modeling has achieved impressive successes in many %applications.
%With nonlinear activation functions,
Neural networks with nonlinear activation functions have proven to be a powerful approach for approximating multi-dimensional functions.  Rectified linear unit (ReLU), defined as
$\sigma_1(x)=\max\{x,0\}, x \in \mathbb{R}$,  is one of the most commonly used activation functions due to its attractive properties in computation and optimization.  ReLU neural networks have  received much attention in
statistical machine learning \citep{schmidt2020nonparametric,bauer2019deep,jiao2021deep} and applied mathematics \citep{yarotsky2017error,yarotsky2018optimal,shen2019deep,shen2019nonlinear,lu2021deep}. However, since partial derivatives are involved in our objective function (\ref{perisk}),
% and
it is not sensible to use piecewise linear ReLU networks.

%we will use a smoothed version of ReLU as activation function.

%it could be theoretically and practically inconvenient when derivatives are involved in the loss %function for ReLU network training since ReLU is a piecewise linear function. As in this study, %the partial derivative for the network is needed in the estimation, hence we focus on the deep %neural network with smooth activation functions.

We will use the
Rectified quadratic unit (ReQU) activation, which is smooth and has a continuous first derivative.  The ReQU activation function, denoted as $\sigma_2$,  is simply the squared
ReLU,
\begin{align}
\label{sig2}
	\sigma_2(x)=\sigma_1^2(x)=[\max\{x,0\}]^2,\ x \in \mathbb{R}.
\end{align}
%is a kind of smooth activation functions with continuous first derivatives.
%Therefore,  $\text{ReQU}(x)=[\text{ReLU}(x)]^2$ for $x\in\mathbb{R}$.
With ReQU as activation function, the network will be smooth and differentiable. Thus ReQU activated networks are suitable to the case that the loss function involves derivatives of the networks as in (\ref{perisk}).

We set the function class $\mathcal{F}_n$ in (\ref{perm})  to be $\mathcal{F}_{\mathcal{D},\mathcal{W}, \mathcal{U},\mathcal{S},\mathcal{B},\mathcal{B}^\prime}$, a class of ReQU activated multilayer perceptrons $f: \mathbb{R}^{d+1} \to \mathbb{R} $ with
%parameter $\theta$,
depth $\mathcal{D}$, width $\mathcal{W}$, size $\mathcal{S}$, number of  neurons $\mathcal{U}$ and $f$ satisfying $\Vert f \Vert_\infty\leq\mathcal{B}$ and $\Vert \p f \Vert_\infty\leq\mathcal{B}^\prime$ for some $0 <\mathcal{B}, \mathcal{B}^\prime< \infty$, where
$\Vert f \Vert_\infty$ is the sup-norm of a function $f$.
The network parameters may depend on the sample size $n$, but the dependence is omitted in the notation for simplicity.
%A brief description of the multilayer perceptrons are given below.
The architecture of a multilayer perceptron can be expressed as a composition of a series of functions
\[
f(x)=\mathcal{L}_\mathcal{D}\circ\sigma_2\circ\mathcal{L}_{\mathcal{D}-1}
\circ\sigma_2\circ\cdots\circ\sigma_2\circ\mathcal{L}_{1}\circ\sigma_2\circ\mathcal{L}_0(x),\  x\in \mathbb{R}^{p_0},
\]
where $p_0=d+1$,  $\sigma_2$
%(x)=\{\max(0, x)\}^2$
is the rectified quadratic unit (ReQU) activation function defined in (\ref{sig2})
 (
%defined for
operating on  $x$ component-wise if $x$ is a vector), and $\mathcal{L}_i$'s are linear functions
$$\mathcal{L}_{i}(x)=W_ix+b_i,\ x \in \mathbb{R}^{p_i},  i=0,1,\ldots,\mathcal{D},$$
with $W_i\in\mathbb{R}^{p_{i+1}\times p_i}$  a weight matrix and $b_i\in\mathbb{R}^{p_{i+1}}$  a bias vector.
% in the $i$-th linear transformation $\mathcal{L}_i$.
 Here $p_i$ is the width (the number of neurons or computational units) of the $i$-th layer.
The input data consisting of predictor values $X$ is the first layer and the output
is the last layer. Such a network $f$ has $\mathcal{D}$ hidden layers and $(\mathcal{D}+2)$ layers in total. We use a $(\mathcal{D}+2)$-vector $(p_0,p_1,\ldots,p_\mathcal{D},p_{\mathcal{D}+1})^\top$ to describe the width of each layer; particularly,   $p_0=d+1$ is the dimension of the input $(X,\xi)$ and $p_{\mathcal{D}+1}=1$ is the dimension of the response $Y$ in model (\ref{qmodel}). The width $\mathcal{W}$ is defined as the maximum width of hidden layers, i.e.,
$\mathcal{W}=\max\{p_1,...,p_\mathcal{D}\}$; 	the size $\mathcal{S}$ is defined as the total number of parameters in the network $f_\phi$, i.e., $\mathcal{S}=\sum_{i=0}^\mathcal{D}\{p_{i+1}\times(p_i+1)\}$; 	
the number of neurons $\mathcal{U}$ is defined as the number of computational units in hidden layers, i.e., $\mathcal{U}=\sum_{i=1}^\mathcal{D} p_i$.  Note that the neurons in consecutive layers are connected to each other via linear transformation matrices $W_i$,  $i=0,1,\ldots,\mathcal{D}$.
%For an feed-forward neural network class $\mathcal{F}_{\mathcal{D},\mathcal{U},\mathcal{W},\mathcal{S},\mathcal{B}}$,
%its parameters satisfy the simple relationship
%	\[
%	\max\{\mathcal{W},\mathcal{D}\}\leq\mathcal{S}\leq\mathcal{W}(d+1)
%	+(\mathcal{W}^2+\mathcal{W})(\mathcal{D}-1)+\mathcal{W}+1
%	=O(\mathcal{W}^2\mathcal{D}).
%	\]
	%\end{equation}
	%	For general feedforward neural networks, the layers need not be connected in a chain, and skip connections are allowed. For example, some neurons in layer  $i$ can be directly connected to neurons in  layer $(i + 2)$  or higher. In such structured networks,  a neuron  is said to be in layer $i=1,2,\ldots,\mathcal{D}$ if it has a predecessor in layer $i-1$ and no predecessor in any layer $i^\prime\geq i$.  ??Given a citation for reference.??
	
%In this paper,
The network parameters can depend on the sample size $n$, but the dependence is suppressed for notational simplicity, that is, $\mathcal{S}=\mathcal{S}_n$, $\mathcal{U}=\mathcal{U}_n$, $\mathcal{D}=\mathcal{D}_n$, $\mathcal{W}=\mathcal{W}_n$, $\mathcal{B}=\mathcal{B}_n$ and $\mathcal{B}^\prime=\mathcal{B}_n^\prime$.
	This makes it possible to approximate the target regression function by neural networks as $n$ increases.
	%For notational simplicity, we omit the subscript below.
	The approximation and excess error rates will be determined in part
	by how these network parameters depend on $n$.

%\section{Convergence analysis}

%\subsection{Excess risk decomposition}

\section{Main results} %Non-asymptotic Error bounds}
\label{sec3}
In this section, we state our main results on the bounds for the excess risk and estimation error of the penalized DQRP estimator.
%give nonasymptotic upper bounds for our interested quantity,
The excess risk
%$\mathcal{R}(\hat{f}_n^\lambda)-\mathcal{R}(f_0)$
of the penalized \text{DQRP} estimator
is defined as
%. Here the excess risk of the DQRP $\hat{f}^\lambda_n$ defined in (\ref{perm}) is
\begin{align*} \mathcal{R}(\hat{f}_n^\lambda)-\mathcal{R}(f_0)&
=\mathbb{E}_{X,Y,\xi}\{\rho_{\xi}(Y-\hat{f}^\lambda_n(X,\xi))-\rho_{\xi}(Y-f_0(X,\xi))\},
\end{align*}
where $(X,Y,\xi)$ is  an independent copy of the random sample $\{(X_i,Y_i, \xi_i)\}_{i=1}^n$.
%with the same distribution as and independent of
%sample
%$\{(X_i,Y_i,\xi_i)\}$.
%with the same distribution
%(or the $\hat{f}^\lambda_n$)
%and follows the same distribution with $(X_i,Y_i,\xi_i)$.

We first state the following basic lemma for bounding the excess risk.
\begin{lemma}[Excess risk decomposition]\label{decom}
	For the penalized empirical risk minimizer $\hat{f}^\lambda_n$ defined in (\ref{perm}), its excess risk can be upper bounded by
	\begin{align*}
		\mathbb{E}\Big\{\mathcal{R}(\hat{f}_n^\lambda)-\mathcal{R}(f_0)\Big\}\le&	\mathbb{E}\Big\{\mathcal{R}^\lambda(\hat{f}_n^\lambda)-\mathcal{R}^\lambda(f_0)\Big\}\\
		\le&\mathbb{E}\Big\{\mathcal{R}^\lambda(\hat{f}^\lambda_n)
-2\mathcal{R}^\lambda_n(\hat{f}^\lambda_n)+\mathcal{R}^\lambda(f_0)\Big\}
+2\inf_{f\in\mathcal{F}_n}\Big[\mathcal{R}^\lambda(f)-\mathcal{R}^\lambda(f_0)\Big].
	\end{align*}
\end{lemma}

Therefore, the bound for excess risk can be decomposed into two parts:
 the stochastic error $\mathbb{E}\{\mathcal{R}^\lambda(\hat{f}^\lambda_n)-2\mathcal{R}^\lambda_n(\hat{f}^\lambda_n)+\mathcal{R}^\lambda(f_0)\}$ and the approximation error $\inf_{f\in\mathcal{F}_n}[\mathcal{R}^\lambda(f)-\mathcal{R}^\lambda(f_0)]$.
 %It is interesting to note that the upper bound can no longer depend on the penalized % %\text{DQRP} estimator itself, but on the function class $\mathcal{F}_n$, the loss function %$\rho_\xi+\lambda\kappa$ and the random sample $\{(X_i,Y_i,\xi_i)\}_{i=1}^n$.
Once bounds for the stochastic error and approximation error are available, we can immediately obtain an upper bound for the excess risk of the penalized DQRP estimator $\hat{f}^\lambda_n$.

\subsection{Non-asymptotic excess risk bounds}
We first state the conditions needed for establishing the excess risk bounds.
\begin{definition}[Multivariate differentiability classes $C^s$]
\label{defCs}
	A function $f: \mathbb{B}\subset \mathbb{R}^{d}\to\mathbb{R}$  defined on a subset $\mathbb{B}$ of $\mathbb{R}^d$ is said to be in class $C^s(\mathbb{B})$ on $\mathbb{B}$ for a positive integer $s$, if all partial derivatives
	$$D^\alpha f:=\frac{\partial^\alpha }{\partial x_1^{\alpha_1}\partial x_2^{\alpha_2}\cdots\partial x_d^{\alpha_d}}f$$
	exist and are continuous on $\mathbb{B}$ for all non-negative integers $\alpha_1,\alpha_2,\ldots,\alpha_d$ such that $\alpha:=\alpha_1+\alpha_2+\cdots+\alpha_d\le s$. In addition, we define the norm of $f$ over  $\mathbb{B}$ by
	$$\Vert f\Vert_{C^s} :=\sum_{\vert\alpha\vert_1\le s}\sup_{\mathbb{B}}\vert D^\alpha f\vert,$$
	where $\vert\alpha\vert_1:=\sum_{i=1}^d\alpha_i$  for any vector $\alpha=(\alpha_1,\alpha_2,\ldots,\alpha_d)\in\mathbb{R}^d$.
\end{definition}

We make the following smoothness assumption on the target regression quantile function $f_0$.
\begin{assumption}\label{assump0}
The target quantile regression function $f_0: \mathcal X \times (0, 1) \to \mathbb R$ defined in (\ref{qmodel}) belongs to $C^s(\mathcal X\times (0, 1))$ for $s\in\mathbb{N}^+$, where $\mathbb{N}^+$ is the set of positive integers.
\end{assumption}

Let $\mathcal{F}_n^\prime:=\{\p f:f\in\mathcal{F}_n\}$ denote the function class induced by $\mathcal{F}_n$.

For a class $\mathcal{F}$ of functions: $\mathcal{X}\to \mathbb{R}$, its pseudo dimension, denoted by $\text{Pdim}(\mathcal{F}),$  is the largest integer $m$ for which there exists $(x_1,\ldots,x_m,y_1,\ldots,y_m)\in\mathcal{X}^m\times\mathbb{R}^m$ such that for any $(b_1,\ldots,b_m)\in\{0,1\}^m$ there exists $f\in\mathcal{F}$ such that $\forall i:f(x_i)>y_i\iff b_i=1$ \citep{anthony1999, bartlett2019nearly}.

{\color{black}
\begin{theorem}[Non-asymptotic excess risk bounds]\label{non-asymp}
Let Assumption \ref{assump0} hold.
For any $N\in\mathbb{N}^+$, let $\mathcal{F}_n:=\mathcal{F}_{\mathcal{D},\mathcal{W},\mathcal{U},\mathcal{S},\mathcal{B},\mathcal{B}^\prime}$ be the ReQU activated neural networks $f:\mathcal{X}\times(0,1)\to\mathbb{R}$ with depth $\mathcal{D}\le2N-1$, width $\mathcal{W}\le12N^d$, the number of neurons $\mathcal{U}\le15N^{d+1}$, the number of parameters $\mathcal{S}\le24N^{d+1}$.
Suppose that
%and satisfying
$\mathcal{B}\ge\Vert f_0\Vert_{C^0}$ and $\mathcal{B}^\prime\ge\Vert f_0\Vert_{C^1}$. Then for $n\ge\max\{\pdim(\mathcal{F}_n),\pdim(\mathcal{F}^\prime_n)\}$, the excess risk of the penalized DQRP  estimator $\hat{f}^\lambda_n$ defined in (\ref{perm}) satisfies
	\begin{align}\label{non-asympb2}
%\notag
	\mathbb{E}\{\mathcal{R}(\hat{f}^\lambda_n)-\mathcal{R}(f_0)\}
%&
\le C_0(\mathcal{B}+\lambda\mathcal{B}^\prime)\frac{\log n}{n}(d+1)N^{d+3}
 %\\ & \qquad\qquad
+C_{s,d,\mathcal{X}}(1+\lambda) \Vert f_0\Vert_{C^s}N^{-(s-1)},
	\end{align}
	where $C_0>0$ is a universal constant and $C_{s,d,\mathcal{X}}$  is a positive constant depending only on $d,s$ and the diameter of the support $\mathcal{X}\times(0,1)$.
\end{theorem}

By Theorem \ref{non-asymp}, for each fixed sample size $n$, one can choose a proper positive integer $N$ based on $n$ to construct such a ReQU network to achieve the upper bound (\ref{non-asympb2}). To achieve the optimal convergence rate with respect to the sample size $n$, we set $N=\lfloor n^{1/(d+s+2)}\rfloor$ and $\lambda=\log n.$  Then from (\ref{non-asympb2}) we obtain an upper bound
\begin{align*}
	\mathbb{E}\{\mathcal{R}(\hat{f}^\lambda_n)-\mathcal{R}(f_0)\}&\le C
%(\mathcal{B},\mathcal{B}^\prime,s,d,\mathcal{X}, \Vert f_0\Vert_{C^s})
(\log n)^{2} n^{-\frac{s-1}{d+s+2}},
\end{align*}
where $C>0$
%(\mathcal{B},\mathcal{B}^\prime,s,d,\mathcal{X},\Vert f_0\Vert_{C^s})>0$
is a constant depending only on $\mathcal{B},\mathcal{B}^\prime,s,d,\mathcal{X}$ and $ \Vert f_0\Vert_{C^s}$.
The convergence rate  is $(\log n)^{2} n^{-(s-1)/(d+s+2)}.$
% which is nearly optimal in terms of the nonparametric regressions
 The term $(s-1)$ in the exponent is due to the approximation of the first-order partial derivative of the target function. Of course, the smoothness of the target function $f_0$ is  unknown in practice and how to determine the smoothness of an unknown function is
 a difficult problem.
  %an important but nontrivial problem.

\subsection{Non-asymptotic estimation error bound}

The empirical risk minimization quantile estimator typically results in an estimator $\hat{f}^\lambda_n$ whose risk $\mathcal{R}(\hat{f}^\lambda_n)$ is close to the optimal risk $\mathcal{R}(f_0)$ in expectation or with high probability. However, small excess risk in general only implies in a weak sense that the penalized empirical risk minimizer $\hat{f}^\lambda_n$  is close to the target $f_0$ (Remark 3.18 \cite{steinwart2007compare}).
We bridge the gap between the excess risk and the mean integrated error of the estimated quantile function. To this end, we need the following condition on the conditional distribution of $Y$ given $X$.
\begin{assumption}\label{assump1}
	There exist constants $K>0$ and $k>0$ such that for any $\vert \delta\vert\le K$,
	$$\vert P_{Y|X}(f_0(x,\tau)+\delta\mid x)-P_{Y|X}(f_0(x,\tau)\mid x)\vert\ge k\vert \delta\vert,$$
	for all $\tau\in(0,1)$ and $x\in\mathcal{X}$ up to a negligible set, where $ P_{Y|X}(\cdot\mid x)$ denotes the conditional distribution function of $Y$ given $X=x$.
\end{assumption}
%\begin{remark}
Assumption \ref{assump1} is a mild condition
	on the distribution of $Y$ in the sense that,
% as this should hold for most realistic sequences of distributions.
%For example,
if $Y$ has a density that is bounded away from zero on any compact interval,  then  Assumption \ref{assump1} will hold. In particular, no moment assumptions are made on the distribution of $Y$.
	Similar conditions are assumed by \cite{padilla2021adaptive} in studying nonparametric quantile trend filtering for a single quantile level $\tau\in(0,1)$. This condition is weaker than Condition 2 in \cite{he1994convergence} where the density function of response is required to be lower bounded every where by some positive constant. Assumption \ref{assump1} is also weaker than Condition D.1 in \cite{belloni2011}, which requires the conditional density of $Y$ given $X=x$ to be continuously differentiable and bounded away from zero uniformly for all quantiles in  $(0,1)$ and all $x$ in the support $\mathcal{X}$. % As discussed in \cite{padilla2021adaptive},

%\end{remark}

Under Assumption \ref{assump1}, the following self-calibration condition can be established as stated below. This will lead to a bound on the mean integrated error of the estimated quantile process based on a bound for the excess risk.

\begin{lemma}[Self-calibration]\label{calib}
	Suppose that Assumption \ref{assump1} holds. For any $f:\mathcal{X}\times(0,1)\to\mathbb{R}$, denote
$$\Delta^2(f,f_0) =\mathbb{E}[ \min\{\vert f(X,\xi)-f_0(X,\xi)\vert,\vert f(X,\xi)-f_0(X,\xi)\vert^2\}],$$
where $X$ is the predictor vector and  $\xi$ is a uniform random variable on (0,1)  independent of $X$. Then we have
	$$\Delta^2(f,f_0)\le c_{K,k} \{\mathcal{R}(f)-\mathcal{R}(f_0)\},$$
	for any $f:\mathcal{X}\times(0,1)\to\mathbb{R}$, where $c_{K,k}=\max\{2/k,4/(Kk)\}$ and $K,k>0$ are defined in Assumption \ref{assump1}.
\end{lemma}

\begin{theorem}\label{thm2}
Suppose Assumptions \ref{assump0} and \ref{assump1} hold.
For any $N\in\mathbb{N}^+$, let $\mathcal{F}_n:=\mathcal{F}_{\mathcal{D},\mathcal{W},\mathcal{U},\mathcal{S},\mathcal{B},\mathcal{B}^\prime}$ be the class of ReQU activated neural networks $f:\mathcal{X}\times(0,1)\to\mathbb{R}$ with depth $\mathcal{D}\le2N-1$, width $\mathcal{W}\le12N^d$, number of neurons $\mathcal{U}\le15N^{d+1}$, number of parameters $\mathcal{S}\le24N^{d+1}$ and satisfying $\mathcal{B}\ge\Vert f_0\Vert_{C^0}$ and $\mathcal{B}^\prime\ge\Vert f_0\Vert_{C^1}$. Then for $n\ge\max\{\pdim(\mathcal{F}_n),\pdim(\mathcal{F}^\prime_n)\}$, the mean integrated error of the penalized DQRP estimator $\hat{f}^\lambda_n$ defined in (\ref{perm}) satisfies
	\begin{align}\label{non-asympb3}
%\notag
		\mathbb{E}\{\Delta^2(\hat{f}^\lambda_n,f_0)\}\le
%&\le
c_{K,k}\Big[C_0(\mathcal{B}+\lambda\mathcal{B}^\prime)(d+1)N^{d+3}\frac{\log n}{n}
%\\ &\qquad\qquad
+C_{s,d,\mathcal{X}}
%(s,d,\mathcal{X})
(1+\lambda) \Vert f_0\Vert_{C^s}N^{-(s-1)}\Big],
	\end{align}
	where $C_0>0$ is a universal constant, $c_{K,k}$ is defined in Lemma \ref{calib} and $C_{s,d,\mathcal{X}}$  is a positive constant depending only on $d,s$ and the diameter of the support $\mathcal{X}\times(0,1)$.

By  setting $N=\lfloor n^{1/\{(d+s+2)\}}\rfloor$ and $\lambda=\log n$ in (\ref{non-asympb3}), we obtain an upper bound
	\begin{align*}
		\mathbb{E}\{\Delta^2(\hat{f}^\lambda_n,f_0)\}&\le C_1
%(\mathcal{B},\mathcal{B}^\prime,s,d,K,k\mathcal{X}, \Vert f_0\Vert_{C^s})
(\log n)^{2} n^{-\frac{s-1}{d+s+2}},
	\end{align*}
	where $C_1>0$
%(\mathcal{B},\mathcal{B}^\prime,s,d,K,k\mathcal{X},\Vert f_0\Vert_{C^s})>0$
is a constant depending only on $\mathcal{B},\mathcal{B}^\prime,s,d,K,k,\mathcal{X}$ and $ \Vert f_0\Vert_{C^s}$.
\end{theorem}

{\color{black} Without the crossing penalty in the objective function, no estimation for the derivative function is needed, thus
the convergence rate can be improved.
%The convergence rate can be improved for a bit if we do not posit crossing penalty in the objective function, since without crossing penalty there is no derivative getting involved with the estimation.
In this case, ReLU activated or other neural networks can be used to estimate the quantile regression process. For instance, \citet{shen2021deep} showed that nonparametric quantile regression based on  ReLU neural networks attains a convergence rate of $n^{-s/(d+s)}$ up to a  logarithmic factor. This rate is slightly faster than the rate $n^{-(s-1)/(d+s+2)}$  in Theorem \ref{thm2} when
%a crossing penalty is posited and
estimation of the derivative function is involved.
}

\section{Stochastic error}
\label{stoerr-sec}
Now we derive non-asymptotic upper bound for the stochastic error given in Lemma \ref{decom}. %$\mathbb{E}\{\mathcal{R}^\lambda(\hat{f}^\lambda_n)
%-2\mathcal{R}^\lambda_n(\hat{f}^\lambda_n)+\mathcal{R}^\lambda(f_0)\}$.
%The stochastic error
%basically
%depends on the size and distribution of the sample, the complexity of the function class %$\mathcal{F}_n$, the loss function $\rho_\xi$ used in the definition of $\mathcal{R}$ and the %penalty function $\kappa$ in (\ref{penp}).
The main difficulty here is that the term
%However, it can be seen that the quantity
\begin{align*} \mathcal{R}^\lambda(\hat{f}^\lambda_n)-2\mathcal{R}^\lambda_n(\hat{f}^\lambda_n)
+\mathcal{R}^\lambda(f_0)=\mathcal{R}(\hat{f}^\lambda_n)-2\mathcal{R}_n(\hat{f}^\lambda_n)+\mathcal{R}(f_0)+\lambda\kappa(\hat{f}^\lambda_n)-2\lambda\kappa_n(\hat{f}^\lambda_n)
\end{align*}
involves the partial derivatives of the neural network functions in $\mathcal{F}_n.$
% which can be another function class.
%and it is troublesome for us
%This makes it difficult to deal with the stochastic error.
%directly via empirical process theories.
%To
%resolve
%overcome this difficulty,
%problem,
Thus we also need to study the properties, especially,  the complexity of the partial derivatives of the neural network
functions in $\mathcal{F}_n$.
%In this paper, we
Let
$$\mathcal{F}_n^\prime:=\Big\{\p f(x,\tau):f\in\mathcal{F}_n, (x, \tau) \in \mathcal{X} \times (0, 1) \Big \}.
$$
%be the class of such partial derivative functions.
%induced by $\mathcal{F}_n$.
{\color{black} Note that the partial derivative operator is not a Lipschitz contraction operator, thus Talagrand's lemma \citep{ledoux1991probability} cannot be used to link the Rademacher complexity of $\mathcal{F}_n$ and $\mathcal{F}_n^\prime$, and to obtain an upper bound of the Rademacher complexity of $\mathcal{F}_n^\prime$. In view of this, we consider a new class of neural network functions
%$\mathcal{F}_n^\prime$ that
whose complexity is convenient to compute. Then the complexity of $\mathcal{F}_n^\prime$ can be upper bounded by the complexity of such a class of neural network functions.

The following lemma shows that $\mathcal{F}_n^\prime$
%can be computed by
is contained in the class of neural network functions with ReLU and ReQU mixed-activated multilayer perceptrons.
In the following,  we refer to the neural networks activated by the ReLU or the ReQU  as ReLU-ReQU activated neural networks, i.e., the activation functions in each layer of ReLU-ReQU network can be ReLU or ReQU and the activation functions in different layers can be different.}

\begin{lemma}[Network for partial derivative]\label{lemmader}
	Let $\mathcal{F}_n:=\mathcal{F}_{\mathcal{D},\mathcal{W}, \mathcal{U},\mathcal{S},\mathcal{B},\mathcal{B}^\prime}$ be a class of ReQU activated neural networks $f:\mathcal{X}\times(0,1)\to\mathbb{R}$ with depth (number of hidden layer) $\mathcal{D}$, width (maximum width of hidden layer) $\mathcal{W}$, number of neurons $\mathcal{U}$, number of parameters (weights and bias) $\mathcal{S}$ and $f$ satisfying $\Vert f \Vert_\infty\leq\mathcal{B}$ and $\Vert \p f \Vert_\infty\leq\mathcal{B}^\prime$. Then for any $f\in\mathcal{F}_n$, the partial derivative $\frac{\partial}{\partial\tau} f$ can be implemented by a ReLU-ReQU activated multilayer perceptron with depth $3\mathcal{D}+3$, width $10\mathcal{W}$, number of neurons $17\mathcal{U}$, number of parameters $23\mathcal{S}$ and bound $\mathcal{B}^\prime$.
\end{lemma}

By Lemma \ref{lemmader}, the partial derivative of a function in $\mathcal{F}_n$  can be implemented by a function in $\mathcal{F}^\prime_n$.
%The complexity of the class of partial derivatives of functions in $\mathcal{F}_n$ can be %bounded by a proper network class $\mathcal{F}^\prime_n$.
Consequently, for $\kappa$ and $\kappa_n$ given in (\ref{penp}) and (\ref{pene}),
\begin{align*}
	\sup_{f\in\mathcal{F}_n}\vert \kappa(f)-\kappa_n(f)\vert
 \le  &\sup_{f^\prime\in\mathcal{F}^\prime_n}\vert \tilde{\kappa}(f^\prime)-\tilde{\kappa}_n(f^\prime)\vert,%\sup_{f\in\mathcal{F}_n}\Big\vert \mathbb{E}[\max\{-\p f(X,\xi),0\}]-\frac{1}{n}\sum_{i=1}^n[\max\{-\p f(X_i,\xi_i),0\}]\right\vert\\
	%=&\sup_{f^\prime\in\mathcal{F}^\prime_n}\Big\vert \mathbb{E}[\max\{-f^\prime(X,\xi),0\}]-\frac{1}{n}\sum_{i=1}^n[\max\{-f^\prime(X_i,\xi_i),0\}]\right\vert
\end{align*}
where $\tilde{\kappa}(f)=\mathbb{E}[\max\{-f(X,\xi),0\}]$ and $\tilde{\kappa}_n(f)=\sum_{i=1}^n[\max\{-f(X_i,\xi_i),0\}]/n$. Note that $\tilde{\kappa}$ and $\tilde{\kappa}_n$ are both 1-Lipschitz in $f$, thus
%a straightforward
an upper bound for $\sup_{f^\prime\in\mathcal{F}^\prime_n}\vert \tilde{\kappa}(f^\prime)-\tilde{\kappa}_n(f^\prime)\vert$ can be derived once the complexity of $\mathcal{F}^\prime_n$ is known. The complexity of a function class can be measured in several ways, including Rademacher complexity, covering number, VC dimension and Pseudo dimension. These measures depict the complexity of a function class differently but are closely related to each other in many ways (a brief description of these measures can be found in Appendix \ref{apdx_b}). Next, we give an upper bound on the Pseudo dimension of the function class $\mathcal{F}^\prime_n$, which facilities our derivation of the upper bound for the stochastic error.

\begin{lemma}[Pseudo dimension of ReLU-ReQU multilayer perceptrons]\label{lemmapdim}
	Let $\mathcal{F}$ be a function class implemented by ReLU-ReQU activated multilayer perceptrons with depth no more than $\tilde{\mathcal{D}}$, width no more than $\tilde{\mathcal{W}}$, number of neurons (nodes) no more than $\tilde{\mathcal{U}}$ and size or number of parameters (weights and bias) no more than $\tilde{\mathcal{S}}$. Then the Pseudo dimension of $\mathcal{F}$ satisfies
%	\begin{align*}		%\pdim(\mathcal{F})\le7\mathcal{D}\mathcal{S}(\mathcal{D}+\log_2\mathcal{U}),\qquad\pdim(\mathcal{F})\le22\mathcal{U}\mathcal{S}.
%	\end{align*}
%As result,
	\begin{align*}	\pdim(\mathcal{F})\le\min\{7\tilde{\mathcal{D}}\tilde{\mathcal{S}}(\tilde{\mathcal{D}}+\log_2\tilde{\mathcal{U}}),22\tilde{\mathcal{U}}\tilde{\mathcal{S}}\}.
\end{align*}
\end{lemma}
%{\color{red} [[[Better to use other notations to avoid any confusion with the notations in Lemma 3. --YY.]]]}
%In Lemma \ref{lemmapdim}, we obtained two upper bounds for ReLU-ReQU multilayer %perceptrons, where one is in terms of the network depth $\mathcal{D}$, network size %$\mathcal{S}$ and number of neurons $\mathcal{U}$ and another is in terms of the network size %$\mathcal{S}$ and number of neurons $\mathcal{U}$. For multilayer perceptrons with different %shapes, these two bounds for its Pseudo dimension

{\color{black}
\begin{theorem}[Stochastic error bound]\label{stoerr}
	Let $\mathcal{F}_n=\mathcal{F}_{\mathcal{D},\mathcal{W},\mathcal{U},\mathcal{S},\mathcal{B},\mathcal{B}^\prime}$ be the ReQU activated multilayer perceptron and let $\mathcal{F}^\prime_n=%\mathcal{F}^\prime_{\mathcal{D}^\prime,\mathcal{W}^\prime,\mathcal{U}^\prime,\mathcal{S}^\prime,\mathcal{B}^\prime}=
	\{\frac{\partial}{\partial\tau} f:f\in\mathcal{F}_{n}\}$ denote the class of first order partial derivatives. Then for $n\ge\max\{\pdim(\mathcal{F}_n),\pdim(\mathcal{F}^\prime_n)\}$, the stochastic error satisfies
 \begin{align}\notag &\mathbb{E}\{\mathcal{R}^\lambda(\hat{f}^\lambda_n)-2\mathcal{R}^\lambda_n(\hat{f}^\lambda_n)+\mathcal{R}^\lambda(f_0)\}\le c_0\big\{\mathcal{B}\pdim(\mathcal{F}_n)+
 \lambda\mathcal{B}^\prime\pdim(\mathcal{F}_n^\prime)\big\}\frac{\log(n)}{n},
\end{align}
	for some universal constant $c_0>0$. Also,
 \begin{align*}\notag &\mathbb{E}\{\mathcal{R}^\lambda(\hat{f}^\lambda_n)-2\mathcal{R}^\lambda_n(\hat{f}^\lambda_n)+\mathcal{R}^\lambda(f_0)\}\\
	&\qquad\qquad\le c_1\big(\mathcal{B}+\lambda\mathcal{B}^\prime\big)
\min\{5796\mathcal{D}\mathcal{S}(\mathcal{D}+\log_2\mathcal{U}),
8602\mathcal{U}\mathcal{S}\}\frac{\log(n)}{n},
\end{align*}
for some universal constant $c_1>0$.
\end{theorem}
}
The proofs of Lemma \ref{lemmapdim} and Theorem \ref{stoerr} are given in the Appendix.

\section{Approximation error}
\label{apperr}

In this section, we give an upper bound on the approximation error of the ReQU network for approximating functions in $C^s$ defined in Definition \ref{defCs}.

The ReQU activation function
 has a continuous first order derivative and its first order derivative is the popular ReLU function. With ReQU as the activation function, the network is smooth and differentiable. Therefore, ReQU is a suitable choice for our problem
since derivatives are involved
in the penalty function.

An important property of ReQU is that
%Most importantly
it can represent the square function $x^2$
% can be represented
 without error.   In the study of ReLU network approximation properties \citep{yarotsky2017error,yarotsky2018optimal,shen2019deep},  the analyses rely essentially on the fact that $x^2$ can be approximated by deep ReLU networks to any error tolerance as long as the network is large enough. With ReQU activated networks, $x^2$ can be represented exactly with one hidden layer and 2 hidden neurons. ReQU can be more efficient in approximating
 smooth functions in the sense that it requires a smaller network size to achieve the same approximation error.

Now we state some basic approximation properties of ReQU networks. The analysis of the approximation power of ReQU networks in our work basically rests on the fact that given inputs $x,y\in\mathbb{R}$, the powers $x,x^2$ and the product $xy$ can be exactly computed by simple ReQU networks.
% with few layers and neurons.
Let $\sigma_2(x)=[\max\{x,0\}]^2$ denote the ReQU activation function. We first list the following basic properties of the ReQU approximation:
\begin{itemize}
	\item [(1)]	 For any $x\in\mathbb{R}$, the square function $x^2$ can be computed by a ReQU network with 1 hidden layer and 2 neurons, i.e.,
	\begin{align*}
		x^2=&\sigma_2(x)+\sigma_2(-x).
	\end{align*}
\item [(2)] For any $x,y\in\mathbb{R}$, the multiplication function $xy$  can be computed by a ReQU network with 1 hidden layer and 4 neurons, i.e.,
	\begin{align*}
	xy=\frac{1}{4}\{\sigma_2(x+y)+\sigma_2&(-x-y)-\sigma_2(x-y)-\sigma_2(-x+y)\}.
	\end{align*}
\item [(3)] For  any $x\in\mathbb{R}$, taking $y=1$ in the above equation, then the identity map $x\mapsto x$  can be computed by a ReQU network with 1 hidden layer and 4 neurons, i.e.,
\begin{align*}
	x=\frac{1}{4}\{\sigma_2(x+1)+\sigma_2(-x-1)-\sigma_2(x-1)-\sigma_2(-x+1)\}.
\end{align*}

\item [(4)] If both $x$ and $y$ are non-negative, the formulas for square function and multiplication can be simplified as follows:
\begin{align*}
	x^2=\sigma_2(x),\qquad xy=\frac{1}{4}\{\sigma_2(x+y)-\sigma_2(x-y)-\sigma_2(-x+y)\}.
\end{align*}
\end{itemize}

The above equations can be verified using simple algebra.
%The proof is straightforward by simply verifying above equations.
The realization of the identity map is not unique here, since for any $a\not=0$, we have $x=\{(x+a)^2-x^2-a^2\}/(2a)$ which can be exactly realized by ReQU networks.
In addition, the constant function $1$ can be computed exactly by a 1-layer ReQU network with zero weight matrix and constant 1 bias vector. In such a case, the basis $1,x,x^2,\ldots,x^p$ of the degree $p\in\mathbb{N}_0$ polynomials in $\mathbb{R}$ can be computed by a ReQU network with proper size. Therefore, any $p$-degree polynomial can be approximated without error.

To approximate the square function in (1) with ReLU networks on bounded regions,  the idea of using ``sawtooth" functions was first raised in \citet{yarotsky2017error}, and it achieves an error $\mathcal{O}(2^{-L})$ with width 6 and depth $\mathcal{O}(L)$ for positive integer $L\in\mathbb{N}^+$. General construction of ReLU networks for approximating a square function can achieve an error $N^{-L}$ with width $3N$ and depth $L$ for any positive integers $N, L\in\mathbb{N}^+$ \citep{lu2021deep}. Based on this basic fact, the ReLU networks approximating multiplication and polynomials can be constructed correspondingly. However, the network complexity (cost)  in terms of network size (depth and width) for a ReLU network to achieve precise approximation can be large compared to that of a ReQU network since ReQU network can compute polynomials exactly with fewer layers and neurons.

\begin{theorem}[Approximation of Polynomials by ReQU networks]\label{approx}
	For any non-negative integer $N\in\mathbb{N}_0$ and any positive integer $d\in\mathbb{N}^+$,  if $f:\mathbb{R}^d\to\mathbb{R}$ is a polynomial of $d$ variables with total degree $N$, then there exists a ReQU activated neural network that can compute $f$ with no error. More exactly,
	\begin{itemize}
		\item [(1)] if $d=1$ where $f(x)=\sum_{i=1}^Na_ix^i$ is a univariate polynomial with degree $N$, then there exists a ReQU neural network with $2N-1$ hidden layers , $5N-1$ number of neurons, $8N$ number of parameters (weights and bias) and network width $4$ that computes $f$ with no error.
		\item [(2)] If $d\ge2$ where $f(x_1,\ldots,x_d)=\sum_{i_1+\ldots+i_d=0}^Na_{i_1,\ldots,i_d}x_1^{i_1}\cdots x_d^{i_d}$ is a multivariate polynomial of $d$ variables with total degree $N$, then there exists a ReQU neural network with $2N-1$ hidden layers , $2(5N-1)N^{d-1}+(5N-1)\sum_{j=1}^{d-2}N^j\le15N^d$ number of neurons, $16N^{d}+8N\sum_{j=1}^{d-2}N^j\le24N^d$ number of parameters (weights and bias) and network width $8N^{d-1}+4\sum_{j=1}^{d-2}N^j\le12N^{d-1}$ that computes $f$ with no error.
	\end{itemize}
\end{theorem}

Theorem \ref{approx} shows any $d$-variate multivariate polynomial with degree $N$ on $\mathbb{R}^d$ can be represented with no error by a ReQU network with $2N-1$ hidden layers, no more than $15N^d$ neurons, no more than $24N^d$ parameters (weights and bias) and width less than $12N^{d-1}$. The approximation powers of ReQU networks (and RePU networks) on polynomials are studied in \citet{li2019powernet,li2019better}, in which the representation of a $d$-variate multivariate polynomials with degree $N$ on $\mathbb{R}^d$ needs a ReQU network with $d\lfloor\log_2N\rfloor+d$ hidden layers, and no more than $\mathcal{O}(\binom{N+d}{d})$ neurons and parameters. Compared to the results in \citet{li2019better,li2019powernet}, the orders of neurons and parameters for the constructed ReQU network in Theorem \ref{approx} are basically the same. The the number of hidden layers for the constructed ReQU network here is $2N-1$ depending only on the degree of the target polynomial and independent of the dimension of input $d$, which is different from the dimension depending $d\lfloor\log_2N\rfloor+d$ hidden layers required in \citet{li2019better}. In addition, ReLU activated networks with width $\{9(W+1)+N-1\}N^{d}=\mathcal{O}(WN^d)$ and depth $7N^2L=\mathcal{O}(LN^2)$ can only approximate $d$-variate multivariate polynomial with degree $N$ with an accuracy $9N(W+1)^{-7NL}=\mathcal{O}(NW^{-LN})$ for any positive integers $W,L\in\mathbb{N}^+$. Note that the approximation results on polynomials using ReLU networks are generally on bounded regions, while ReQU can exactly compute the polynomials on $\mathbb{R}^d$. In this sense, the approximation power of ReQU networks is generally greater than that of ReLU networks.

Next, we leverage the approximation power of multivariate polynomials to derive error bounds of approximating general multivariate smooth functions using ReQU activated neural networks. Here we focus on the approximation of multivariate smooth functions in  $C^s$ space for $s\in\mathbb{N}^+$ defined in Definition \ref{defCs}.

%Now we give the error bound for the approximation of the smooth functions $C^s$  by %ReQU networks approximation.

\begin{theorem}\label{approx2}
	Let $f$ be a real-valued function defined on $\mathcal{X}\times(0,1)\subset\mathbb{R}^{d+1}$ belonging to  class $C^s$ for $0\le s<\infty$. For any $N\in\mathbb{N}^+$, there exists a ReQU activated neural network $\phi_N$ with width no more than $12N^{d}$, hidden layers no more than $2N-1$, number of neurons no more than $15N^{d+1}$ and parameters no more than $24N^{d+1}$ such that for each multi-index $\alpha\in\mathbb{N}^d_0$, we have $\vert\alpha\vert_1\le\min\{s,N\}$,
	$$\sup_{\mathcal{X}\times(0,1)}\vert D^\alpha (f-\phi_N)\vert\le C_{s,d,\mathcal{X}}\, N^{-(s-\vert\alpha\vert_1)}\Vert f\Vert_{C^s},$$
	where $C_{s,d,\mathcal{X}}$  is a positive constant depending only on $d,s$ and the diameter of $\mathcal{X}\times(0,1)$.
\end{theorem}

{\color{black}
In \citet{li2019powernet, li2019better}, a similar rate of convergence $\mathcal{O}(N^{-(s-\alpha)})$ under the Jacobi-weighted $L^2$ norm was obtained for the approximation of $\alpha$-th derivative of a univariate target function, where $\alpha\le s\le N+1$ and $s$ denotes the smoothness of the target function belonging to Jacobi-weighted Sobolev space.  The ReQU network in \citet{li2019better} has a different shape from
ours specified in Theorem \ref{thm2}. The results of  \citet{li2019better} achieved a $\mathcal{O}(N^{-(s-\alpha)})$ rate using a ReQU network with $\mathcal{O}(\log_2(N))$ hidden layers, $\mathcal{O}(N)$ neurons and nonzero weights and width $\mathcal{O}(N)$. Simultaneous approximation
to the target function and its derivatives by a ReQU network was also considered in \citet{duan2021convergence} for solving partial differential equations for $d$-dimensional smooth target functions in $C^2$.
  %For deep Ritz method using deep ReQU networks, \citet{duan2021convergence} showed that  can be approximated with accuracy $\epsilon$ under $C^1$ norm by fixed-depth ReQU network with depth $\lceil\log_2(d)\rceil+2$ and width $\mathcal{O}(\epsilon^{-d})$.
}

Now we assume that the target function $f_0:\mathcal{X}\times(0,1)\to\mathbb{R}$ in our QRP estimation problem belongs to the smooth function class $C^s$ for some $s\in\mathbb{N}^+$. The approximation error $\inf_{f\in\mathcal{F}_n}\Big[\mathcal{R}(f)-\mathcal{R}(f_0)+\lambda\{\kappa(f)-\kappa(f_0)\}\Big]$ given in Lemma \ref{decom} can be handled correspondingly.

\begin{corollary}[Approximation error bound]\label{cor1}
Suppose that the target function $f_0$ defined in (\ref{qmodel}) belongs to $C^s$ for some $s\in\mathbb{N}^+$.
For any $N\in\mathbb{N}^+$, let $\mathcal{F}_n:=\mathcal{F}_{\mathcal{D},\mathcal{W},\mathcal{U},\mathcal{S},\mathcal{B},\mathcal{B}^\prime}$ be the ReQU activated neural networks $f:\mathcal{X}\times(0,1)\to\mathbb{R}$ with depth (number of hidden layer) $\mathcal{D}\le2N-1$, width $\mathcal{W}\le12N^d$, number of neurons $\mathcal{U}\le15N^{d+1}$, number of parameters (weights and bias) $\mathcal{S}\le24N^{d+1}$, satisfying $\mathcal{B}\ge\Vert f_0\Vert_{C^0}$ and $\mathcal{B}^\prime\ge\Vert f_0\Vert_{C^1}$. Then the approximation error given in Lemma \ref{decom} satisfies
$$\inf_{f\in\mathcal{F}_n}\Big[\mathcal{R}(f)-\mathcal{R}(f_0)+\lambda\{\kappa(f)-\kappa(f_0)\}\Big]\le C_{s,d,\mathcal{X}} (1+\lambda) N^{-(s-1)}\Vert f_0\Vert_{C^s},$$
where $C_{s,d,\mathcal{X}}$  is a positive constant depending only on $d,s$ and the diameter of the support $\mathcal{X}\times(0,1)$.
\end{corollary}

\section{Computation}
\label{computation}
In this section, we
%illustrate
describe the training algorithms for  the proposed
%deep quantile regression process
penalized DQRP estimator, including a generic
%learning
algorithm and an improved
%learning
algorithm.

\begin{algorithm}[H]
	\caption{%Minibatch
An stochastic gradient descent algorithm for
% training of
the penalized DQRP estimator}\label{alg:1}
	\begin{algorithmic}
	\Require Sample data $\{(X_i,Y_i)\}_{i=1}^n$ with $n\ge1$; Minibatch size $ m\le n$.
	\State Generate $n$ random values $\{\xi_i\}_{i=1}^n$ uniformly from $(0,1)$
			\For{number of training iterations}
			\State Sample minibatch of $m$ data $\{(X^{(j)},Y^{(j)},\xi^{(j)})\}_{j=1}^m$ form the data $\{(X_i,Y_i,\xi_i)\}_{i=1}^n$
			\State Update the ReQU network $f$ parametrized by $\theta$ by descending its stochastic gradient:
			$$\nabla_\theta\frac{1}{m} \sum_{j=1}^{m}\Big[ \rho_{\xi^{(j)}}(Y^{(j)}-f(X^{(j)},\xi^{(j)}))+\lambda\max
\Big\{-\frac{\partial}{\partial\tau}f(X^{(j)},\xi^{(j)}),0\Big\}\Big]$$
		\EndFor
		\State The gradient-based updates can use any standard gradient-based algorithm. We used Adam in our experiments.
	\end{algorithmic}
\end{algorithm}

In Algorithm \ref{alg:1}, the number of random values $\{\xi_i\}_{i=1}^n$ is set to be the same as the sample size $n$ and each $\xi_i$ is coupled with the sample $(X_i,Y_i)$ for $i=1,\ldots,n$ during the training process. This may degrade the efficiency of the learning DQRP $\hat{f}^\lambda_n$ since each data $(X_i,Y_i)$ has only been used to train the ReQU network $f(\cdot,\xi_i)$ at a single value (quantile) $\xi_i$. Hence, we proposed an improved algorithm.

\begin{algorithm}[H]
	\caption{An improved
%Minibatch
stochastic gradient descent algorithm for
%training of
the penalized DQRP estimator}\label{alg:2}
	\begin{algorithmic}
		\Require Sample data $\{(X_i,Y_i)\}_{i=1}^n$ with $n\ge1$; Minibatch size $ m\le n$.
		\For{number of training iterations}
		\State Sample minibatch of $m$ data $\{(X^{(j)},Y^{(j)})\}_{j=1}^m$ form the data $\{(X_i,Y_i)\}_{i=1}^n$
		\State Generate $m$ random values $\{\xi_j\}_{j=1}^m$ uniformly from $(0,1)$
		\State Update the ReQU network $f$ parametrized by $\theta$ by descending its stochastic gradient:
		$$\nabla_\theta\frac{1}{m} \sum_{j=1}^{m}\Big[ \rho_{\xi_j}(Y^{(j)}-f(X^{(j)},\xi_j))+\lambda\max
\Big\{-\frac{\partial}{\partial\tau}f(X^{(j)},\xi_j),0\Big\}\Big]$$
		\EndFor
		\State The gradient-based updates can use any standard gradient-based algorithm. We used Adam in our experiments.
	\end{algorithmic}
\end{algorithm}
In Algorithm \ref{alg:2}, at each minibatch training iteration, $m$ random values $\{\xi_j\}_{j=1}^m$ are generated uniformly from $(0,1)$ and coupled with the minibatch sample $\{(X^{(j)},Y^{(j)})\}_{j=1}^m$ for the gradient-based updates. In this case, each sample $(X_i,Y_i)$ gets involved in the training of ReQU network $f(\cdot,\xi)$ at multiple values (quantiles) of $\xi=\xi^{(1)}_i,\ldots,\xi^{(t)}_i$ where $t$ denotes the number of minibatch iterations and $\xi^{(j)}_i,j=1,\ldots,t$ denotes the random value generated at iteration $t$ that is coupled with the sample $(X_i,Y_i)$. In such a way, the utilization of each sample $(X_i,Y_i)$ is greatly improved while the computation complexity does not increase compared to the generic Algorithm \ref{alg:1}.

{\color{black} We use an example to  demonstrate the advantage of  Algorithm \ref{alg:2} over Algorithm \ref{alg:1}.
Figure \ref{fig:algo} displays a comparison between Algorithm \ref{alg:1} and Algorithm \ref{alg:2} with the same simulated dataset generated from the ``Wave" model (see section \ref{sec_num} for detailed introduction of the simulated model). The sample size  $n=512$, and the tuning parameter is chosen as $\lambda=\log(n)$. Two ReQU neural networks with  same architecture (width of hidden layers $(256,256,256)$) are trained for 200 epochs by Algorithm \ref{alg:1} and Algorithm \ref{alg:2}, respectively.
The example and the simulation studies in section \ref{sec_num} show that Algorithm \ref{alg:2} has a better and more stable performance than Algorithm \ref{alg:1} without additional computational complexity.
}

\begin{figure}[H]
	\centering
	\begin{minipage}[H]{0.40\textwidth}
		\centering
		\includegraphics*[width=\textwidth,height=1.5 in]{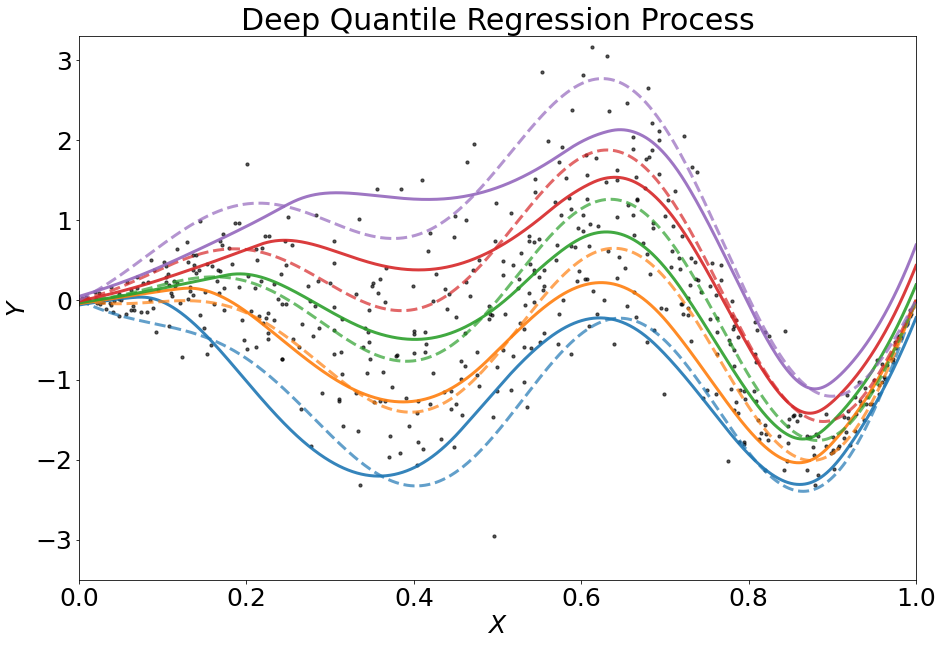}
		\centerline{(a) Trained by Algorithm \ref{alg:1}}
	\end{minipage}
	\begin{minipage}[H]{0.40\textwidth}
		\centering
		\includegraphics*[width=\textwidth,height=1.5 in]{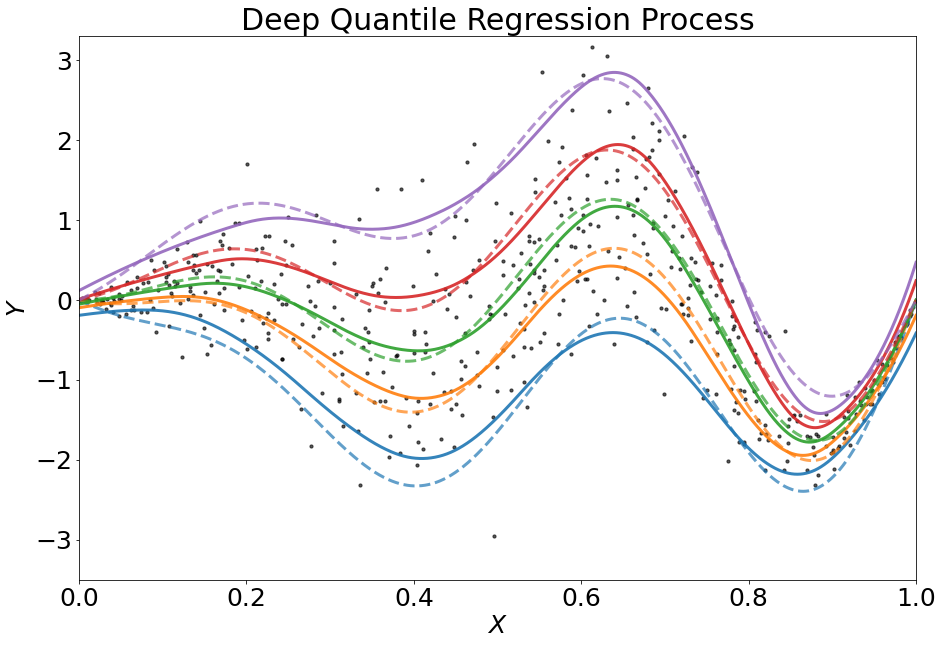}
		\centerline{(b) Trained by Algorithm \ref{alg:2}}
	\end{minipage}
	\caption{A comparison of Algorithms \ref{alg:1} and \ref{alg:2}. The 512 training data generated from the ``Wave" model are depicted as black dots. The target quantile functions at  quantile levels 0.05 (blue), 0.25 (orange), 0.5 (green), 0.75 (red), 0.95 (purple) are depicted as dashed curves, and the estimated quantile functions are the solid curves with the same color. In the left panel, the estimator is trained by  Algorithm \ref{alg:1}. In the right panel, the estimator is trained by the improved Algorithm \ref{alg:2}. Both trainings stop after 200 epochs.}
	\label{fig:algo}
\end{figure}

\section{Numerical studies}\label{sec_num}

In this section, we compare the proposed penalized deep quantile regression with the following nonparametric quantile regression methods:
% reproducing kernel methods and random forest methods using simulated data. To be %specific, we compare the following  methods of quantile regression:
\begin{itemize}
	%\item The traditional linear quantile regression as described in \cite{koenker1978}, denoted by \textit{linear QR}. Without regularization, the empirical risk is minimized over the parameter space (intercept included) $\mathbb{R}^{d+1}$ to give an linear estimator. These estimation are implemented on Python via package \textit{statsmodels}.
	\item  Kernel-based nonparametric quantile regression  \citep{sangnier2016joint}, denoted by \textit{kernel QR}.
	This is a joint quantile regression method based on a vector-valued reproducing kernel Hilbert space (RKHS), which enjoys fewer quantile crossings and enhanced performance compared to the estimation of the quantile functions separately.
% and hard non-crossing constraints.
In our implementation, the radial basis function (RBF) kernel is chosen and a coordinate descent primal-dual algorithm \citep{fercoq2019coordinate} is used via the Python package \textit{qreg}.
	
\item Quantile regression forests \citep{meinshausen2006quantile}, denoted by \textit{QR Forest}. Conditional quantiles can be estimated using  quantile regression forests, a method based on random forests. Quantile regression forests can nonparametrically
 estimate quantile regression functions with high-dimensional predictor variables. This method is shown to be consistent in \citet{meinshausen2006quantile}.
	
	\item Penalized
%deep quantile regression process estimation
DQRP estimator as described in Section \ref{sec2}, denoted by \textit{DQRP}. We implement it in Python via \textit{Pytorch} and use \textit{Adam} \citep{kingma2014adam} as the optimization algorithm with default learning rate 0.01 and default $\beta=(0.9,0.99)$ (coefficients used for computing running averages of gradients and their squares).

\end{itemize}

\subsection{Estimation and Evaluation}
For the proposed penalized DQRP estimator, we set the tuning parameter $\lambda=\log(n)$ across the simulations. Since the \textit{Kernel QR}  and \textit{QR Forest} can only estimate the curves at a given quantile level,
%one or finitely many quantile levels,
we consider using \textit{Kernel QR} and \textit{QR Forest} to estimate the quantile curves at 5 different levels for each simulated model, i.e., we estimate quantile curves for $\tau\in\{0.05,0.25,0.5,0.75,0.95\}$. For each target $f_0$, according to model (\ref{model}) we generate the training data $(X_i^{train},Y_i^{train})_{i=1}^n$ with sample size $n$ to train the empirical risk minimizer at $\tau\in\{0.05,0.25,0.5,0.75,0.95\}$
using Kernel QR and QR Forest, i.e.
\begin{align*}
	\hat{f}^\tau_n\in\arg\min_{f\in\mathcal{F}} \frac{1}{n}\sum_{i=1}^n\rho_\tau(Y_i^{train}-f(X_i^{train})),
\end{align*}
where $\mathcal{F}$ is the class of RKHS,  the class of functions for \textit{QR forest},
% based on some kernel
or the class of
%functions implemented by
ReQU neural network functions.

For each $f_0$, we also generate the testing data $(X_t^{test},Y_t^{test})_{t=1}^T$ with sample size $T$ from the same distribution of the training data. For the proposed method and for each obtained estimate $\hat{f}_n$, we denote $\hat{f}_n^\tau(\cdot)=\hat{f}_n(\cdot,\tau)$ for notational simplicity. For DQRP, Kernel QR and QR Forest, we calculate the testing error on $(X_t^{test},Y_t^{test})_{t=1}^T$ at different quantile levels $\tau$. For quantile level $\tau\in(0,1)$, we calculate the $L_1$ distance between $\hat{f}_n^\tau$ and the corresponding risk minimizer $f_0^\tau(\cdot):=f_0(\cdot,\tau)$ by
\begin{align*}
	\Vert\hat{f}_n^\tau-f_0^\tau\Vert_{L^1(\nu)}=\frac{1}{T}\sum_{t=1}^T \Big\vert\hat{f}_n(X_t^{test},\tau)-f_0^\tau(X_t^{test},\tau)\Big\vert,
\end{align*}
and we also calculate the $L_2^2$ distance between $\hat{f}^\tau_n$ and the $f_0^\tau$, i.e.
\begin{align*}
	\Vert\hat{f}^\tau_n-f_0^\tau\Vert^2_{L^2(\nu)}=\frac{1}{T}\sum_{t=1}^T \Big\vert\hat{f}_n(X_t^{test},\tau)-f_0^\tau(X_t^{test},\tau)\Big\vert^2.
\end{align*}
%All the $L_2$ test error results are provided in the supplementary material.
The specific forms of $f_0$ are given in the  data generation models below.

In the simulation studies, the size of testing data $T=10^5$  for each data generation model. We report the mean and standard deviation of
the $L_1$  and $L^2_2$ distances  over $R = 100$ replications under different scenarios.  %For \textit{DLS}, the testing risk and the excess risk are calculated in terms of mean squares loss function other than the check loss $\rho_\tau$.

\subsection{Univariate models}
%In our simulation,
We consider three basic  univariate models, including ``Linear'', ``Wave'' and ``Triangle'', which corresponds to different specifications of the target function $f_0$. The formulae are given below.
\begin{enumerate}[(a)]
	\setlength\itemsep{-0.05 cm}
	\item Linear:
$f_0(x,\tau)=2x+F_t^{-1}(\tau),$
			\item Wave:
	$f_0(x,\tau)=2x\sin(4\pi x) +\vert \sin(\pi x)\vert \Phi^{-1}(\tau), $
		\item Triangle:
	$f_0(x,\tau)=4(1-\vert x-0.5\vert)+\exp(4x-2)  \Phi^{-1}(\tau), $
\end{enumerate}
where  where $F_t(\cdot)$ is the cumulative distribution function of the standard Student's t random variable,   $\Phi(\cdot)$ is the cumulative distribution function of the standard normal random variable.
 We use the linear model as a baseline model in our simulations and expect all the methods
perform well under the linear model. The ``Wave'' is a nonlinear smooth model and the ``Triangle'' is a
nonlinear, continuous but non-differentiable model. These models
are chosen so that we can evaluate the
performance of \textit{DQRP},  \textit{kernel QR} and \textit{QR Forest} under different types of  models.
%we hypothesize that \textit{DQR} and $kernel QR$ perform reasonably well, but $linear QR$ will have %difficulty in fitting this nonlinear smooth model. The ``Triangle'' is a nonlinear continuous but %non-differentiable model, we hypothesize that \textit{DQR} can still perform well, but $kernel QR$ will not do %as well in  capturing the non-differentiable feature of the curve.
%The simulation results largely confirm our hypotheses.

\begin{figure}[H]
	\centering
	\includegraphics[width=5.0 in, height=1.4 in]{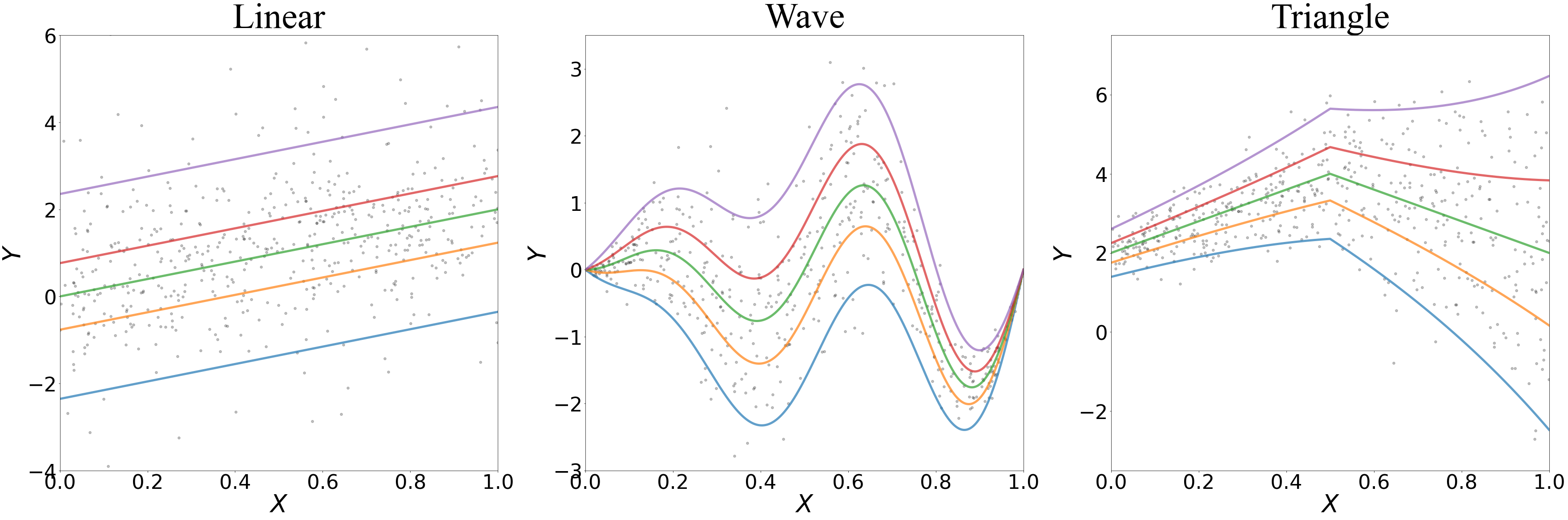}
	\caption{The target quantiles curves.
% at different quantile levels under different models.
From the left to the right, each column corresponds a data generation model, ``Linear'', ``Wave'' and ``Triangle''. The sample data with size $n=512$ is depicted as grey dots.The target quantile functions at the quantile levels $\tau=$0.05 (blue), 0.25 (orange), 0.5 (green), 0.75 (red), 0.95 (purple) are depicted as solid curves.}
	\label{fig:target}
\end{figure}

%Throughout our univariate model simulations,
For these models, we generate $X$ uniformly from the unit interval $[0,1]$.
The $\tau$-th conditional quantile of the response $Y$ given $X=x$ can be calculated directly based on the expression of $f_0(x,\tau)$.
%where $F^{-1}_{\eta\mid X=x}(\cdot)$ is the inverse of the conditional cumulated distribution function of $\eta$ given $X=x$. For $t(3)$ error, $\eta$ is independent with $X$, then $F^{-1}_{\eta\mid X=x}(\cdot)$ is simply the inverse of distributional function of the $2 t(3)$. For the \textit{Sine} error, $F^{-1}_{\eta\mid X=x}(\tau)=0.5\times\sin(\pi x)\times\Phi^{-1}(\tau)$ where $\Phi^{-1}(\cdot)$ is the inverse of the CDF of a standard normal random variable. Similarly, for the \textit{Exp} error, $F^{-1}_{\eta\mid X=x}(\tau)=0.5\times\exp(2x-1)\times\Phi^{-1}(\tau)$.
Figure \ref{fig:target} shows all these univariate data generation models and their corresponding conditional quantile curves at $\tau=0.05,0.25, 0.50, 0.75,0.95$.

Figures \ref{fig:2} to \ref{fig:3} show an instance of the estimated quantile curves for the ``Wave'' and
``Triangle'' models. The plot for the ``Linear'' model is included in the Appendix.
%We also visualize the estimates of these three methods under different models in Figure %\ref{fig:1}-\ref{fig:3}.
In these plots, the training data is depicted as grey dots. The target quantile functions at the quantile levels $\tau=$0.05 (blue), 0.25 (orange), 0.5 (green), 0.75 (red), 0.95 (purple) are depicted as dashed curves, and the estimated quantile functions are represented by solid curves with the same color. For each figure, from the top to the bottom, the rows correspond to the sample size
 $n=512, 2048$. %($n=128,512,2048$).
From the left to the right, the columns correspond to the methods DQRP, kernel QR and QR Forest.

\begin{figure}[H]
	\centering
	\includegraphics[width=5.0 in, height=2.8 in]{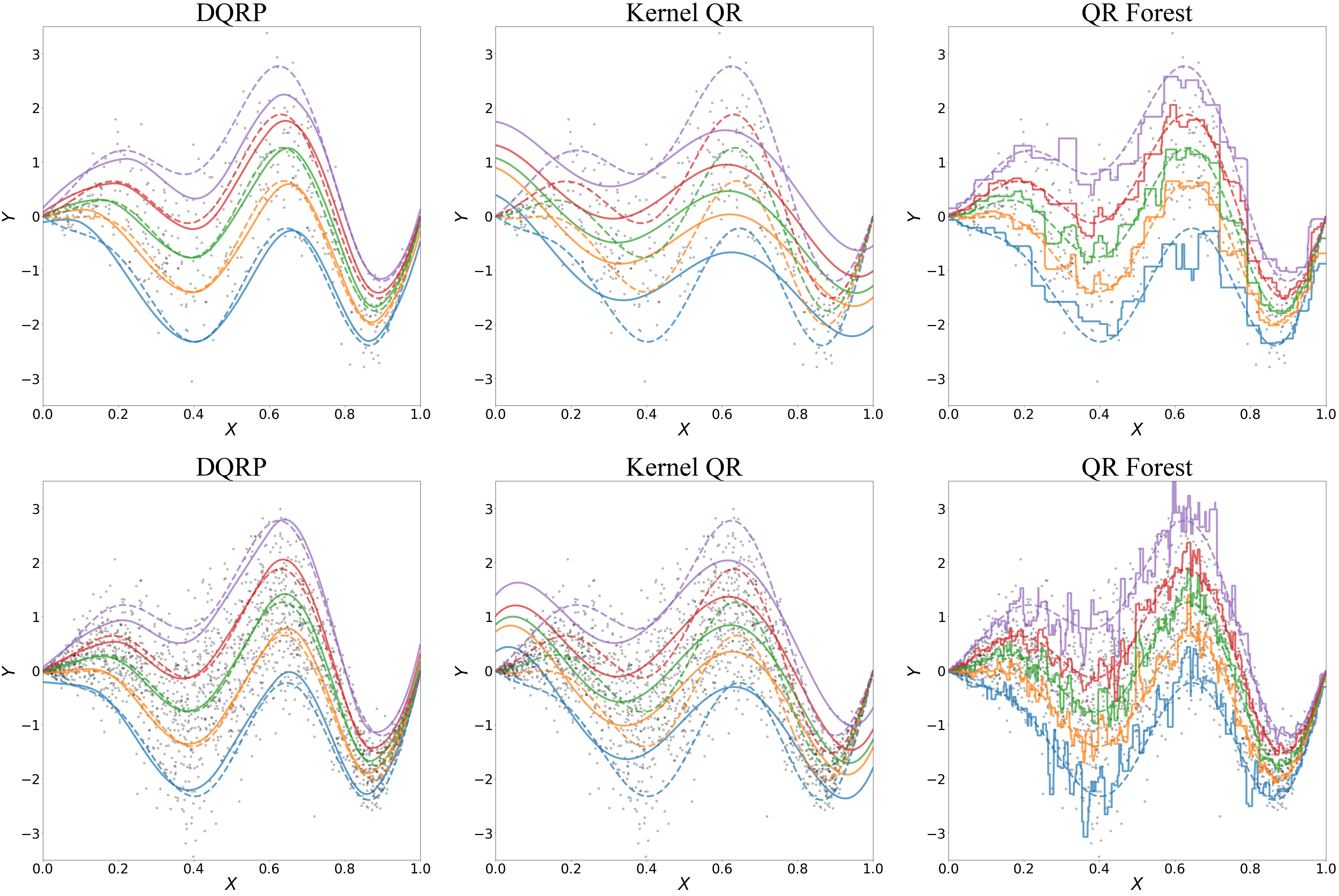}
	\caption{The fitted quantile curves
%by different methods
under the univariate ``Wave" model.
% when $n=512,2048$.
The training data is depicted as grey dots.The target quantile functions at the quantile levels $\tau=$0.05 (blue), 0.25 (orange), 0.5 (green), 0.75 (red), 0.95 (purple) are depicted as dashed curves, and the estimated quantile functions are represented by solid curves with the same color. From the top to the bottom, the rows correspond to the sample size $n=512,2048$. From the left to the right, the columns correspond to the methods DQRP, kernel QR and QR Forest.}
	\label{fig:2}
\end{figure}

\begin{figure}[H]
	\centering
	\includegraphics[width=5.0 in, height=2.8 in]{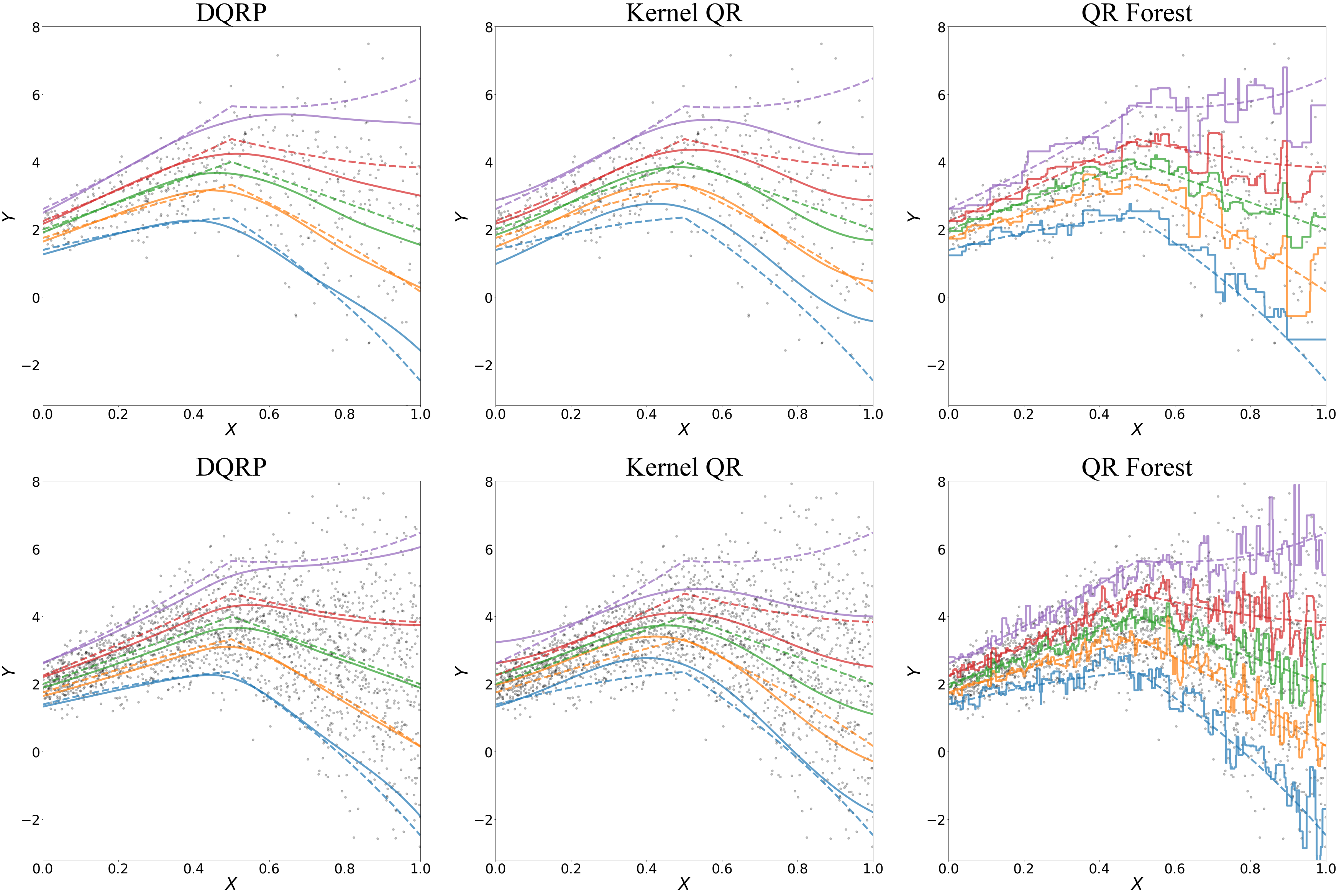}
	\caption{The fitted quantile curves
% by different methods
under the univariate ``Triangle" model.
% when $n=512,2048$.
The training data is depicted as grey dots.The target quantile functions at the quantile levels $\tau=$0.05 (blue), 0.25 (orange), 0.5 (green), 0.75 (red), 0.95 (purple) are depicted as dashed curves, and the estimated quantile functions are represented by solid curves with the same color. From the top to the bottom, the rows correspond to the sample sizes $n=512,2048$. From the left to the right, the columns correspond to the methods DQRP, kernel QR and QR Forest.}
	\label{fig:3}
\end{figure}

\begin{table}[H]
	\caption{Data is generated from the ``Wave" model with training sample size $n= 512,2048$ and the number of replications $R = 100$. The averaged  $L_1$ and $L_2^2$ test errors with the corresponding standard deviation (in parentheses) are reported for the estimators trained by different methods.}
	\label{tab:2}
%	\resizebox{\textwidth}{!}{%
		\begin{tabular}{@{}c|c|cc|cc@{}}
			\toprule
			& Sample size    & \multicolumn{2}{c|}{$n=512$}                     & \multicolumn{2}{c}{$n=2048$}                    \\ \midrule
			$\tau$                & Method               & $L_1$                  & $L_2^2$                & $L_1$                  & $L_2^2$                \\\midrule
			\multirow{3}{*}{0.05} & DQRP    & \textbf{0.184(﻿0.072)} & \textbf{0.065(﻿0.061)} & \textbf{0.127(﻿0.055)} & \textbf{0.029(﻿0.026)} \\
			& Kernel QR         & 0.461(﻿0.072)          & 0.377(﻿0.125)          & 0.599(﻿0.224)          & 0.600(﻿0.470)          \\
			& QR Forest      & 0.228(﻿0.030)          & 0.092(﻿0.024)          & 0.195(﻿0.017)          & 0.071(﻿0.013)          \\\midrule
			\multirow{3}{*}{0.25} & DQRP   & \textbf{0.124(0.041)}  & \textbf{0.030(0.023)}  & \textbf{0.086(0.034)}  & \textbf{0.013(0.011)}  \\
			& Kernel QR    & 0.441(0.064)           & 0.298(0.109)           & 0.571(0.225)           & 0.545(0.460)           \\
			& QR Forest      & 0.166(0.024)           & 0.051(0.015)           & 0.143(0.012)           & 0.039(0.007)           \\\midrule
			\multirow{3}{*}{0.5}  & DQRP     & \textbf{0.112(0.030)}  & \textbf{0.022(0.012)}  & \textbf{0.076(0.024)}  & \textbf{0.010(0.006)}  \\
			& Kernel QR    & 0.440(0.058)           & 0.289(0.105)           & 0.555(0.226)           & 0.530(0.461)           \\
			& QR Forest    & 0.157(0.024)           & 0.045(0.014)           & 0.137(0.010)           & 0.036(0.005)           \\\midrule
			\multirow{3}{*}{0.75} & DQRP   & \textbf{0.131(0.047)}  & \textbf{0.030(0.023)}  & \textbf{0.087(0.030)}  & \textbf{0.013(0.009)}  \\
			& Kernel QR    & 0.462(0.055)           & 0.322(0.107)           & 0.560(0.219)           & 0.546(0.462)           \\
			& QR Forest   & 0.168(0.022)           & 0.050(0.014)           & 0.146(0.013)           & 0.041(0.008)           \\\midrule
			\multirow{3}{*}{0.95} & DQRP     & \textbf{0.192(0.072)}  & \textbf{0.064(0.049)}  & \textbf{0.127(0.042)}  & \textbf{0.027(0.018)}  \\
			& Kernel QR    & 0.552(0.064)           & 0.469(0.120)           & 0.615(0.200)           & 0.648(0.462)           \\
			& QR Forest      & 0.224(0.030)           & 0.090(0.026)           & 0.198(0.018)           & 0.074(0.014)           \\ \bottomrule
		\end{tabular}%
%	}
\end{table}

\begin{table}[H]
	\caption{Data is generated from the ``Triangle" model with training sample size $n= 512,2048$ and the number of replications $R = 100$. The averaged  $L_1$ and $L_2^2$ test errors with the corresponding standard deviation (in parentheses) are reported for the estimators trained by different methods.}
	\label{tab:3}
%	\resizebox{\textwidth}{!}{%
		\begin{tabular}{@{}c|c|cc|cc@{}}
			\toprule
			& Sample size    & \multicolumn{2}{c|}{$n=512$}                     & \multicolumn{2}{c}{$n=2048$}                    \\ \midrule
			$\tau$                & Method         & $L_1$                  & $L_2^2$                & $L_1$                  & $L_2^2$                \\ \midrule
			\multirow{3}{*}{0.05} & DQRP   & \textbf{0.263(﻿0.103)} & \textbf{0.152(﻿0.135)} & \textbf{0.174(﻿0.081)} & \textbf{0.069(﻿0.074)} \\
			& Kernel QR   & 0.533(﻿0.147)          & 0.520(﻿0.268)          & 0.515(﻿0.200)          & 0.490(﻿0.459)          \\
			& QR Forest   & 0.364(﻿0.061)          & 0.282(﻿0.115)          & 0.359(﻿0.031)          & 0.264(﻿0.053)          \\ \midrule
			\multirow{3}{*}{0.25} & DQRP      & \textbf{0.181(0.079)}  & \textbf{0.058(0.054)}  & \textbf{0.112(0.047)}  & \textbf{0.023(0.018)}  \\
			& Kernel QR   & 0.249(0.084)           & 0.110(0.084)           & 0.308(0.138)           & 0.179(0.217)           \\
			& QR Forest      & 0.272(0.044)           & 0.155(0.061)           & 0.259(0.021)           & 0.140(0.026)           \\\midrule
			\multirow{3}{*}{0.5}  & DQRP       & 0.187(0.078)           & 0.061(0.057)           & \textbf{0.118(0.060)}  & \textbf{0.025(0.023)}  \\
			& Kernel QR   & \textbf{0.164(0.063)}  & \textbf{0.052(0.044)}  & 0.231(0.134)           & 0.125(0.158)           \\
			& QR Forest     & 0.251(0.040)           & 0.131(0.053)           & 0.252(0.021)           & 0.133(0.026)           \\ \midrule
			\multirow{3}{*}{0.75} & DQRP            & \textbf{0.238(0.110)}  & \textbf{0.097(0.093)}  & \textbf{0.150(0.900)}  & \textbf{0.041(0.049)}  \\
			& Kernel QR      & 0.252(0.087)           & 0.109(0.081)           & 0.334(0.123)           & 0.192(0.166)           \\
			& QR Forest    & 0.261(0.043)           & 0.142(0.059)           & 0.266(0.024)           & 0.149(0.031)           \\ \midrule
			\multirow{3}{*}{0.95} & DQRP      & \textbf{0.343(0.197)}  & \textbf{0.216(0.259)}  & \textbf{0.224(0.135)}  & \textbf{0.097(0.118)}  \\
			& Kernel QR        & 0.540(0.123)           & 0.522(0.242)           & 0.541(0.175)           & 0.527(0.363)           \\
			& QR Forest   & 0.359(0.057)           & 0.256(0.090)           & 0.357(0.033)           & 0.259(0.053)           \\ \bottomrule
		\end{tabular}%
%	}
\end{table}

Tables \ref{tab:2} and \ref{tab:3} summarize the results for the models ``Wave'' and ``Triangle'',
respectively.
%The simulation results under ``Linear", ``Wave" and ``Triangle" models are summarized in Table %\ref{tab:1}-\ref{tab:3} respectively.
For Kernel QR, QR Forest and our proposed DQRP estimator, the corresponding $L_1$ and $L_2^2$ errors %distances
(standard deviation in parentheses) between the estimates and the target are reported at different quantile levels $\tau=0.05,0.25,0.50,0.75,0.95$. For each column, using bold text we highlight the best method which produces the smallest risk among these three methods.
For the ``Wave'' model, the proposed DQRP outperforms Kernel QR  and QR Forest in all the scenarios. For the nonlinear ``Triangle'' model, DQRP also tends to perform better than Kernel QR and QR Forest.  For the ``Linear'' model, the results from the three methods are comparable, but Kernel QR tends to have better performance. The results for the ``Linear'' model are given in Table \ref{tab:1} in the Appendix.

%It can be seen that our proposed DQRP performs very well across different models, especially under %the nonlinear models ``Wave" and ``Triangle".

\subsection{Multivariate models}
We consider three basic multivariate models, including linear model (``Linear''), single index model (``SIM'') and additive model (``Additive''), which correspond to different specifications of the target function $f_0$. The formulae are given below.
\begin{enumerate}[(a)]
	\setlength\itemsep{-0.05 cm}
	\item Linear:
	\begin{align*} f_0(x,\tau)=2A^\top x+F_t^{-1}(\tau),
	\end{align*}
	\item Single index model:
	$$f_0(x,\tau)=\exp(0.1\times A^\top x) +\vert \sin(\pi B^\top x)\vert \Phi^{-1}(\tau),$$
	\item Additive model:
	$$f_0(x,\tau)= 3x_1+4(x_2-0.5)^2+2\sin(\pi x_3)-5\vert x_4-0.5\vert+\exp\{0.1(B^\top x-0.5)\}  \Phi^{-1}(\tau),$$
\end{enumerate}
where $F_t(\cdot)$ denotes the cumulative distribution function of the standard Student's t random variable, $\Phi(\cdot)$ denotes the cumulative distribution function of the standard normal random variable and the parameters ($d$-dimensional vectors)
\begin{align*}
	A=&(0.409,  0.908, 0,  0, -2.061,  0.254,  3.024,  1.280)^\top, \\
	B=&(1.386, -0.902,  5.437, 0,  0, -0.482,  4.611,  0)^\top.
\end{align*}

\begin{table}[H]
	\caption{Data is generated from the ``Single index model" with training sample size $n= 512,2048$ and the number of replications $R = 100$. The averaged  $L_1$ and $L_2^2$ test errors with the corresponding standard deviation (in parentheses) are reported for the estimators trained by different methods.}
	\label{tab:5}
%		\resizebox{\textwidth}{!}{%
	\begin{tabular}{@{}c|c|cc|cc|cc@{}}
		\toprule
		& Sample size            & \multicolumn{2}{c|}{$n=512$}                     & \multicolumn{2}{c}{$n=1024$}                    \\ \midrule
		$\tau$                & Method               & $L_1$                  & $L_2^2$                & $L_1$                  & $L_2^2$                \\\midrule
		\multirow{3}{*}{0.05} & DQRP       & 0.487(﻿0.034)          & 0.422(﻿0.065)          & \textbf{0.391(﻿0.017)} & \textbf{0.277(﻿0.034)} \\
		& Kernel QR   & 0.641(﻿0.043)          & 0.596(﻿0.078)          & 0.620(﻿0.030)       & 0.561(﻿0.059)   \\
		& QR Forest  & \textbf{0.460(﻿0.012)} & \textbf{0.318(﻿0.031)} & 0.450(﻿0.004)          & 0.305(﻿0.013)          \\\midrule
		\multirow{3}{*}{0.25} & DQRP  & 0.241(0.027)           & 0.126(0.034)           & \textbf{0.188(0.013)}  & 0.068(0.010)           \\
		& Kernel QR     & 0.462(0.043)           & 0.330(0.054)           & 0.486(0.043)           & 0.361(0.062)           \\
		& QR Forest    & \textbf{0.214(0.010)}  & \textbf{0.065(0.006)}  & 0.207(0.005)           & \textbf{0.058(0.003)}  \\ \midrule
		\multirow{3}{*}{0.5}  & DQRP      & 0.198(0.026)           & 0.096(0.029)           & 0.112(0.016)           & 0.029(0.008)           \\
		& Kernel QR     & 0.339(0.035)           & 0.188(0.036)           & 0.346(0.048)           & 0.193(0.053)           \\
		& QR Forest   & \textbf{0.081(0.010)}  & \textbf{0.010(0.003)}  & \textbf{0.058(0.005)}  & \textbf{0.005(0.001)}  \\ \midrule
		\multirow{3}{*}{0.75} & DQRP            & 0.279(0.025)           & 0.168(0.036)           & \textbf{0.202(0.016)}  & 0.080(0.013)           \\
		& Kernel QR          & 0.453(0.047)           & 0.317(0.059)           & 0.492(0.044)           & 0.368(0.063)           \\
		& QR Forest    & \textbf{0.213(0.011)}  & \textbf{0.064(0.007)}  & 0.207(0.005)           & \textbf{0.058(0.003)}  \\ \midrule
		\multirow{3}{*}{0.95} & DQRP        & 0.488(0.033)           & 0.443(0.057)           & \textbf{0.416(0.025)}  & \textbf{0.303(0.035)}  \\
		& Kernel QR         & 0.637(0.044)           & 0.589(0.080)           & 0.627(0.033)           & 0.572(0.066)           \\
		& QR Forest   & \textbf{0.459(0.010)}  & \textbf{0.317(0.028)}  & 0.451(0.005)           & 0.306(0.014)           \\ \bottomrule
	\end{tabular}
%}
\end{table}

\begin{table}[H]
	\caption{Data is generated from the ``Additive" model with training sample size $n= 512,2048$ and the number of replications $R = 100$. The averaged  $L_1$ and $L_2^2$ test errors with the corresponding standard deviation (in parentheses) are reported for the estimators trained by different methods.}
	\label{tab:6}
%		\resizebox{\textwidth}{!}{%
	\begin{tabular}{@{}c|c|cc|cc@{}}
		\toprule
		& Sample size   & \multicolumn{2}{c|}{$n=512$}                     & \multicolumn{2}{c}{$n=1024$}                    \\ \midrule
		$\tau$                & Method         & $L_1$                  & $L_2^2$                & $L_1$                  & $L_2^2$                \\\midrule
		\multirow{3}{*}{0.05} & DQRP        & \textbf{0.263(﻿0.103)} & \textbf{0.152(﻿0.135)} & \textbf{0.174(﻿0.081)} & \textbf{0.069(﻿0.074)} \\
		& Kernel QR        & 0.533(﻿0.147)          & 0.520(﻿0.268)          & 0.515(﻿0.200)          & 0.490(﻿0.459)          \\
		& QR Forest  & 0.364(﻿0.061)          & 0.282(﻿0.115)          & 0.359(﻿0.031)          & 0.264(﻿0.053)          \\\midrule
		\multirow{3}{*}{0.25} & DQRP      & \textbf{0.181(0.079)}  & \textbf{0.058(0.054)}  & \textbf{0.112(0.047)}  & \textbf{0.023(0.018)}  \\
		& Kernel QR    & 0.249(0.084)           & 0.110(0.084)           & 0.308(0.138)           & 0.179(0.217)           \\
		& QR Forest        & 0.272(0.044)           & 0.155(0.061)           & 0.259(0.021)           & 0.140(0.026)           \\\midrule
		\multirow{3}{*}{0.5}  & DQRP         & 0.187(0.078)           & 0.061(0.057)           & \textbf{0.118(0.060)}  & \textbf{0.025(0.023)}  \\
		& Kernel QR    & \textbf{0.164(0.063)}  & \textbf{0.052(0.044)}  & 0.231(0.134)           & 0.125(0.158)           \\
		& QR Forest       & 0.251(0.040)           & 0.131(0.053)           & 0.252(0.021)           & 0.133(0.026)           \\\midrule
		\multirow{3}{*}{0.75} & DQRP        & \textbf{0.238(0.110)}  & \textbf{0.097(0.093)}  & \textbf{0.150(0.900)}  & \textbf{0.041(0.049)}  \\
		& Kernel QR         & 0.252(0.087)           & 0.109(0.081)           & 0.334(0.123)           & 0.192(0.166)           \\
		& QR Forest   & 0.261(0.043)           & 0.142(0.059)           & 0.266(0.024)           & 0.149(0.031)           \\\midrule
		\multirow{3}{*}{0.95} & DQRP        & \textbf{0.343(0.197)}  & \textbf{0.216(0.259)}  & \textbf{0.224(0.135)}  & \textbf{0.097(0.118)}  \\
		& Kernel QR           & 0.540(0.123)           & 0.522(0.242)           & 0.541(0.175)           & 0.527(0.363)           \\
		& QR Forest    & 0.359(0.057)           & 0.256(0.090)           & 0.357(0.033)           & 0.259(0.053)           \\ \bottomrule
	\end{tabular}
%}
\end{table}

The simulation results under multivariate
%``Linear",
``SIM" and ``Additive" models are summarized in Tables \ref{tab:5}-\ref{tab:6} respectively. For Kernel QR, QR Forest and our proposed DQRP, the corresponding $L_1$ and $L_2^2$ distances (standard deviation in parentheses) between the estimates and the target are reported at different quantile levels $\tau=0.05,0.25,0.50,0.75,0.95$. For each column, using bold text we highlight the best method which produces the smallest risk among these three methods.

\subsection{Tuning Parameter}
In this subsection, we study the effects of the tuning parameter $\lambda$ on the proposed method.
First, we demonstrate that the ``quantile crossing" phenomenon can be mitigated.
 We
apply our method  to the bone mineral density (BMD) dataset.
% for estimating the quantile regression process.
This dataset is originally reported in \citet{bachrach1999bone} and analyzed in \citet{takeuchi06a, %nonparametric,
hastie2009elements}\footnote{The data is also available from the website \href{http://www-stat.stanford.edu/ElemStatlearn}{http://www-stat.stanford.edu/ElemStatlearn.}}.
The dataset collects the bone mineral density data of 485 North
American adolescents ranging from 9.4 years old to 25.55 years old. Each response value is the difference of the bone mineral density taken on two consecutive visits, divided by the average. The predictor age is the averaged age over the two visits.

In Figure \ref{dem}, we present the estimated quantile regression processes with ($\lambda=\log(n)$) or without ($\lambda=0$) the proposed non-crossing penalty. With or without the penalty, we use the Adam optimizer with the same parameters (for the optimization process) to train a fixed-shape ReQU network with  three hidden layers and width $(128,128,128)$. The estimated quantile curves at $\tau=0.1,0.2,\ldots,0.9$ and the observations are depicted in Figure \ref{dem}. It can be seen that the proposed non-crossing penalty is effective to avoid quantile crossing, even in the area outside the range of the training data.

\begin{figure}[H]
	\centering
	\begin{minipage}[H]{0.48\textwidth}
		\centering
		\includegraphics*[width=2.4 in, height=1.5 in ]{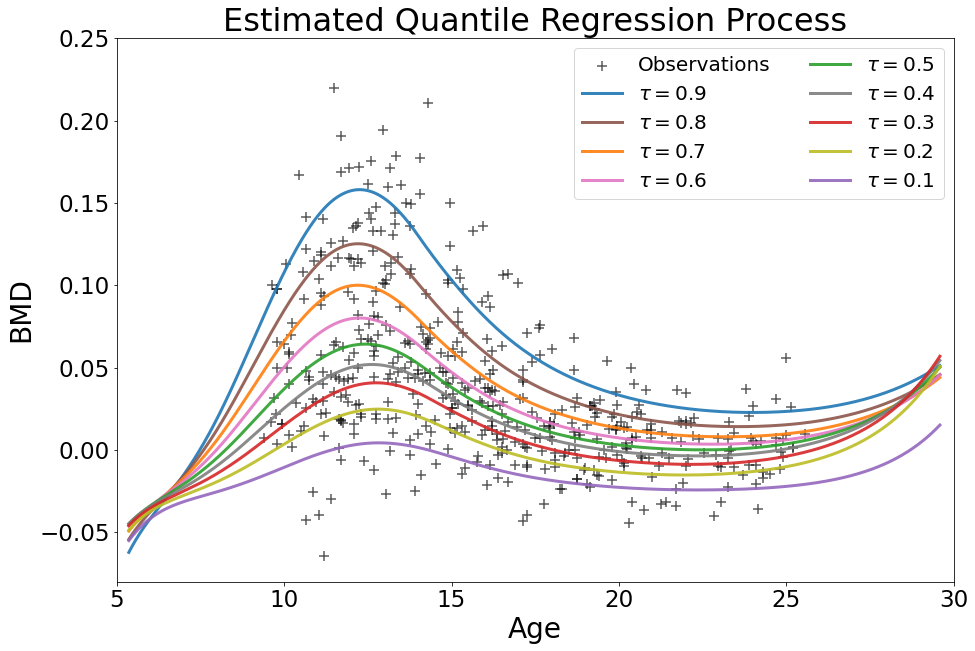}
		\centerline{(a) Without non-crossing penalty}
	\end{minipage}
	\begin{minipage}[H]{0.48\textwidth}
		\centering
		\includegraphics*[width=2.4 in, height=1.5 in]{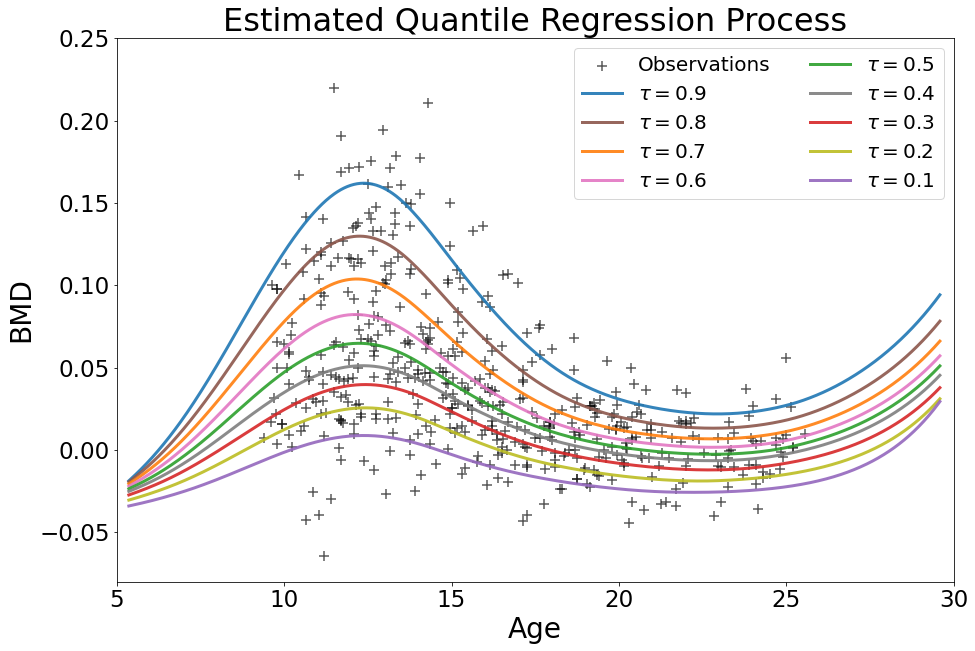}
		\centerline{(b) With non-crossing penalty}
	\end{minipage}
\caption{An example of quantile crossing problem in BMD data set. The estimated quantile curves at $\tau=0.1,0.2,\ldots,0.9$ and the observations are depicted. In the left panel, the estimation is conducted without non-crossing penalty and there are crossings at both edges of the graph. In the right figure, the estimation is conducted with non-crossing penalty. There is no quantile crossing even in the area outside the range of the training data.}
\label{dem}
\end{figure}

Second, we study how the value of tuning parameter $\lambda$ affects the risk of the estimated quantile regression process and how it helps avoiding crossing. Given a sample with size $n$,
% and a certain model,
we train a series of the DQRP estimators at different values of the tuning parameter $\lambda$. For each DQRP estimator, we record its risk and penalty values and the track of these values are plotted in Figures \ref{fig:5}-\ref{fig:7}. For each obtained DQRP estimator $\hat{f}^\lambda_n$, the statistics ``Risk" is calculated according to the formula
$$\mathcal{R}(\hat{f}^\lambda_n)=\mathbb{E}_{X,Y,\xi}\{\rho_{\xi}(Y-f(X,\xi))\},$$
and the statistics ``Penalty" is calculated according to
$$\kappa(\hat{f}^\lambda_n)=\mathbb{E}_{X,\xi}[\max\{-\p \hat{f}^\lambda_n(X,\xi),0\}].$$
In practice, we generate $T=10,000$ testing data $(X_t^{test},Y_t^{test},\xi_t^{tesr})_{t=1}^T$ to empirically calculate risks and penalty values.

In each figure, a vertical dashed line is also depicted at the value $\lambda=\log(n)$. It can be seen that crossing seldom happens when we choose a tiny value of the tuning parameter $\lambda$. And the loss caused by penalty can be negligible compared to the total risk, since the penalty values are generally of order $O(10^{-3})$ instead of $O(10^{3})$ for the total risk. For large value of tuning parameter $\lambda$, the crossing nearly disappears which is intuitive and encouraged by the formulation of our penalty. However, the risk could be very large resulting a poor estimation of the target function. As shown by the dashed vertical line in each figure, numerically the choice of $\lambda=\log(n)$ can lead to a reasonable estimation of the target function with tiny risk (blue lines) and little crossing (red lines) across different model settings. Empirically, we choose $\lambda=\log(n)$ in general for the simulations. By Theorem \ref{non-asymp}, such choice of tuning parameter can lead to a consistent estimator with reasonable fast rate of convergence.

\begin{figure}[H]
	\centering
	\includegraphics[width=0.315\textwidth]{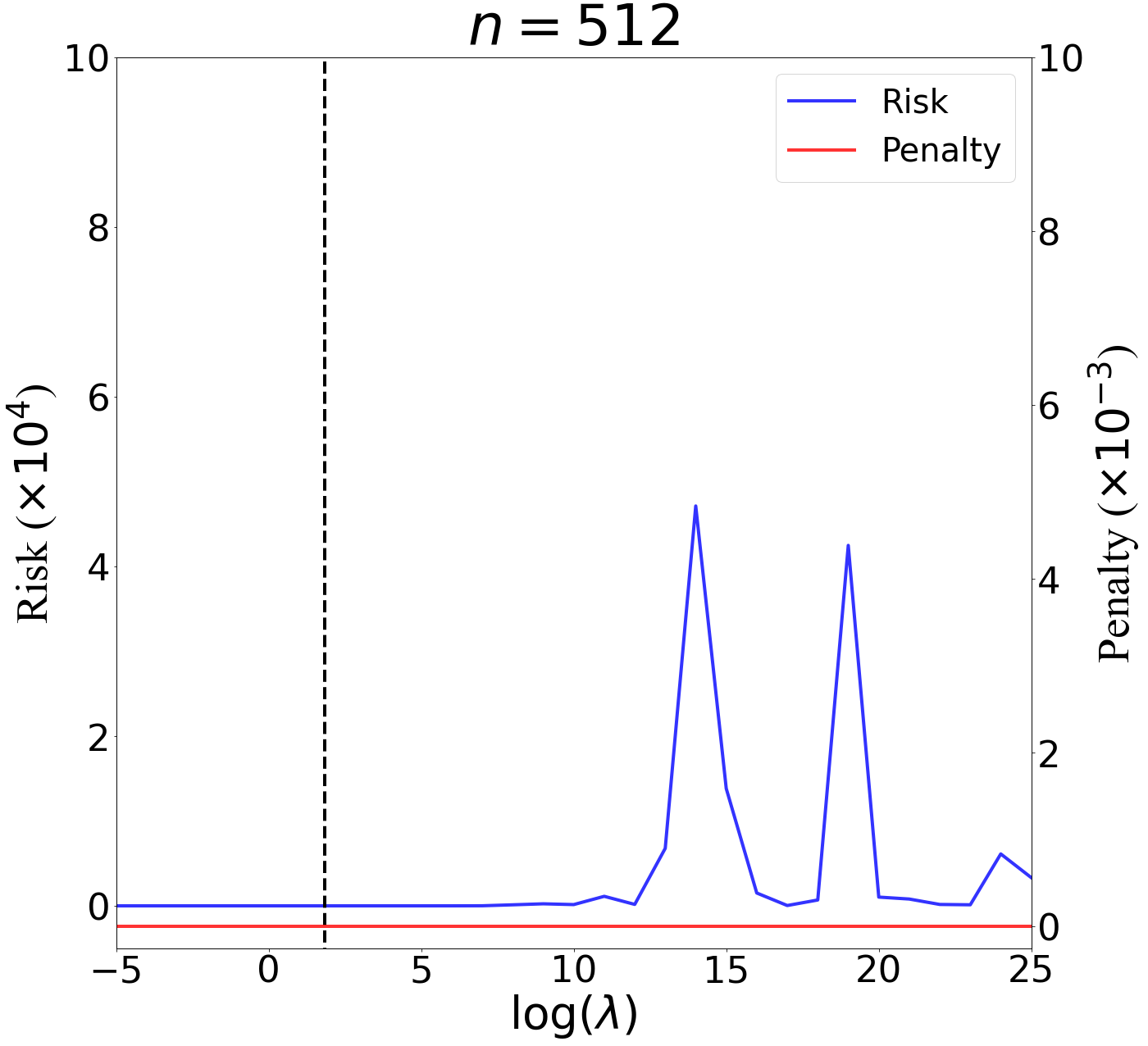}
	\includegraphics[width=0.315\textwidth]{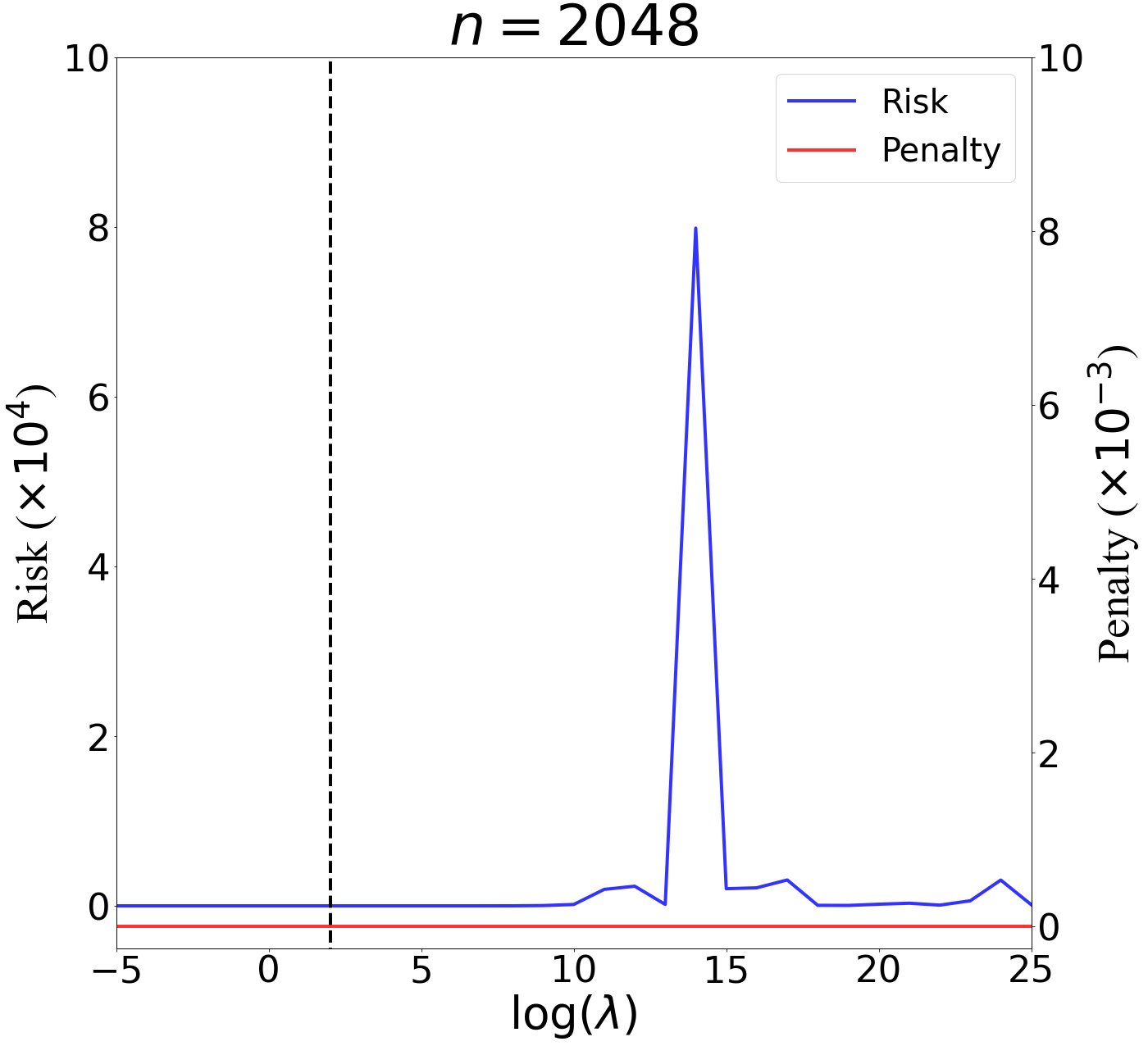}
	\caption{The value of risks and penalties under the univariate ``Triangle" model when $n=512,2048$. A vertical dashed line is depicted at the value $\lambda=\log(n)$ on x-axis in each figure.}
	\label{fig:5}
\end{figure}

\begin{figure}[H]
	\centering
	\includegraphics[width=0.315\textwidth]{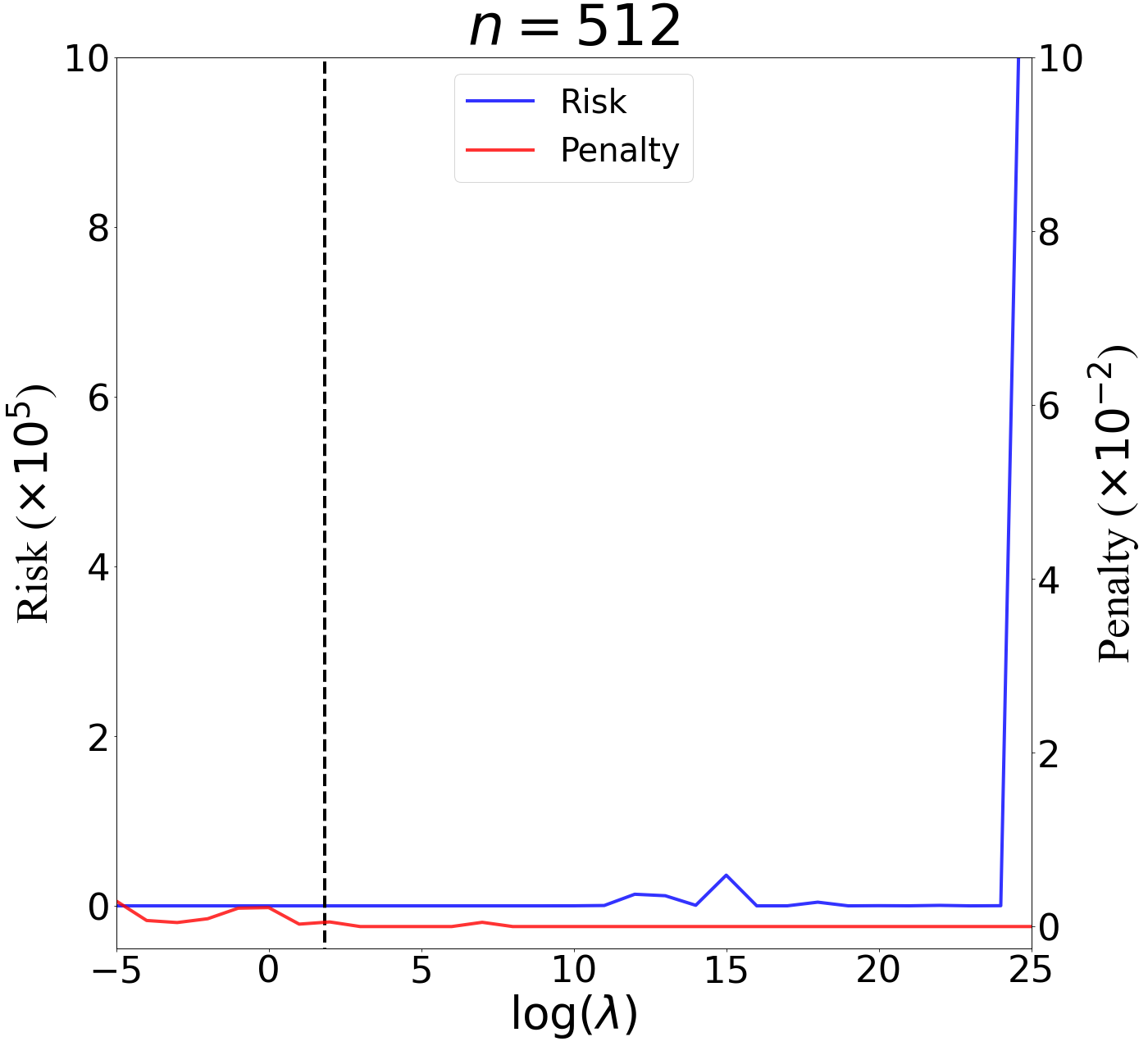}
	\includegraphics[width=0.315\textwidth]{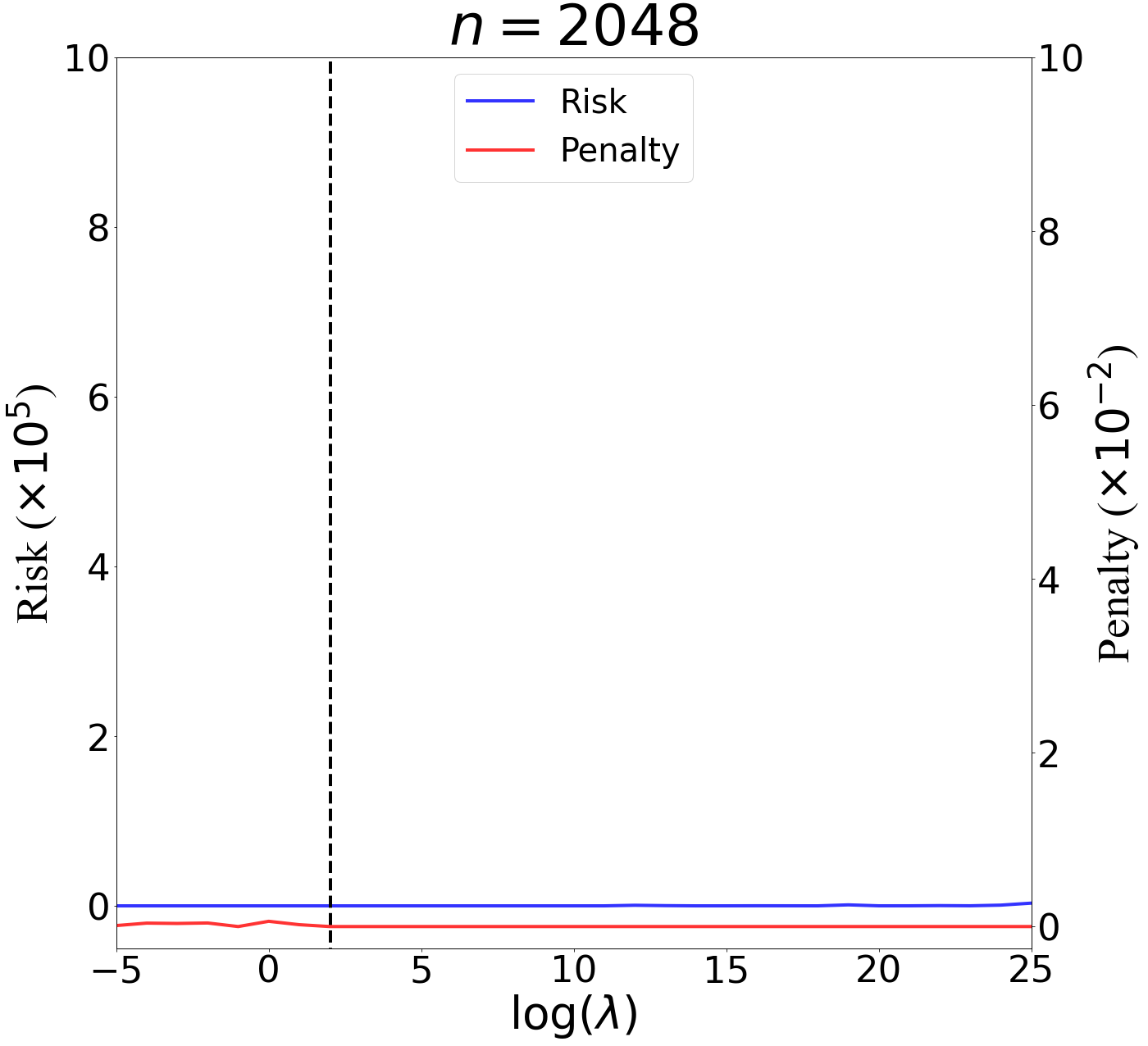}
	\caption{The value of risks and penalties under the multivariate additive model when $n=512,2048$ and $d=8$. A vertical dashed line is depicted at the value $\lambda=\log(n)$ on x-axis in each figure.}
	\label{fig:7}
\end{figure}

\section{Conclusion}
\label{conclusion}
We have proposed a penalized nonparametric approach to estimating the nonseparable model (\ref{model}) using  ReQU activated deep neural networks and introduced a novel penalty function to enforcing non-crossing quantile curves. We have established  non-asymptotic excess risk bounds for the estimated QRP
and derived the mean integrated squared error for the estimated QRP under mild smoothness and regularity conditions.
We have also developed a new approximation error bound for  $C^s$ smooth functions  with
smoothness index $s > 0$  using ReQU activated neural networks.
%This is a new approximation result for ReQU networks.
Our numerical experiments demonstrate that the proposed method is competitive with or outperforms two existing methods, including methods using reproducing kernels and random forests, for nonparmetric quantile regression. Therefore, the proposed approach can be a useful addition to the methods for multivariate nonparametric regression analysis.

The results and methods of this work are expected to be
 useful in other settings.
 %nonparametric estimation  problems.
In particular, our approximation results on ReQU activated networks are of independent interest.
It would be interesting to take advantage of the smoothness of  ReQU activated networks
and use them in other nonparametric estimation problems, such as the estimation of a regression function and its derivative.
% density functions and their derivatives.

\iffalse
\section*{Acknowledgements}
%The authors wish to thank the editor, the associate editor and three anonymous reviewers %for their insightful comments and constructive suggestions that helped improve the paper %significantly.
The work of Y. Jiao is supported in part by the National Science Foundation of China grant 11871474 and by the research fund of KLATASDSMOE of China.
The work of Y. Lin is supported by the Hong Kong Research Grants Council (Grant No.
14306219 and 14306620), the National Natural Science Foundation of China (Grant No.
11961028) and Direct Grants for Research, The Chinese University of Hong Kong.
%The work of J. Huang is partially supported by the U.S. NSF grant DMS-1916199.
\fi

%\clearpage
%\bibliographystyle{apalike}
\bibliographystyle{imsart-nameyear}
\bibliography{bib_dqr.bib}    % Bibliography file (usually '*.bib')

%\end{document}
\clearpage

\appendix

\section{Proof of Theorems, Corollaries and Lemmas}
In 
%this section of 
the appendix, we include the proofs for the results stated in Section \ref{sec3} and the technical details needed in the proofs.

{\color{black}
\subsection*{Proof of Proposition \ref{prop1}}
	For any random variable $\xi$ supported on $(0,1)$, the risk
	\begin{align*}
		\mathcal{R}(f)=&\mathbb{E}_{X,Y,\xi}\{\rho_{\xi}(Y-f(X,\xi))\}\\
		=&\int_{0}^1\mathbb{E}_{X,Y}\{\rho_{\xi}(Y-f(X,\tau))\}\pi_\xi(\tau)d\tau
	\end{align*}
	where $\pi_\xi(\cdot)\ge0$ is the density function of $\xi$. By the definition of $f_0$ and the property of quantile loss function, it is known $f_0$ minimizes $\mathbb{E}_{X,Y}\{\rho_{\xi}(Y-f(X,\tau))\}$ as well as $\mathbb{E}_{X,Y}\{\rho_{\xi}(Y-f(X,\tau))\}\pi_\xi(\tau)$ for each $\tau\in(0,1)$. Thus $f_0$ minimizes the integral or the risk $\mathcal{R}(\cdot)$ over measurable functions.
	
	Note that if $\pi_\xi(\tau)=0$ for some $\tau\in T$ where $T$ is a subset of $(0,1)$, then any function $\tilde{f}_0$ defined on $\mathcal{X}\times(0,1)$ that is different from $f_0$ only on $\mathcal{X}\times T$ will also be a minimizer of $\mathcal{R}(\cdot)$. To be exact,
	$$\tilde{f}_0\in\arg\min_{f}\mathcal{R}(f)\qquad{\rm if\ and\ only\ if}\qquad \tilde{f}_0=f_0{\rm\ on\  }\mathcal{X}\times T.$$
	
	Further, if $(X,\xi)$ has non zero density almost everywhere on $\mathcal{X}\times(0,1)$ and the probability measure of $(X,\xi)$ is absolutely continuous with respect to Lebesgue measure, then above defined set $\mathcal{X}\times T$ is measure-zero and $f_0$ is the unique minimizer of $\mathcal{R}(\cdot)$ over all measurable functions in the sense of almost everywhere(almost surely), i.e.,
	$$f_0=\arg\min_{f} \mathcal{R}(f)=\arg\min_{f}\mathbb{E}_{X,Y,\xi}\{\rho_{\xi}(Y-f(X,\xi))\},$$
	up to a negligible set with respect to the probability measure of $(X,\eta)$ on $\mathcal{X}\times(0,1)$. $\hfill \Box$
}

\subsection*{Proof of Lemma \ref{decom}}
Recall that $\hat{f}^\lambda_n$ is the penalized empirical risk minimizer. Then, for any $f\in\mathcal{F}_n$ we have
\begin{comment}
\begin{align*}
\mathcal{R}_n(\hat{f}^\lambda_n)+\lambda\kappa_n(\hat{f}^\lambda_n)=\mathcal{R}^\lambda_n(\hat{f}_n^\lambda)\le\mathcal{R}^\lambda_n(f)=\mathcal{R}_n(f)+\lambda\kappa_n(f),
\end{align*}
\end{comment}
\begin{align*}
	\mathcal{R}^\lambda_n(\hat{f}^\lambda_n)\le\mathcal{R}^\lambda_n(f).
\end{align*}
Besides, for any $f\in\mathcal{F}$ we have $\kappa(f)\ge0$ and  $\kappa_n(f)\ge0$ since $\kappa$ and $\kappa_n$ are nonnegative functions. Note that $\kappa(f_0)=\kappa_n(f_0)=0$ by the assumption that $f_0$ is increasing in its second argument. Then,
\begin{align*}
	\mathcal{R}(\hat{f}^\lambda_n)-\mathcal{R}(f_0)\le&\mathcal{R}(\hat{f}^\lambda_n)-\mathcal{R}(f_0)+\lambda\{\kappa(\hat{f}^\lambda_n)-\kappa(f_0)\}=\mathcal{R}^\lambda(\hat{f}^\lambda_n)-\mathcal{R}^\lambda(f_0).
\end{align*}
We can then give upper bounds for the excess risk $\mathcal{R}(\hat{f}^\lambda_n)-\mathcal{R}(f_0)$. For any $f\in\mathcal{F}_n$,
\begin{align*}
	&\mathbb{E}\{\mathcal{R}(\hat{f}^\lambda_n)-\mathcal{R}(f_0)\}\\
	&\le\mathbb{E}\{\mathcal{R}^\lambda(\hat{f}^\lambda_n)-\mathcal{R}^\lambda(f_0)\}\\
	&\le\mathbb{E}\{\mathcal{R}^\lambda(\hat{f}^\lambda_n)-\mathcal{R}^\lambda(f_0)\}+2\mathbb{E}\{\mathcal{R}_n^\lambda(f)-\mathcal{R}_n^\lambda(\hat{f}^\lambda_n)\}\\
	&=\mathbb{E}\{\mathcal{R}^\lambda(\hat{f}^\lambda_n)-\mathcal{R}^\lambda(f_0)\}+2\mathbb{E}[\{\mathcal{R}_n^\lambda(f)-\mathcal{R}_n^\lambda(f_0)\}-\{\mathcal{R}_n^\lambda(\hat{f}^\lambda_n)-\mathcal{R}_n^\lambda(f_0)\}]\\
	&=\mathbb{E}\{\mathcal{R}^\lambda(\hat{f}^\lambda_n)-2\mathcal{R}^\lambda_n(\hat{f}^\lambda_n)+\mathcal{R}^\lambda(f_0)\}+2\mathbb{E}\{\mathcal{R}_n^\lambda(f)-\mathcal{R}_n^\lambda(f_0)\}
%	&=\mathcal{R}^\lambda(\hat{f}^\lambda_n)-\mathcal{R}^\lambda_n(\hat{f}^\lambda_n)%-\{\mathcal{R}^\lambda(f_0)-\mathcal{R}^\lambda_n(f_0)\}
%	\\\label{ine2}
%	&+\mathcal{R}^\lambda_n(\hat{f}^\lambda_n)-\mathcal{R}^\lambda_n(f)\\\label{ine3}
%	&+\mathcal{R}^\lambda_n(f)-\mathcal{R}^\lambda(f)%-\{\mathcal{R}^\lambda_n(f_0)-\mathcal{R}^\lambda(f_0)\}
%	\\\label{ine4}
%	&+\mathcal{R}^\lambda(f)-\mathcal{R}^\lambda(f_0)\\\notag
%	&\le 2\sup_{f\in\mathcal{F}_n}\vert \mathcal{R}^\lambda(f)-\mathcal{R}^\lambda_n(f)\vert+\mathcal{R}^\lambda(f)-\mathcal{R}^\lambda(f_0)\\\notag
	%&\le 2\sup_{f\in\mathcal{F}_n}\vert \{\mathcal{R}^\lambda(f)-\mathcal{R}^\lambda(f_0)\}-\{\mathcal{R}^\lambda_n(f)-\mathcal{R}^\lambda_n(f_0)\}\vert+\mathcal{R}^\lambda(f)-\mathcal{R}^\lambda(f_0)\\\notag
	%&= 2\sup_{f\in\mathcal{F}_n}\vert \{\mathcal{R}(f)-\mathcal{R}(f_0)\}-\{\mathcal{R}_n(f)-\mathcal{R}_n(f_0)\}+\lambda\{\kappa(f)-\kappa_n(f)\}\vert\\\notag
%	&= 2\sup_{f\in\mathcal{F}_n}\vert \mathcal{R}(f)-\mathcal{R}_n(f)+\lambda\{\kappa(f)-\kappa_n(f)\}\vert\\\notag
%	&+\mathcal{R}(f)-\mathcal{R}(f_0)+\lambda\{\kappa(f)-\kappa(f_0)\},
\end{align*}
where the second inequality holds by the the fact that $\hat{f}^\lambda_n$ satisfies $\mathcal{R}_n^\lambda(f)\ge\mathcal{R}_n^\lambda(\hat{f}^\lambda_n)$ for any $f\in\mathcal{F}_n$. Since the inequality holds for any $f\in\mathcal{F}_n$, we have
\begin{align*}
	\mathbb{E}\{\mathcal{R}(\hat{f}^\lambda_n)-\mathcal{R}(f_0)\}&\le\mathbb{E}\{\mathcal{R}^\lambda(\hat{f}^\lambda_n)-2\mathcal{R}^\lambda_n(\hat{f}^\lambda_n)+\mathcal{R}^\lambda(f_0)\}+2\inf_{f\in\mathcal{F}_n}\{\mathcal{R}^\lambda(f)-\mathcal{R}^\lambda(f_0)\}.
\end{align*}
This completes the proof. $\hfill \Box$

\subsection*{Proof of Theorem  \ref{non-asymp}}
The proof is straightforward by consequences of Theorem \ref{stoerr} and Corollary \ref{cor1}.

For any $N\in\mathbb{N}^+$, let $\mathcal{F}_n:=\mathcal{F}_{\mathcal{D},\mathcal{W},\mathcal{U},\mathcal{S},\mathcal{B},\mathcal{B}^\prime}$ be the ReQU activated neural networks $f:\mathcal{X}\times(0,1)\to\mathbb{R}$ with depth $\mathcal{D}\le2N-1$, width $\mathcal{W}\le12N^d$, number of neurons $\mathcal{U}\le15N^{d+1}$, number of parameters $\mathcal{S}\le24N^{d+1}$ and satisfying $\mathcal{B}\ge\Vert f_0\Vert_{C^0}$ and $\mathcal{B}^\prime\ge\Vert f_0\Vert_{C^1}$.  Then we would compare the stochastic error bounds $8602{\mathcal{U}\mathcal{S}}$ and $5796{\mathcal{D}\mathcal{S}(\mathcal{D}+\log_2U)}$. By simple math it can be shown that ${\mathcal{D}\mathcal{S}}(\mathcal{D}+\log_2U)=\mathcal{O}(dN^{d+3})$ and $\mathcal{U}\mathcal{S}=\mathcal{O}(N^{2d+2})$. Since $d\ge1$, then we choose apply the upper bound $\mathcal{D}\mathcal{S}(\mathcal{D}+\log_2U)$ in Theorem \ref{stoerr} to get a excess risk bound with lower order in terms of $N$. This completes the proof. $\hfill \Box$

\subsection*{Proof of Lemma \ref{lemmader}}

	Let $\sigma_1(x)=\max\{0,x\}$ and $\sigma_2(x)=\max\{0,x\}^2$ denote the ReLU and ReQU  activation functions respectively. Let $(d_0,d_1,\ldots,d_{\mathcal{D}+1})$ be vector of the width (number of neurons) of each layer in the original ReQU network where $d_0=d+1$ and $d_{\mathcal{D}+1}=1$ in our problem. We let $f^{(i)}_j$ be the function (subnetwork of the ReQU network) from $\mathcal{X}\times(0,1)\subset\mathbb{R}^{d+1}$ to $\mathbb{R}$ which takes $(X,\xi)=(x_1,\ldots,x_d,x_{d+1})$ as input and outputs the $j$-th neuron of the $i$-th layer for  $j=1,\ldots,d_i$ and $i=1,\ldots,\mathcal{D}+1$.
	
	We next construct iteratively ReLU-ReQU activated subnetworks to compute $(\p f^{(i)}_1,\ldots,f^{(i)}_{d_i})$ for $i=1,\ldots,\mathcal{D}+1$, i.e., the partial derivatives of the original ReQU subnetworks step by step. We illustrate the details of the construction of the  ReLU-ReQU subnetworks for the first two layers ($i=1,2$) and the last layer $(\i=\mathcal{D}+1)$ and apply induction for layers $i=3,\ldots,\mathcal{D}$.
	Note that the derivative of ReQU activation function is $\sigma^\prime_2(x)=2\sigma_1(x)$, then when $i=1$ for any $j=1,\ldots,d_1$,
	\begin{align}\label{pderlayer1}
		\p f^{(1)}_j=\p \sigma_2\Big(\sum_{i=1}^{d+1}w^{(1)}_{ji}x_i+b_j^{(1)}\Big)=2\sigma_1\Big(\sum_{i=1}^{d+1}w^{(1)}_{ji}x_i+b_j^{(1)}\Big)\cdot w_{j,d+1}^{(1)},
	\end{align}
	where we denote $w^{(1)}_{ji}$ and $b_j^{(1)}$ by the corresponding weights and bias in $1$-th layer of the original ReQU network and with a little bit abuse of notation we view $x_{d+1}$ as the argument $\tau$ and calculate its partial derivative. Now we intend to construct a 4 layer (2 hidden layers) ReLU-ReQU network with width $(d_0,3d_1,10d_1,2d_1)$ which takes $(X,\xi)=(x_1,\ldots,x_d,x_{d+1})$ as input and outputs
	$$(f^{(1)}_1,\ldots,f^{(1)}_{d_1},\p f^{(1)}_1,\ldots,\p f^{(1)}_{d_1})\in\mathbb{R}^{2d_1}.$$
	Note that the output of such network contains all the quantities needed to calculated  $(\p f^{(2)}_1,\ldots,\p f^{(2)}_{d_2})$, and the process of construction  can be continued iteratively and the induction proceeds. In the firstly hidden layer, we can obtain $3d_1$ neurons $$(f^{(1)}_1,\ldots,f^{(1)}_{d_1},\vert w^{(1)}_{1,d_0}\vert ,\ldots,\vert w^{(1)}_{d_1,d_0}\vert ,\sigma_1(\sum_{i=1}^{d_0}w^{(1)}_{1i}x_i+b^{(1)}_1),\ldots,\sigma_1(\sum_{i=1}^{d_0}w^{(1)}_{d_1i}x_i+b^{(1)}_{d_1})),$$
	with weight matrix $A^{(1)}_1$ having $2d_0d_1$ parameters, bias vector $B^{(1)}_1$ and activation function vector $\Sigma_1$ being
	\begin{align*}A^{(1)}_1=\left[
		\begin{array}{ccccc}
			w^{(1)}_{1,1} &w^{(1)}_{1,2} & \cdots& \cdots&w^{(1)}_{1,d_0} \\
			w^{(1)}_{2,1} &w^{(1)}_{2,2} &\cdots& \cdots&w^{(1)}_{2,d_0} \\
			\ldots& \ldots&\ldots& \ldots&\ldots\\
			w^{(1)}_{d_1,1} &w^{(1)}_{d_1,2} &\cdots&\cdots&w^{(1)}_{d_1,d_0} \\
			0 & 0& 0&0 & 0\\
			\ldots& \ldots&\ldots& \ldots&\ldots\\
			0 & 0& 0&0 & 0\\
			w^{(1)}_{1,1} &w^{(1)}_{1,2} & \cdots& \cdots&w^{(1)}_{1,d_0} \\
			w^{(1)}_{2,1} &w^{(1)}_{2,2} &\cdots& \cdots&w^{(1)}_{2,d_0} \\
			\ldots& \ldots&\ldots& \ldots&\ldots\\
			w^{(1)}_{d_1,1} &w^{(1)}_{d_1,2} &\cdots&\cdots&w^{(1)}_{d_1,d_0} \\
		\end{array}\right]\in\mathbb{R}^{3d_1\times d_0},\quad B^{(1)}_1=\left[\begin{array}{c}
	b^{(1)}_1\\
	b^{(1)}_2\\
	\ldots\\
	b^{(1)}_{d_1}\\
	\vert w^{(1)}_{1,d_0}\vert \\
	\vert w^{(1)}_{2,d_0}\vert \\
	\ldots\\
	\vert w^{(1)}_{d_1,d_0}\vert \\
	b^{(1)}_1\\
	b^{(1)}_2\\
	\ldots\\
	b^{(1)}_{d_1}\\
	\end{array}\right]\in\mathbb{R}^{3d_1},\quad\Sigma^{(1)}_1=\left[\begin{array}{c}
	\sigma_2\\
	\ldots\\
	\sigma_2\\
	\sigma_1\\
	\ldots\\
	\sigma_1\\
	\sigma_1\\
	\ldots\\
	\sigma_1\\
\end{array}\right],
	\end{align*}
	where the first $d_1$ activation functions of $\Sigma_1$ are chosen to be $\sigma_2$ and others $\sigma_1$.
	In the second hidden layer, we can obtain $10d_1$ neurons. The first $2d_1$ neurons of the second hidden layer (or the third layer) are
	$$(\sigma_1(f^{(1)}_1),\sigma_1(-f^{(1)}_1)),\ldots,\sigma_1(f^{(1)}_{d_1}),\sigma_1(f^{(1)}_{d_1})),$$
	which intends to implement identity map such that $(f^{(1)}_1,\ldots,f^{(1)}_{d_1})$ can be kept and outputted in the next layer since identity map can be realized by $x=\sigma_1(x)-\sigma_1(-x)$. The first $8d_1$ neurons of the second hidden layer (or the third layer) are
	\begin{align*}
		\left[\begin{array}{c}
			\sigma_2(w^{(1)}_{1,d_0}+\sigma_1(\sum_{i=1}^{d_0}w^{(1)}_{1i}x_i+b^{(1)}_{1}))\\
			\sigma_2(w^{(1)}_{1,d_0}-\sigma_1(\sum_{i=1}^{d_0}w^{(1)}_{1i}x_i+b^{(1)}_{1}))\\
			\sigma_2(-w^{(1)}_{1,d_0}+\sigma_1(\sum_{i=1}^{d_0}w^{(1)}_{1i}x_i+b^{(1)}_{1})\\
			\sigma_2(-w^{(1)}_{1,d_0}-\sigma_1(\sum_{i=1}^{d_0}w^{(1)}_{1i}x_i+b^{(1)}_{1}))\\
			\ldots\\
			\sigma_2(w^{(1)}_{d_1,d_0}+\sigma_1(\sum_{i=1}^{d_0}w^{(1)}_{d_1i}x_i+b^{(1)}_{d_1}))\\
			\sigma_2(w^{(1)}_{d_1,d_0}-\sigma_1(\sum_{i=1}^{d_0}w^{(1)}_{d_1i}x_i+b^{(1)}_{d_1}))\\
			\sigma_2(-w^{(1)}_{d_1,d_0}+\sigma_1(\sum_{i=1}^{d_0}w^{(1)}_{d_1i}x_i+b^{(1)}_{d_1})\\
			\sigma_2(-w^{(1)}_{d_1,d_0}-\sigma_1(\sum_{i=1}^{d_0}w^{(1)}_{d_1i}x_i+b^{(1)}_{d_1}))\\
		\end{array}\right]\in\mathbb{R}^{8d_1},
	\end{align*}
	which is ready for implementing the multiplications in (\ref{pderlayer1}) to obtain $(\p f^{(1)}_1,\ldots,\p f^{(1)}_{d_1})\in\mathbb{R}^{d_1}$ since
	\begin{align*}
		x\cdot y=\frac{1}{4}\{(x+y)^2-(x-y)^2\}=\frac{1}{4}\{\sigma_2(x+y)+\sigma_2(-x-y)-\sigma_2(x-y)-\sigma_2(-x+y)\}.
	\end{align*}
	In the second hidden layer (the third layer), the bias vector is zero $B^{(1)}_2=(0,\ldots,0)\in\mathbb{R}^{10d_1}$, activation functions vector $$\Sigma^{(1)}_2=(\underbrace{\sigma_1,\ldots,\sigma_1}_{2d_1\ {\rm times}},\underbrace{\sigma_2,\ldots,\sigma_2}_{8d_1\ {\rm times}}),$$
	and the corresponding weight matrix $A^{(1)}_2$  can be formulated correspondingly without difficulty which contains $2d_1+8d_1=10d_1$ non-zero parameters. Then in the last layer, by the identity maps and multiplication operations with weight matrix $A^{(1)}_3$ having $2d_1+4d_1=6d_1$ parameters, bias vector $B^{(1)}_3$ being zeros, we obtain
	$$(f^{(1)}_1,\ldots,f^{(1)}_{d_1},\p f^{(1)}_1,\ldots,\p f^{(1)}_{d_1})\in\mathbb{R}^{2d_1}.$$
	Such ReLU-ReQU neural network has 2 hidden layers (4 layers), $15d_1$ hidden neurons,
	$2d_0d_1+3d_1+10d_1+6d_1=2d_0d_1+19d_1$ parameters and its width is $(d_0,3d_1,10d_1,2d_1)$. It worth noting that the ReLU-ReQU activation functions do not apply to the last layer since the construction here is for a single network.  When we are combining two consecutive subnetworks into one long neural network, the ReLU-ReQU activation functions should apply to the last layer of the first subnetwork. Hence, in the construction of the whole big network, the last layer of the subnetwork here should output $4d_1$ neurons
	\begin{align*}
		&(\sigma_1(f^{(1)}_1),\sigma_1(-f^{(1)}_1)\ldots,\sigma_1(f^{(1)}_{d_1}),\sigma_1(-f^{(1)}_{d_1}),\\
		&\qquad\sigma_1(\p f^{(1)}_1),\sigma_1(-\p f^{(1)}_1)\ldots,\sigma_1(\p f^{(1)}_{d_1}),\sigma_1(-\p f^{(1)}_{d_1}))\in\mathbb{R}^{4d_1},
	\end{align*}
	to keep  $(f^{(1)}_1,\ldots,f^{(1)}_{d_1},\p f^{(1)}_1,\ldots,\p f^{(1)}_{d_1})$ in use in the next subnetwork. Then for this ReLU-ReQU neural network, the weight matrix $A^{(1)}_3$ has $2d_1+8d_1=10d_1$ parameters, the bias vector $B^{(1)}_3$ is zeros and the activation functions vector $\Sigma^{(1)}_3$ has all $\sigma_1$ as elements. And such  ReLU-ReQU neural network has 2 hidden layers (4 layers), $17d_1$ hidden neurons,
	$2d_0d_1+3d_1+10d_1+10d_1=2d_0d_1+23d_1$ parameters and its width is $(d_0,3d_1,10d_1,4d_1)$.

 Now we consider the second step, for any $j=1,\ldots,d_2$,
	\begin{align}\label{pderlayer2}
		\p f^{(2)}_j=\p \sigma_2\Big(\sum_{i=1}^{d_1}w^{(2)}_{ji}f^{(1)}_{i}+b_j^{(2)}\Big)=2\sigma_1\Big(\sum_{i=1}^{d_1}w^{(2)}_{ji}f^{(1)}_i+b_j^{(2)}\Big)\cdot \sum_{i=1}^{d_1}w_{j,i}^{(2)}\p f^{(1)}_i,
	\end{align}
	where $w^{(2)}_{ji}$ and $b_j^{2)}$ are defined correspondingly as the weights and bias in $2$-th layer of the original ReQU network.
	By the previous constructed subnetwork, we can start with its outputs
		\begin{align*}
		&(\sigma_1(f^{(1)}_1),\sigma_1(-f^{(1)}_1)\ldots,\sigma_1(f^{(1)}_{d_1}),\sigma_1(-f^{(1)}_{d_1}),\\
		&\qquad\sigma_1(\p f^{(1)}_1),\sigma_1(-\p f^{(1)}_1)\ldots,\sigma_1(\p f^{(1)}_{d_1}),\sigma_1(-\p f^{(1)}_{d_1}))\in\mathbb{R}^{4d_1},
	\end{align*}
	as the inputs of the second subnetwork we are going to build.  In the firstly hidden layer of the second subnetwork, we can obtain $3d_2$ neurons
	\begin{align*}
		&\Big(f^{(2)}_1,\ldots,f^{(2)}_{d_2},\vert \sum_{i=1}^{d_1}w_{1,i}^{(2)}\p f^{(1)}_i\vert ,\ldots,\vert \sum_{i=1}^{d_1}w_{d_2,i}^{(2)}\p f^{(1)}_i\vert,\\
		&\qquad\qquad\sigma_1(\sum_{i=1}^{d_1}w^{(2)}_{1i}f^{(1)}_i+b^{(1)}_1),\ldots,\sigma_1(\sum_{i=1}^{d_1}w^{(2)}_{d_2i}f^{(1)}_i+b^{(2)}_{d_2})\Big),
	\end{align*}
	with weight matrix $A^{(2)}_1\in\mathbb{R}^{4d_1\times3d_2}$ having $6d_1d_2$ non-zero parameters, bias vector $B^{(2)}_1\in\mathbb{R}^{3d_2}$ and activation functions vector $\Sigma^{(2)}_1=\Sigma^{(1)}_1$.
	Similarly, the second hidden layer can be constructed to have $10d_2$ neurons with  weight matrix $A^{(2)}_2\in\mathbb{R}^{3d_2\times10d_2}$ having $2d_2+8d_2=10d_2$ non-zero parameters, zero bias vector $B^{(2)}_1\in\mathbb{R}^{10d_2}$ and activation functions vector $\Sigma^{(2)}_2=\Sigma^{(1)}_2$. The second hidden layer here serves exactly the same as that in the first subnetwork, which intends to implement the identity map for
	$$(f^{(2)}_1,\ldots,f^{(2)}_{d_2}),$$
	and implement the multiplication in (\ref{pderlayer2}). Similarly, the last layer can also be constructed as that in the first subnetwork, which outputs
		\begin{align*}
			&(\sigma_1(f^{(2)}_1),\sigma_1(-f^{(2)}_1)\ldots,\sigma_1(f^{(2)}_{d_2}),\sigma_1(-f^{(2)}_{d_2}),\\
			&\qquad\sigma_1(\p f^{(2)}_1),\sigma_1(-\p f^{(2)}_1)\ldots,\sigma_1(\p f^{(2)}_{d_2}),\sigma_1(-\p f^{(2)}_{d_2}))\in\mathbb{R}^{4d_2},
		\end{align*}
	with the weight matrix $A^{(2)}_3$ having $2d_2+8d_2=10d_2$ parameters, the bias vector $B^{(2)}_3$ being zeros and the activation functions vector $\Sigma^{(1)}_3$ with elements being  $\sigma_1$. Then the second ReLU-ReQU  subnetwork has 2 hidden layers (4 layers), $17d_2$ hidden neurons,
	$6d_1d_2+3d_2+10d_2+10d_2=6d_1d_2+23d_2$ parameters and its width is $(4d_1,3d_2,10d_2,4d_2)$.

	Then we can continuing this process of construction. For integers $k=3,\ldots,\mathcal{D}$ and for any $j=1,\ldots,d_{k}$,
	\begin{align*}
		\p f^{(k)}_j&=\p \sigma_2\Big(\sum_{i=1}^{d_{k-1}}w^{(k)}_{ji}f^{(k-1)}_{i}+b_j^{(k)}\Big)\\
&=2\sigma_1\Big(\sum_{i=1}^{d_{k-1}}w^{(k)}_{ji}f^{(k-1)}_i+b_j^{(k)}\Big)\cdot \sum_{i=1}^{d_{k-1}}w_{j,i}^{(k)}\p f^{(k-1)}_i,
	\end{align*}
	where $w^{(k)}_{ji}$ and $b_j^{(k)}$ are defined correspondingly as the weights and bias in $k$-th layer of the original ReQU network. We can construct a ReLU-ReQU network taking
	\begin{align*}
		&(\sigma_1(f^{(k-1)}_1),\sigma_1(-f^{(k-1)}_1)\ldots,\sigma_1(f^{(k-1)}_{d_{k-1}}),\sigma_1(-f^{(k-1)}_{d_{k-1}}),\\
		&\qquad\sigma_1(\p f^{(k-1)}_1),\sigma_1(-\p f^{(k-1)}_1)\ldots,\sigma_1(\p f^{(k-1)}_{d_{k-1}}),\sigma_1(-\p f^{(k-1)}_{d_{k-1}}))\in\mathbb{R}^{4d_{k-1}},
	\end{align*}
	as input, and it outputs
		\begin{align*}
		&(\sigma_1(f^{(k)}_1),\sigma_1(-f^{(k)}_1)\ldots,\sigma_1(f^{(k)}_{d_{k}}),\sigma_1(-f^{(k)}_{d_{k}}),\\
		&\qquad\sigma_1(\p f^{(k)}_1),\sigma_1(-\p f^{(k)}_1)\ldots,\sigma_1(\p f^{(k)}_{d_{k}}),\sigma_1(-\p f^{(k)}_{d_{k}}))\in\mathbb{R}^{4d_{k}},
	\end{align*}
	with 2 hidden layers, $17d_k$ hidden neurons, $6d_{k-1}d_k+23d_k$ parameters and its width is $(4d_{k-1},3d_{k},10d_{k},4d_K)$.
	
	Iterate this process until the $k=\mathcal{D}+1$ step, where the last layer of the original ReQU network has only $1$ neurons. That is for the ReQU activated neural network $f\in\mathcal{F}_n=\mathcal{F}_{\mathcal{D},\mathcal{W},\mathcal{U},\mathcal{S},\mathcal{B}}$, the output of the network $f:\mathcal{X}\times(0,1)\to\mathbb{R}$ is a scalar and the partial derivative with respect to $\tau$ is
	\begin{align*}
		\p f=\p \sum_{i=1}^{d_{\mathcal{D}+1}}w^{(\mathcal{D})}_{i}f^{(\mathcal{D})}_i+b^{(\mathcal{D})}=\sum_{i=1}^{d_{\mathcal{D}+1}}w^{(\mathcal{D})}_{i} \p f^{(\mathcal{D})}_i,
	\end{align*}
	where $w^{(\mathcal{D})}_{i}$ and $b^{(\mathcal{D})}$ are the weights and bias parameter in the last layer of the ReQU network.
	The the constructed $\mathcal{D}+1$-th subnetwork taking
		\begin{align*}
		&(\sigma_1(f^{(\mathcal{D})}_1),\sigma_1(-f^{(\mathcal{D})}_1)\ldots,\sigma_1(f^{(\mathcal{D})}_{d_{\mathcal{D}}}),\sigma_1(-f^{(\mathcal{D})}_{d_{\mathcal{D}}}),\\
		&\qquad\sigma_1(\p f^{(\mathcal{D})}_1),\sigma_1(-\p f^{(\mathcal{D})}_1)\ldots,\sigma_1(\p f^{(\mathcal{D})}_{d_{\mathcal{D}}}),\sigma_1(-\p f^{(\mathcal{D})}_{d_{\mathcal{D}}}))\in\mathbb{R}^{4d_{\mathcal{D}}},
	\end{align*}
	as input and  it outputs $\p f^{(\mathcal{D}+1)}=\p f$ which is the partial derivative of the whole ReQU network with respect to its last argument $\tau$ or $x_{d_0}=x_{d+1}$ here. The subnetwork should have 2 hidden layers width $(4d_{\mathcal{D}},2,8,1)$ with $11$ hidden neurons, $4d_{\mathcal{D}}+2+16=4d_{\mathcal{D}}+18$ non-zero parameters.

	Lastly, we combing all the $\mathcal{D}+1$ subnetworks in order to form a big ReLU-ReQU network which takes $(X,\xi)=(x_1,\ldots,x_{d+1})\in\mathbb{R}^{d+1}$ as input and outputs $\p f$ for $f\in\mathcal{F}_n=\mathcal{F}_{\mathcal{D},\mathcal{W}, \mathcal{U},\mathcal{S},\mathcal{B},\mathcal{B}^\prime}$. Recall that here $\mathcal{D},\mathcal{W}, \mathcal{U},\mathcal{S}$ are the depth, width, number of neurons and number of parameters of the ReQU network respectively, and we have $\mathcal{U}=\sum_{i=0}^{\mathcal{D}+1}d_i$ and $\mathcal{S}=\sum_{i=0}^{\mathcal{D}}d_id_{i+1}+d_{i+1}.$ Then the big network has $3\mathcal{D}+3$ hidden layers (totally 3$\mathcal{D}+5$ layers), $d_0+\sum_{i=1}^{\mathcal{D}}17d_{i}+11\le 17\mathcal{U}$ neurons, $2d_0d_1+23d_1+\sum_{i=1}^\mathcal{D}(6d_{i}d_{i+1}+23d_{i+1})+4d_\mathcal{D}+18\le23\mathcal{S}$ parameters and its width is $10\max\{d_1,\ldots,d_\mathcal{D}\}=10\mathcal{W}$. This completes the proof. $\hfill \Box$

\subsection*{Proof of Lemma \ref{lemmapdim}}
Our proof has two parts. In the first part, we follow the idea of the proof of Theorem 6 in \citet{bartlett2019nearly} to prove a somewhat stronger result, where we give the upper bound of the Pseudo dimension of $\mathcal{F}$ in terms of the depth, size and number of neurons of the network. Instead of the VC dimension of ${\rm sign}(\mathcal{F})$ given in \citet{bartlett2019nearly}, our Pseudo dimension  bound is stronger since ${\rm VCdim}({\rm sign}(\mathcal{F}))\le\pdim(\mathcal{F})$. In the second part, based on Theorem 2.2 in \citet{goldberg1995bounding}, we also follow and improve the result in Theorem 8 of \cite{bartlett2019nearly} to give an upper bound of the Pseudo dimension of $\mathcal{F}$ in terms of the size and number of neurons of the network.

\subsubsection*{Part I}
Let $\mathcal{Z}$ denote the domain of the functions $f\in\mathcal{F}$ and let $t\in\mathbb{R}$, we consider a new class of functions
$$\tilde{\mathcal{F}}:=\{\tilde{f}(z,t)=\s(f(z)-t):f\in\mathcal{F}\}.$$
Then it is clear that $\pdim(\mathcal{F})\le\vdim(\tilde{\mathcal{F}})$ and we next bound the VC dimension of $\tilde{\mathcal{F}}$.
Recall that the the total number of parameters (weights and biases) in the neural network implementing functions in $\mathcal{F}$ is $\mathcal{S}$, we let $\theta\in\mathbb{R}^{\mathcal{S}}$ denote the parameters vector of the network $f(\cdot,\theta):\mathcal{Z}\to\mathbb{R}$ implemented in $\mathcal{F}$. And here we intend to derive a bound for
$$K(m):=\Big\vert\{(\s(f(z_1,\theta)-t_1),\ldots,\s(f(z_m,\theta)-t_m)):\theta\in\mathbb{R}^{\mathcal{S}}\}\Big\vert$$
which uniformly hold for all choice of $\{z_i\}_{i=1}^m$ and $\{t_i\}_{i=1}^m$. Note that the maximum of $K(m)$ over all all choice of $\{z_i\}_{i=1}^m$ and $\{t_i\}_{i=1}^m$ is just the growth function of $\tilde{\mathcal{F}}$. To give a uniform bound of $K(m)$, we use the Theorem 8.3 in \citet{anthony1999} as a main tool to deal with the analysis.
\begin{lemma}[Theorem 8.3 in \citet{anthony1999}]\label{polynum}
	Let $p_1,\ldots,p_m$ be polynomials in $n$ variables of degree at most $d$. If $n\le m$, define
	$$K:=\vert\{(\s(p_1(x),\ldots,\s(p_m(x))):x\in\mathbb{R}^n)\}\vert,$$
	i.e. $K$ is the number of possible sign vectors given by the polynomials. Then $K\le2(2emd/n)^n$.
\end{lemma}

Now if we can find a partition $\mathcal{P}=\{P_1,\ldots,P_N\}$ of the parameter domain $\mathbb{R}^{\mathcal{S}}$ such that within each region $P_i$, the functions $f(z_j,\cdot)$ are all fixed polynomials of bounded degree, then $K(m)$ can be bounded via the following sum
\begin{align}\label{Km}
	K(m)\le\sum_{i=1}^N\Big\vert\{(\s(f(z_1,\theta)-t_1),\ldots,\s(f(z_m,\theta)-t_m)):\theta\in P_i\}\Big\vert,
\end{align}
and each term in this sum can be bounded via Lemma \ref{polynum}.
Next, we construct the partition follows the same way as in \citet{bartlett2019nearly} iteratively layer by layer. We define the a sequence of successive refinements $\mathcal{P}_1,\ldots,\mathcal{P}_{\mathcal{D}}$ satisfying the following properties:
\begin{itemize}
	\item [1.] The cardinality $\vert\mathcal{P}_1\vert=1$ and for each $n\in\{1,\ldots,\mathcal{D}\}$,
	$$\frac{\vert\mathcal{P}_{n+1}\vert}{\vert\mathcal{P}_{n}\vert}\le2\Big(\frac{2emk_n(1+(n-1)2^{n-1})}{\mathcal{S}_n}\Big)^{\mathcal{S}_n},$$
	where $k_n$ denotes the number of neurons in the $n$-th layer and $\mathcal{S}_n$ denotes the total number of parameters (weights and biases) at the inputs to units in all the layers up to layer $n$.
	\item [2.] For each $n\in\{1,\ldots,\mathcal{D}\}$, each element of $P$ of $\mathcal{P}_n$, each $j\in\{1,\ldots,m\}$, and each unit $u$ in the $n$-th layer, when $\theta$ varies in $P$, the net input to $u$ is a fixed polynomial function in $\mathcal{S}_n$ variables of $\theta$, of total degree no more than $1+(n-1)2^{n-1}$ (this polynomial may depend on $P,j$ and $u$.)
\end{itemize}
One can define $\mathcal{P}_1=\mathbb{R}^\mathcal{S}$, and it can be verified that $\mathcal{P}_1$ satisfies property 2 above. Note that in our case, for fixed $z_j$ and $t_j$ and any subset $P\subset\mathbb{R}^{\mathcal{S}}$,  $f(z_j,\theta)-t_j$ is a polynomial with respect to $\theta$ with degree the same as that of $f(z_j,\theta)$, which is no more than $1+(\mathcal{D}-1)2^{\mathcal{D}-1}$. Then the construction of $\mathcal{P}_1,\ldots,\mathcal{P}_{\mathcal{D}}$ and its verification for properties 1 and 2 can follow the same way in \citet{bartlett2019nearly}. Finally we obtain a partition $\mathcal{P}_{\mathcal{D}}$ of $\mathbb{R}^\mathcal{S}$ such that for $P\in\mathcal{P}_\mathcal{D}$, the network output in response to any $z_j$ is a fixed polynomial of $\theta\in P$ of degree no more than $1+(\mathcal{D}-1)2^{\mathcal{D}-1}$ (since the last node just outputs its input). Then by Lemma \ref{polynum}
$$\Big\vert\{(\s(f(z_1,\theta)-t_1),\ldots,\s(f(z_m,\theta)-t_m)):\theta\in P\}\Big\vert\le2\Big(\frac{2em(1+(\mathcal{D}-1)2^{\mathcal{D}-1})}{\mathcal{S}_\mathcal{D}}\Big)^{\mathcal{S}_\mathcal{D}}.$$
Besides, by property 1 we have
\begin{align*}
	\vert\mathcal{P}_\mathcal{D}\vert&\le\Pi_{i=1}^\mathcal{D-1}2\Big(\frac{2emk_i(1+(i-1)2^{i-1})}{\mathcal{S}_i}\Big)^{\mathcal{S}_i}.
\end{align*}
Then using (\ref{Km}), and since the sample $z_1,\ldots,Z_m$ are arbitrarily chosen, we have
\begin{align*}
	K(m)&\le\Pi_{i=1}^\mathcal{D}2\Big(\frac{2emk_i(1+(i-1)2^{i-1})}{\mathcal{S}_i}\Big)^{\mathcal{S}_i}\\
	&\le2^\mathcal{D}\Big(\frac{2em\sum k_i(1+(i-1)2^{i-1})}{\sum\mathcal{S}_i}\Big)^{\sum\mathcal{S}_i}\\
	&\le\Big(\frac{4em(1+(\mathcal{D}-1)2^{\mathcal{D}-1})\sum k_i}{\sum\mathcal{S}_i}\Big)^{\sum\mathcal{S}_i}\\
	&\le\Big(4em(1+(\mathcal{D}-1)2^{\mathcal{D}-1})\Big)^{\sum\mathcal{S}_i},
\end{align*}
where the second inequality follows from weighted arithmetic and geometric means inequality, the third holds since $\mathcal{D}\le\sum\mathcal{S}_i$ and the last holds since $\sum k_i\le\sum \mathcal{S}_i$. Since $K(m)$ is the growth function of $\tilde{\mathcal{F}}$, we have
\begin{align*}
	2^{\pdim(\mathcal{F})}\le2^{\vdim(\tilde{\mathcal{F}})}\le K(\vdim(\tilde{\mathcal{F}}))\le2^\mathcal{D}\Big(\frac{2emR\cdot\vdim(\tilde{\mathcal{F}})}{\sum\mathcal{S}_i}\Big)^{\sum\mathcal{S}_i}
\end{align*}
where $R:=\sum_{i=1}^\mathcal{D} k_i(1+(i-1)2^{i-1})\le \mathcal{U}+\mathcal{U}(\mathcal{D}-1)2^{\mathcal{D}-1}.$
Since $\mathcal{U}>0$ and $2eR\ge16$, then by Lemma 16 in \citet{bartlett2019nearly} we have
$$\pdim(\mathcal{F})\le\mathcal{D}+(\sum_{i=1}^n\mathcal{S}_i)\log_2(4eR\log_2(2eR)).$$
Note that $\sum_{i=1}^\mathcal{D}\mathcal{S}_i\le\mathcal{D}\mathcal{S}$ and $\log_2(R)\le\log_2(\mathcal{U}\{1+(\mathcal{D}-1)2^{\mathcal{D}-1}\})\le\log_2(\mathcal{U})+2\mathcal{D}$, then we have
$$\pdim(\mathcal{F})\le \mathcal{D}+\mathcal{D}\mathcal{S}(4\mathcal{D}+2\log_2\mathcal{U}+6)\le7\mathcal{D}\mathcal{S}(\mathcal{D}+\log_2\mathcal{U}))$$
for some universal constant $c>0$.

\subsubsection*{Part II}
We first list  Theorem 2.2 in \citet{goldberg1995bounding}.
\begin{lemma}[Theorem 2.2 in \citet{goldberg1995bounding}]\label{bool}
	Let $k,n$ be positive integers and $f:\mathbb{R}^n\times\mathbb{R}^k\to\{0,1\}$ be a function that can be expressed as a Boolean formula containing $s$ distinct atomic predicates where each atomic predicate is a polynomial inequality or equality in $k+n$ variables of degree at most $d$.  Let $\mathcal{F}=\{f(\cdot,w):w\in\mathbb{R}^k\}.$ Then $\vdim(\mathcal{F})\le2k\log_2(8eds)$.
\end{lemma}
 Suppose the functions in $f\in\mathcal{F}$ are implemented by ReLU-REQU neural networks with $\mathcal{S}$ parameters (weights and bias) and $\mathcal{U}$ neurons. The activation function of $f\in\mathcal{F}$ is piecewise polynomial of degree at most $2$ with $2$ pieces. As in Part I of the proof, let $\mathcal{Z}$ denote the domain of the functions $f\in\mathcal{F}$ and let $t\in\mathbb{R}$, we consider the class of functions
$$\tilde{\mathcal{F}}:=\{\tilde{f}(z,t)=\s(f(z)-t):f\in\mathcal{F}\}.$$
Since the outputs of functions in $\tilde{\mathcal{F}}$ are 0 or 1, to apply above lemma, we intend to show that the function in $\tilde{\mathcal{F}}$ as Boolean functions consisting of no more than $2\cdot3^\mathcal{U}$ atomic predicates with each being a polynomial inequality of degree at most $3\cdot2^\mathcal{U}$.  We topologically sort the neurons of the network since the neural network graph is acyclic. Let $u_i$ be the $i$-th neuron in the topological ordering for $i=1,\ldots,\mathcal{U}+1$. Note that the input to each neuron $u$ comes from one of the 2 pieces  of the activation function ReLU or ReQU, then we call "$u_i$ is in the state $j$" if the input of $u_i$ lies in the $j$-th piece for $i\in\{1,\ldots,\mathcal{U}+1\}$ and $j\in\{1,2\}$. For $u_1$ and $j\in\{1,2\}$, the predicate “$u_1$ is in state$j$” is a single atomic predicate thus the state of $u_1$ can be expressed as a function of 2 atomic predicates.  Given  that $u_1$ is in a certain state, the state of $u_2$ can be decided by 2 atomic predicates, which are polynomial inequalities of degree at most $2\times2+1$. Hence the state of
$u_2$ can be determined using $2+2^2$ atomic predicates, each of which is a polynomial of
degree no more than $2\times2+1$. By induction, the state of $u_i$ is decided using $2(1+2)^{i-1}$ atomic predicates, each of which is a polynomial of degree at most $2^{i-1}+\sum_{j=0}^{i-1}2^j$. Then the state of all neurons can be decided using no more than $3^{\mathcal{U}+1}$ atomic predicates, each of which is a polynomial of degree at most $2^{\mathcal{U}}+\sum_{j=0}^{\mathcal{U}}2^j\le3\cdot2^\mathcal{U}$. Then by Lemma \ref{bool}, an upper bound for $\vdim(\tilde{\mathcal{F}})$ is $2\mathcal{S}\log_2(8e\cdot3^{\mathcal{U}+1}\cdot3\cdot2^{\mathcal{U}})=2\mathcal{S}[\log_2(6^{\mathcal{U}})+\log_2(72e)]\le22\mathcal{S}\mathcal{U}$. Since $\pdim(\mathcal{F})\le\vdim(\tilde{\mathcal{F}})$, then the upper bounds hold also for $\pdim(\mathcal{F})$,
$$\pdim(\mathcal{F})\le22\mathcal{S}\mathcal{U},$$
 which completes the proof. $\hfill \Box$

\subsection*{Proof of Theorem \ref{stoerr}}
{\color{black}
Recall that the stochastic error is
\begin{align}\label{sto1}
	\mathbb{E}\{\mathcal{R}^\lambda(\hat{f}^\lambda_n)-2\mathcal{R}^\lambda_n(\hat{f}^\lambda_n)+\mathcal{R}^\lambda(f_0)\}
\end{align}

Let $S=\{Z_i=(X_i,Y_i,\xi_i)\}_{i=1}^n$ be the sample used to estimate $\hat{f}^\lambda_n$ from the distribution $Z=(X,Y,\xi)$. And let $S^\prime=\{Z^\prime_i=(X^\prime_i,Y^\prime_i,\xi^\prime_i)\}_{i=1}^n$ be another sample independent of $S$. Define

\begin{gather*}
		g_1(f,X_i)=\mathbb{E}\big\{\rho_\xi(Y_i-f(X_i,\xi))-\rho_\xi(Y_i-f_0(X_i,\xi))\mid X_i\big\}\\
g_2(f,X_i)=\mathbb{E}\big[\lambda\max\{-\frac{\partial}{\partial\tau}f(X_i,\xi),0\}-\lambda\max\{-\frac{\partial}{\partial\tau} f_0(X_i,\xi),0\}\mid X_i\big]\\
	g(f,X_i)=g_1(f,X_i)+g_2(f,X_i)
\end{gather*}
for any $f$ and sample $X_i$. Note the the empirical risk minimizer $\hat{f}^\lambda_n$ depends on the sample $S$, and the stochastic error is
\begin{align*}%\label{sto1}
	\mathbb{E}\{\mathcal{R}^\lambda(\hat{f}^\lambda_n)-2\mathcal{R}^\lambda_n(\hat{f}^\lambda_n)+\mathcal{R}^\lambda(f_0)\}=\mathbb{E}_S\Big(\frac{1}{n}\sum_{i=1}^{n}\bigg[\mathbb{E}_{S^\prime}\big\{g(\hat{f}^\lambda_n,X^\prime_i)\big\}-2g(\hat{f}^\lambda_n,X_i)\bigg]\Big)
\end{align*}

Given a positive number $\delta>0$ and a $\delta$-uniform covering of $\mathcal{F}_n$, we denote the centers of the balls by $f_j,j=1,\ldots,\mathcal{N}_{2n}$, where $\mathcal{N}_{2n}=\mathcal{N}_{2n}(\delta,\mathcal{F}_n,\Vert\cdot\Vert_\infty)$ is the uniform covering number with radius $\delta$ ($\delta\le\mathcal{B}$) under the norm $\Vert\cdot\Vert_\infty$, where the detailed definition of $\mathcal{N}_{2n}(\delta,\mathcal{F}_n,\Vert\cdot\Vert_\infty)$ can be found in (\ref{ucover}) in Appendix \ref{apdx_b}. By the definition of covering, there exists a (random) $j^*$ such that $\Vert\hat{f}^\lambda_n(x,\tau)-f_{j^*}(x,\tau)\Vert_\infty\le\delta$ for all $(x,\tau)\in\{(X_1,\xi_1),\ldots,(X_n,\xi_n),(X^\prime_1,\xi_1^\prime),\ldots,(X^\prime_n,\xi_n^\prime)\}$. Recall that
$g_1(f,X_i)=\mathbb{E}\big\{\rho_\xi(Y_i-f(X_i,\xi))-\rho_\xi(Y_i-f_0(X_i,\xi))\mid X_i\big\}$ and $\rho_\tau$ is 1-Lipschitz for any $\tau\in(0,1)$, then for $i=1,\ldots,n$ we have
\begin{align*}
\vert g_1(\hat{f}^\lambda_n,X_i)-g_1(f_{j^*},X_i)\vert\le\delta \quad{\rm and}\quad \vert \mathbb{E}_{S^\prime}g_1(\hat{f}^\lambda_n,X^\prime_i)-\mathbb{E}_{S^\prime}g_1(f_{j^*},X^\prime_i)\vert\le\delta.
\end{align*}
Let $\mathcal{F}^\prime_n=\{\frac{\partial}{\partial\tau} f:f\in\mathcal{F}_{n}\}$ denote the class of first order partial derivatives of functions in $\mathcal{F}_n$.
Similarly, given any positive number $\delta>0$ and a $\delta$-uniform covering of $\mathcal{F}^\prime_n$, we denote the centers of the balls by $f^\prime_k,k=1,\ldots,\mathcal{N}^\prime_{2n}$, where $\mathcal{N}^\prime_{2n}=\mathcal{N}_{2n}(\delta,\mathcal{F}^\prime_n,\Vert\cdot\Vert_\infty)$. By the definition of covering, there exists a (random) $k^*$ such that $\Vert\frac{\partial}{\partial\tau}\hat{f}^\lambda_n(x,\tau)-f^\prime_{k^*}(x,\tau)\Vert_\infty\le\delta$ for all $(x,\tau)\in\{(X_1,\xi_1),\ldots,(X_n,\xi_n),(X^\prime_1,\xi_1^\prime),\ldots,(X^\prime_n,\xi_n^\prime)\}$. And recall
$g_2(f,X_i)=\mathbb{E}[\lambda\max\{-\frac{\partial}{\partial\tau}f(X_i,\xi),0\}-\lambda\max\{-\frac{\partial}{\partial\tau} f_0(X_i,\xi),0\}\mid X_i]$ and $\max$ function is 1-Lipschitz, then for $i=1,\ldots,n$ we have
\begin{align*}
	\vert g_2(\hat{f}^\lambda_n,X_i)-g_2(f_{k^*},X_i)\vert\le\lambda\delta \quad{\rm and}\quad \vert \mathbb{E}_{S^\prime}g_2(\hat{f}^\lambda_n,X^\prime_i)-\mathbb{E}_{S^\prime}g_2(f_{k^*},X^\prime_i)\vert\le\lambda\delta.
\end{align*}
Combining above inequalities, we have
\begin{align}%\label{sto1}
	\notag
	\mathbb{E}\{\mathcal{R}^\lambda(\hat{f}^\lambda_n)-2\mathcal{R}^\lambda_n(\hat{f}^\lambda_n)+\mathcal{R}^\lambda(f_0)\}&\le\mathbb{E}_S\Big(\frac{1}{n}\sum_{i=1}^{n}\big[\mathbb{E}_{S^\prime}\big\{g_1(f_{j^*},X^\prime_i)\big\}-2g_1(f_{j^*},X_i)\big]\Big)\\\notag
	&+\mathbb{E}_S\Big(\frac{1}{n}\sum_{i=1}^{n}\big[\mathbb{E}_{S^\prime}\big\{g_2(f_{k^*},X^\prime_i)\big\}-2g_2(f_{k^*},X_i)\big]\Big)\\ \label{ine1}
	&+3(1+\lambda)\delta.
\end{align}
Now we consider giving upper bounds for $\mathbb{E}_S(\frac{1}{n}\sum_{i=1}^{n}[\mathbb{E}_{S^\prime}\{g_1(f_{j^*},X^\prime_i)\}-2g_1(f_{j^*},X_i)])$ and $\mathbb{E}_S(\frac{1}{n}\sum_{i=1}^{n}[\mathbb{E}_{S^\prime}\{g_2(f_{k^*},X^\prime_i)\}-2g_2(f_{k^*},X_i)])$ respectively. Recall that it is assumed $\Vert f_0\Vert_\infty\le \mathcal{B}$ and $\Vert f\Vert_\infty\le\mathcal{B}$ for any $f\in\mathcal{F}_n$, then $\vert g_1(f,X_i)\vert\le\Vert f-f_0\Vert_\infty\le2\mathcal{B}$ and $\sigma^2_{g_1}(f):={\rm Var}(g_1(f,X_i))\le\mathbb{E}\{g_1(f,X_i)^2\}\le2\mathcal{B}\mathbb{E}\{g_1(f,X_i)\}$. Then for each $f_j$ and any $t>0$, let $u=t/2+{\sigma_{g_1}^2(f_j)}/(4\mathcal{B})$, by applying the Bernstein inequality we have

\begin{align*}
	&P\Big\{\frac{1}{n}\sum_{i=1}^{n}[\mathbb{E}_{S^\prime}\{g_1(f_{j^*},X^\prime_i)\}-2g_1(f_{j^*},X_i)]>t\Big\}\\
	=&P\Big\{\mathbb{E}_{S^\prime} \{g_1(f_j,X_i^\prime)\}-\frac{2}{n}\sum_{i=1}^ng_1(f_j,X_i)>t\Big\}\\
	=&P\Big\{\mathbb{E}_{S^\prime} \{g_1(f_j,X_i^\prime)\}-\frac{1}{n}\sum_{i=1}^ng_1(f_j,X_i)>\frac{t}{2}+\frac{1}{2}\mathbb{E}_{S^\prime} \{g_1(f_j,X_i^\prime)\}\Big\}\\
	\leq& P\Big\{\mathbb{E}_{S^\prime} \{g_1(f_j,X_i^\prime)\}-\frac{1}{n}\sum_{i=1}^ng_1(f_j,X_i)>\frac{t}{2}+\frac{1}{2}\frac{\sigma_{g_1}^2(f_j)}{4\mathcal{B}}\Big\}\\
	\leq& \exp\Big( -\frac{nu^2}{2\sigma_{g_1}^2(f_j)+8u\mathcal{B}/3}\Big)\\
	\leq& \exp\Big( -\frac{nu^2}{8u\mathcal{B}+8u\mathcal{B}/3}\Big)\\
	\leq& \exp\Big( -\frac{1}{8+8/3}\cdot\frac{nu}{\mathcal{B}}\Big)\\
	\leq& \exp\Big( -\frac{1}{16+16/3}\cdot\frac{nt}{\mathcal{B}}\Big).
\end{align*}
This leads to a tail probability bound of $\sum_{i=1}^{n}[\mathbb{E}_{S^\prime}\{g_1(f_{j^*},X^\prime_i)\}-2g_1(f_{j^*},X_i)]/n$, which is
$$P\Big\{\frac{1}{n}\sum_{i=1}^n\big[\mathbb{E}_{S^\prime}\{g_1(f_{j^*},X^\prime_i)\}-2g_1(f_{j^*},X_i)\big]>t\Big\}\leq 2\mathcal{N}_{2n}\exp\Big( -\frac{1}{22}\cdot\frac{nt}{\mathcal{B}}\Big).$$
Then for $a_n>0$,
\begin{align*}
	\mathbb{E}_S\Big[ \frac{1}{n}\sum_{i=1}^nG(f_{j^*},X_i)\Big]\leq& a_n +\int_{a_n}^\infty P\Big\{\frac{1}{n}\sum_{i=1}^nG(f_{j^*},X_i)>t\Big\} dt\\
	\leq& a_n+ \int_{a_n}^\infty 2\mathcal{N}_{2n}\exp\Big( -\frac{1}{22}\cdot\frac{nt}{\mathcal{B}}\Big) dt\\
	\leq& a_n+ 2\mathcal{N}_{2n}\exp\Big( -a_n\cdot\frac{n}{22\mathcal{B}}\Big)\frac{22\mathcal{B}}{n}.
\end{align*}
Choosing $a_n=\log(2\mathcal{N}_{2n})\cdot{22\mathcal{B}}/{n}$, we have
\begin{equation} \label{bound3}
	\mathbb{E}_S\Big[ \frac{1}{n}\sum_{i=1}^nG(f_{j^*},X_i)\Big]\leq \frac{22\mathcal{B}(\log(2\mathcal{N}_{2n})+1)}{n}.
	\end{equation}

Similarly, for the derivative of the functions in $\mathcal{F}_n$, it is assumed $\Vert \p f_0\Vert_\infty\le \mathcal{B}^\prime$ and $\Vert f\Vert_\infty\le\mathcal{B}^\prime$ for any $f\in\mathcal{F}^\prime_n$, then $\vert g_2(f,X_i)\vert\le\lambda\Vert f-\p f_0\Vert_\infty\le2\lambda\mathcal{B}^\prime$ and $\sigma^2_{g_2}(f):={\rm Var}(g_2(f,X_i))\le\mathbb{E}\{g_2(f,X_i)^2\}\le2\lambda\mathcal{B}^\prime\mathbb{E}\{g_2(f,X_i)\}$. Then for each $f_k$ and any $t>0$, let $u=t/2+{\sigma_{g_2}^2(f_k)}/(4\lambda\mathcal{B}^\prime)$, by applying the Bernstein inequality we have

\begin{align*}
	&P\Big\{\frac{1}{n}\sum_{i=1}^{n}[\mathbb{E}_{S^\prime}\{g_2(f_{k},X^\prime_i)\}-2g_2(f_{k},X_i)]>t\Big\}
	\leq \exp\Big( -\frac{1}{22}\cdot\frac{nt}{\lambda\mathcal{B}^\prime}\Big),
\end{align*}
which leads to
$$P\Big\{\frac{1}{n}\sum_{i=1}^n\big[\mathbb{E}_{S^\prime}\{g_2(f_{k^*},X^\prime_i)\}-2g_2(f_{k^*},X_i)\big]>t\Big\}\leq 2\mathcal{N}^\prime_{2n}\exp\Big( -\frac{1}{22}\cdot\frac{nt}{\lambda\mathcal{B}^\prime}\Big),$$
and
\begin{equation} \label{bound4}
	\mathbb{E}_S\Big[\frac{1}{n}\sum_{i=1}^n\big[\mathbb{E}_{S^\prime}\{g_2(f_{k^*},X^\prime_i)\}-2g_2(f_{k^*},X_i)\big]\Big]\leq \frac{22\lambda\mathcal{B}^\prime(\log(2\mathcal{N}^\prime_{2n})+1)}{n}.
\end{equation}

 Setting $\delta=1/n$ and combining (\ref{ine1}), (\ref{bound3}) and (\ref{bound4}), we have
 \begin{align}\notag
 &\mathbb{E}\{\mathcal{R}^\lambda(\hat{f}^\lambda_n)-2\mathcal{R}^\lambda_n(\hat{f}^\lambda_n)+\mathcal{R}^\lambda(f_0)\}\\
 &\le69\big[\mathcal{B}\log(\mathcal{N}_{2n}(\frac{1}{n},\mathcal{F}_n,\Vert\cdot,\Vert_\infty))+\lambda\mathcal{B}^\prime\log(\mathcal{N}^\prime_{2n}(\frac{1}{n},\mathcal{F}^\prime_n,\Vert\cdot,\Vert_\infty))\big]n^{-1}.
 \end{align}

Then by Lemma \ref{c2p} in Appendix \ref{apdx_b}, we can further bound the covering number by the Pseudo dimension. More exactly, for $n\ge \pdim(\mathcal{F}_n)$ and any $\delta>0$, we have
\begin{align*}
	\log(\mathcal{N}_{2n}(\delta,\mathcal{F}_n,\Vert\cdot\Vert_\infty))\le\pdim(\mathcal{F}_n)\log\Big(\frac{en\mathcal{B}}{\delta\pdim(\mathcal{F}_n)}\Big),
\end{align*}
and for $n\ge \pdim(\mathcal{F}^\prime_n)$ and any $\delta>0$, we have
\begin{align*}
	\log(\mathcal{N}_{2n}(\delta,\mathcal{F}^\prime_n,\Vert\cdot\Vert_\infty))\le\pdim(\mathcal{F}^\prime_n)\log\Big(\frac{en\mathcal{B}^\prime}{\delta\pdim(\mathcal{F}^\prime_n)}\Big).
\end{align*}
Combining the upper bounds of the covering numbers, we have
 \begin{align}\notag
	&\mathbb{E}\{\mathcal{R}^\lambda(\hat{f}^\lambda_n)-2\mathcal{R}^\lambda_n(\hat{f}^\lambda_n)+\mathcal{R}^\lambda(f_0)\}\le c_0\frac{\big(\mathcal{B}\pdim(\mathcal{F}_n)+\lambda\mathcal{B}^\prime\pdim(\mathcal{F}_n^\prime)\big)\log(n)}{n},
\end{align}
for $n\ge\max\{\pdim(\mathcal{F}_n),\pdim(\mathcal{F}^\prime_n)\}$ and some universal constant $c_0>0$.
By Lemma \ref{lemmapdim}, for the function class $\mathcal{F}_n$ implemented by ReLU-ReQU activated multilayer perceptrons with depth no more than $\mathcal{D}$, width no more than $\mathcal{W}$, number of neurons (nodes) no more than $\mathcal{U}$ and size or number of parameters (weights and bias) no more than $\mathcal{S}$, we have
\begin{align*}
	\pdim(\mathcal{F}_n)\le\min\{7\mathcal{D}\mathcal{S}(\mathcal{D}+\log_2\mathcal{U}),22\mathcal{U}\mathcal{S}\},
\end{align*}
and by Lemma \ref{lemmader}, for any function $f\in\mathcal{F}_n$, its partial derivative $\p f$ can be implemented by a ReLU-ReQU activated multilayer perceptron with depth $3\mathcal{D}+3$, width $10\mathcal{W}$, number of neurons $17\mathcal{U}$, number of parameters $23\mathcal{S}$ and bound $\mathcal{B}^\prime$. Then
\begin{align*}
	\pdim(\mathcal{F}^\prime_n)\le\min\{5796\mathcal{D}\mathcal{S}(\mathcal{D}+\log_2\mathcal{U}),8602\mathcal{U}\mathcal{S}\}.
\end{align*}
%Then this leads to
It follows that
 \begin{align}\notag
	&\mathbb{E}\{\mathcal{R}^\lambda(\hat{f}^\lambda_n)-2\mathcal{R}^\lambda_n(\hat{f}^\lambda_n)+\mathcal{R}^\lambda(f_0)\}\le c_1\big(\mathcal{B}+\lambda\mathcal{B}^\prime\big)\frac{\min\{5796\mathcal{D}\mathcal{S}(\mathcal{D}+\log_2\mathcal{U}),8602\mathcal{U}\mathcal{S}\}\log(n)}{n},
\end{align}
for $n\ge\max\{\pdim(\mathcal{F}_n),\pdim(\mathcal{F}^\prime_n)\}$ and some universal constant $c_1>0$. This completes the proof. $\hfill \Box$
}

\subsection*{Proof of Theorem \ref{approx}}
The idea of the proof is to construct a ReQU network that computes a multivariate polynomial with degree $N$ with no error. We begin our proof with consider the simple case, which is to construct a proper ReQU network to represent a univariate polynomial with no error. Recall that to approximate the multiplication operator is simple and straightforward, we can leverage Horner’s method or Qin Jiushao's algorithm in China to construct such networks. Suppose $f(x)=a_0+a_1x+\cdots+a_Nx^N$ is a univariate polynomial of degree $N$, then it can be written as
$$f(x)=a_0+x(a_1+x(a_2+x(a_3+\cdots+x(a_{N-1}+xa_N)))).$$
We can illiterately calculate a sequence of  intermediate variables $b_1,\ldots,b_N$ by
\begin{align*}
	b_k=\Big\{\begin{array}{lr}
		a_{N-1}+xa_N, \qquad k=1,\\
		a_{N-k}+xb_{N-1}, \ \ k=2,\ldots,N.\\
	\end{array}\Big.
\end{align*}
Then we can obtain $b_N=f(x)$. By the basic approximation property we know that to calculate $b_1$ needs a ReQU network with 1 hidden layer and 4 hidden neurons, and to calculate $b_2$ needs a ReQU network with 3 hidden layer, $2\times4+2-1$ hidden neurons. By induction, to calculate $b_N=f(x)$ needs a ReQU network with $2N-1$ hidden layer, $N\times4+N-1=5N-1$ hidden neurons, $8N$ parameters(weights and bias), and its width equals to 4.

Apart from the construction based on the Horner's method, another construction is shown in Theorem 2.2 of \citet{li2019better}, where the constructed ReQU network has $\lfloor\log_2N\rfloor+1$ hidden layers, $8N-2$ neurons and no more than $61N$ parameters (weights and bias).

Now we consider constructing ReQU networks to compute multivariate polynomial $f$ with total degree $N$ on $\mathbb{R}^d$. For any $d\in\mathbb{N}^+$ and $N\in\mathbb{N}_0$, let
$$f^d_N(x_1,\ldots,x_d)=\sum_{i_1+\cdots+i_d=0}^N a_{i_1,i_2,\ldots,i_d}x_1^{i_1}x_2^{i_2}\cdots x_d^{i_d},$$
denote the polynomial with total degree $N$ of $d$ variables, where $i_1,i_2,\ldots,i_d$ are non-negative integers, $\{a_{i_1,i_2,\ldots,i_d}: i_1+\cdots+i_d\le N\}$ are coefficients in $\mathbb{R}$. Note that the multivariate polynomial $f^d_N$ can be written as
\begin{align*}
	f^d_N(x_1,\ldots,x_d)=\sum_{i_1=0}^N\Big(\sum_{i_2+\cdots+i_d=0}^{N-i_1} a_{i_1,i_2\ldots,i_d}x_2^{i_2}\cdots x_d^{i_d}\Big)x_1^{i_1},
\end{align*}
and we can view $f^d_N$ as a univariate polynomial of $x_1$ with degree $N$ if $x_2,\ldots,x_d$ are given and for each $i_1\in\{0,\ldots,N\}$ the $(d-1)$-variate polynomial $\sum_{i_2+\cdots+i_d=0}^{N-i_1} a_{i_1,i_2\ldots,i_d}x_2^{i_2}\cdots x_d^{i_d}$ with degree no more than $N$ can be computed by a proper ReQU network.
This reminds us the construction of ReQU network for $f^d_N$ can be implemented iteratively via composition of $f^{1}_N,f^2_N,\ldots,f^{d}_N$ by induction.

By Horner's method we have constructed a ReQU network with $2N-1$ hidden layers, $5N-1$ hidden neurons and $8N$ parameters to exactly compute $f^1_N$. Now we start to show $f^2_N$ can be computed by proper ReQU networks. We can write $f^2_N$ as
$$f^2_N(x_1,x_2)=\sum_{i+j=0}^Na_{ij}x_1^ix_2^j=\sum_{i=0}^N\Big(\sum_{j=0}^{N-i}a_{ij}x_2^j\Big)x_1^i.$$
Note that for $i\in\{0,\ldots,N\}$, the the degree of polynomial $\sum_{j=0}^{N-i}a_{ij}x_2^j$ is $N-i$ which is less than $N$. But we can still view it as a polynomial with degree $N$ by padding (adding zero terms) such that $\sum_{j=0}^{N-i}a_{ij}x_2^j=\sum_{j=0}^{N}a^*_{ij}x_2^j$ where $a^*_{ij}=a_{ij}$ if $i+j\le N$ and $a^*_{ij}=0$ if $i+j> N$. In such a way, for each $i\in\{0,\ldots,N\}$ the polynomial $\sum_{j=0}^{N-i}a_{ij}x_2^j$ can be computed by a ReQU network with $2N-1$ hidden layers, $5N-1$ hidden neurons, $8N$ parameters and its width equal to 4. Besides, for each $i\in\{0,\ldots,N\}$, the monomial $x^i$ can also be computed by a ReQU network with $2N-1$ hidden layers, $5N-1$ hidden neurons, $8N$ parameters and its width equal to 4, in whose implementation the identity maps are used after the $(2i-1)$-th hidden layer. Now we parallel these two sub networks to get a ReQU network which takes $x_1$ and $x_2$ as input and outputs $(\sum_{j=0}^{N-i}a_{ij}x_2^j)x^i$ with width $8$, hidden layers $2N-1$, number of neurons $2\times(5N-1)$ and size $2\times8N$. Since for each $i\in\{0,\ldots,N\}$, such paralleled ReQU network can be constructed, then with straightforward paralleling of $N$ such ReQU networks, we obtain a ReQU network exactly computes $f^2_N$ with width $8N$, hidden layers $2N-1$, number of neurons $2\times(5N-1)\times N$ and number of parameters $2\times8N\times N=16N^2$.

Similarly for polynomial $f^3_N$ of $3$ variables, we can write $f^3_N$ as
$$f^3_N(x_1,x_2,x_3)=\sum_{i+j+k=0}^Na_{ijk}x_1^ix_2^jx_3^k=\sum_{i=0}^N\Big(\sum_{j+k=0}^{N-i}a_{ijk}x_2^jx_3^k\Big)x_1^i.$$
By our previous argument, for each $i\in\{0,\ldots,N\}$, there exists a ReQU network which takes $(x_1,x_2,x_3)$ as input and outputs $\Big(\sum_{j+k=0}^{N-i}a_{ijk}x_2^jx_3^k\Big)x_1^i$ with width $8N+4$, hidden layers $2N-1$, number of neurons $2N(5N-1)+(5N-1)$ and parameters $16N^2+8N$. And by paralleling $N$ such subnetworks, we obtain a ReQU network that exactly computes $f^3_N$ with width $(8N+4)\times N=8N^2+4N$, hidden layers $2N-1$, number. of neurons $N(2N(5N-1)+(5N-1))=2N^2(5N-1)+N(5N-1)$ and number of parameters $16N^3+8N^2$.

Continuing this process, we can construct ReQU networks exactly compute polynomials of any $d$ variables with total degree $N$. With a little bit abuse of notations, we let $\mathcal{W}_k$, $\mathcal{D}_k$, $\mathcal{U}_k$ and $\mathcal{S}_k$ denote the width, number of hidden layers, number of neurons and number of parameters (weights and bias) respectively of the ReQU network computing $f^k_N$ for $k=1,2,3,\ldots$. We have known that
\begin{align*}
	\mathcal{D}_1=2N-1\qquad\mathcal{W}_1=4\qquad\mathcal{U}_1=5N-1\qquad\mathcal{S}_1=8N
\end{align*}
Besides, based on the iterate procedure of the network construction, by induction we can see that for $k=2,3,4,\ldots$ the following equations hold,
\begin{align*}
\mathcal{D}_k=&2N-1,\\
\mathcal{W}_k=&N\times(\mathcal{W}_{k-1}+\mathcal{W}_1),\\
\mathcal{U}_k=&N\times(\mathcal{U}_{k-1}+\mathcal{U}_1),\\
\mathcal{S}_k=&N\times(\mathcal{S}_{k-1}+\mathcal{S}_1).
\end{align*}
Then based on the values of $\mathcal{D}_1,\mathcal{W}_1,\mathcal{U}_1,\mathcal{S}_1$ and the recursion formula, we have for $k=2,3,4,\ldots$
\begin{gather*}
	\mathcal{D}_k=2N-1,\\
	\mathcal{W}_k=8N^{k-1}+4\frac{N^{k-1}-N}{N-1}\le12N^{k-1},\\
	\mathcal{U}_k=N\times(\mathcal{U}_{k-1}+\mathcal{U}_1)=2(5N-1)N^{k-1}+(5N-1)\frac{N^{k-1}-N}{N-1}\le15N^{k},\\
	\mathcal{S}_k=N\times(\mathcal{S}_{k-1}+\mathcal{S}_1)=16N^{k}+8N\frac{N^{k-1}-N}{N-1}\le24N^k.
\end{gather*}
This completes our proof. $\hfill \Box$

\subsection*{Proof of Theorem \ref{approx2}}
The proof is straightforward by and leveraging the approximation power of multivariate polynomials since Theorem \ref{approx} told us any multivariate polynomial can be represented by proper ReQU networks. The theories for polynomial approximation have been extensively studies on various spaces of smooth functions. We refer to \citet{bagby2002multivariate} for the polynomial approximation on smooth functions in our proof.
\begin{lemma}[Theorem 2 in \citet{bagby2002multivariate}]\label{approx_poly}
	Let $f$ be a function of compact support on $\mathbb{R}^d$ of class $C^s$ where $s\in\mathbb{N}^+$ and let $K$ be a compact subset of $\mathbb{R}^d$ which contains the support of $f$. Then for each nonnegative integer $N$ there is a polynomial $p_N$ of degree at most $N$ on $\mathbb{R}^d$ with the following property: for each multi-index $\alpha$ with $\vert\alpha\vert_1\le\min\{s,N\}$ we have
	$$\sup_{K}\vert D^\alpha (f-p_N)\vert\le\frac{C}{N^{s-\vert\alpha\vert_1}}\sum_{\vert\alpha\vert_1\le s}\sup_K\vert D^\alpha f\vert,$$
	where $C$ is a positive constant depending only on $d,s$ and $K$.
\end{lemma}
The proof of Lemma \ref{approx_poly} can be found in \citet{bagby2002multivariate} based on the Whitney extension
theorem (Theorem 2.3.6 in \citet{hormander2015analysis}). To use Lemma \ref{approx_poly}, we need to find a ReQU network to compute the $p_N$ for each $N\in\mathbb{N}^+$. By Theorem \ref{approx}, we know that any $p_N$ of $d+1$ variables can be computed by a ReQU network with $2N-1$ hidden layer, no more than $15N^{d+1}$ neurons, no more than $24N^{d+1}$ parameters and width no more than $12N^{d}$. This completes the proof. By examining the proof of Theorem 1 in \citet{bagby2002multivariate}, the dependence of the constant $C$ in Lemma \ref{approx_poly} on the $d,s$ and $K$ can be detailed. $\hfill \Box$

\subsection*{Proof of Corollary \ref{cor1}}
Recall that
\begin{align*}
	&\inf_{f\in\mathcal{F}_n}\Big[\mathcal{R}(f)-\mathcal{R}(f_0)+\lambda\{\kappa(f)-\kappa(f_0)\}\Big]\\
	=&\inf_{f\in\mathcal{F}_n}\Bigg[\mathbb{E}_{X,Y,\xi}\big\{\rho_{\xi}(Y-f(X,\xi))-\rho_{\xi}(Y-f_0(X,\xi))\\
	&\qquad\qquad+\lambda(\max\{-\p f(X,\xi),0\}-\max\{-\p f_0(X,\xi),0\})\big\}\Bigg]\\
	\le&\inf_{f\in\mathcal{F}_n}\Big[\mathbb{E}_{X,Y,\xi}\Big\{\vert f(X,\xi)-f_0(X,\xi)\vert+\lambda\vert \p f(X,\xi)-\p f_0(X,\xi)\vert\Big\}\Big].
\end{align*}
By Theorem \ref{approx2}, for each $N\in\mathbb{N}^+$, there exists a ReQU network $\phi_N\in\mathcal{F}_n$ with $2N-1$ hidden layer, no more than $15N^{d+1}$ neurons, no more than $24N^{d+1}$ parameters and width no more than $12N^{d}$ such that for each multi-index $\alpha\in\mathbb{N}^d_0$ with $\vert\alpha\vert_1\le\min\{s,N\}$ we have
$$\sup_{\mathcal{X}\times(0,1)}\vert D^\alpha (f-\phi_N)\vert\le C(s,d,\mathcal{X})\times N^{-(s-\vert\alpha\vert_1)}\Vert f\Vert_{C^s},$$
where $C(s,d,\mathcal{X})$  is a positive constant depending only on $d,s$ and the diameter of $\mathcal{X}\times(0,1)$. This implies
$$\sup_{\mathcal{X}\times(0,1)}\vert  f-\phi_N\vert\le C(s,d,\mathcal{X})\times N^{-s}\Vert f\Vert_{C^s},$$
and
$$\sup_{\mathcal{X}\times(0,1)}\Big\vert \p (f-\phi_N)\Big\vert\le C(s,d,\mathcal{X})\times N^{-(s-1)}\Vert f\Vert_{C^s}.$$
Combine above two uniform bounds, we have
\begin{align*}
	&\inf_{f\in\mathcal{F}_n}\Big[\mathcal{R}(f)-\mathcal{R}(f_0)+\lambda\{\kappa(f)-\kappa(f_0)\}\Big]\\
	\le&\Big[\vert \mathbb{E}_{X,\xi}\Big\{\phi_N(X,\xi)-f_0(X,\xi)\vert+\lambda\vert \p \phi_N(X,\xi)-\p f_0(X,\xi)\vert\Big\}\Big]\\
	\le&C(s,d,\mathcal{X})\times N^{-s}\Vert f\Vert_{C^s}+\lambda C(s,d,\mathcal{X})\times N^{-(s-1)}\Vert f\Vert_{C^s}\\
		\le& C(s,d,\mathcal{X}) (1+\lambda) N^{-(s-1)}\Vert f\Vert_{C^s},
\end{align*}
which completes the proof. $\hfill \Box$

\subsection*{Proof of Lemma \ref{calib}}
%\begin{proof}
	By equation (B.3) in \cite{belloni2011}, for any scalar $w,v\in\mathbb{R}$ and $\tau\in(0,1)$ we have
	\begin{align*}
		\rho_\tau(w-v)-\rho_\tau(w)=-v\{\tau-I(w\leq0)\}+\int_0^v\{I(w\leq z)-I(w\leq0)\}dz.
	\end{align*}
	Given any quantile level $\tau\in(0,1)$, function $f$ and $X=x$, let $w=Y-f_0(X,\tau)$, $v=f(X,\tau)-f_0(X,\tau)$. Suppose $\vert f(x,\tau)-f_0(x,\tau)\vert\leq K$, taking conditional expectation on above equation with respect to $Y\mid X=x$, we have
	\begin{align*}
		&\mathbb{E}\{\rho_\tau(Y-f(X,\tau))-\rho_\tau(Y-f_0(X,\tau))\mid X=x\}\\
		=&\mathbb{E}\big[-\{f(X,\tau)-f_0(X,\tau)\}\{\tau-I(Y-f_0(X,\tau)\leq0)\}\mid X=x\big]\\
		&+\mathbb{E}\big[\int_0^{f(X,\tau)-f_0(X,\tau)}\{I(Y-f_0(X,\tau)\leq z)-I(Y-f_0(X,\tau)\leq0)\}dz\mid X=x\big]\\
		=&0+\mathbb{E}\big[\int_0^{f(X,\tau)-f_0(X,\tau)}\{I(Y-f_0(X,\tau)\leq z)-I(Y-f_0(X,\tau)\leq0)\}dz\mid X=x\big]\\
		=&\int_0^{f(x,\tau)-f_0(x,\tau)} \{P_{Y|X}(f_0(x,\tau)+z)-P_{Y|X}(f_0(x,\tau))\}dz\\
		\geq&\int_0^{f(x,\tau)-f_0(x,\tau)} k \vert z\vert dz\\
		=&\frac{k}{2}\vert f(x,\tau)-f_0(x,\tau)\vert^2.
	\end{align*}
	Suppose $f(x)-f_0(x)> K$, then similarly we have
	\begin{align*}
		&\mathbb{E}\{\rho_\tau(Y-f(X,\tau))-\rho_\tau(Y-f_0(X,\tau))\mid X=x\}\\
		=&\int_0^{f(x,\tau)-f_0(x,\tau)} \{P_{Y|X}(f_0(x,\tau)+z)-P_{Y|X}(f_0(x,\tau))\}dz\\
		\geq&\int_{K/2}^{ f(x,\tau)-f_0(x,\tau)} \{P_{Y|X}(f_0(x,\tau)+K/2)-P_{Y|X}(f_0(x,\tau))\}dz\\
		\geq&(f(x,\tau)-f_0(x,\tau)-K/2)(kK/2)\\
		\geq &\frac{kK}{4}\vert f(x,\tau)-f_0(x,\tau)\vert.
	\end{align*}
	
	The case $f(x,\tau)-f_0(x,\tau)\leq -K$ can be handled similarly as in \cite{padilla2021adaptive}. The conclusion follows combining the three different cases and taking expectation with respect to $X$ of above obtained inequality. Finally for any function $f:\mathcal{X}\times(0,1)\to\mathbb{R}$, we have
	\begin{align*}
	\Delta^2(f,f_0)=&\mathbb{E}\min\{\vert f(X,\xi)-f_0(X,\xi)\vert,\vert f(X,\xi)-f_0(X,\xi)\vert^2\}\\
	\le& \max\{2/k,4/(kK)\}\mathbb{E} \Big[\int_{0}^{1}\rho_\tau(Y-f(X,\tau))-\rho_\tau(Y-f_0(X,\tau))d\tau\Big]\\
	=& \max\{2/k,4/(kK)\} [\mathcal{R}(f)-\mathcal{R}(f_0)].
	\end{align*}
This completes the proof. $\hfill \Box$
\section{Definitions and Supporting Lemmas}\label{apdx_b}
\subsection{Definitions}
%We give the definition of% Rademacher Complexity and  covering number and others
The following definitions are used in the proofs.
\begin{comment}
\begin{definition}[Rademacher Complexity]
	Let $\mathcal{F}$ be a family of functions mapping from $\mathcal{Z}$ to $\mathcal{R}$ and $S=(z_1,\ldots,z_n)$ a fixed sample of size $n$ with elements in $\mathcal{Z}$. Then, the empirical Rademacher complexity of $\mathcal{F}$ with respect to the sample $S$ is defined as:
	\begin{align*}
		\hat{\mathfrak{R}}_S(\mathcal{F})=\mathbb{E}_{\bf\sigma}\Big[\sup_{f\in\mathcal{F}}\Big\vert\frac{1}{n}\sum_{i=1}^n\sigma_if(z_i)\Big\vert\Big],
	\end{align*}
	where $\sigma=(\sigma_1,\ldots,\sigma_n)^\top$ with $\sigma_i$ independent uniform random variables taking values in $\{-1,+1\}$. The Rademacher complexity of $\mathcal{F}$ is defined as:
	\begin{align*}
		{\mathfrak{R}}_n(\mathcal{F})=\mathbb{E}_S\Big[\hat{\mathfrak{R}}_S(\mathcal{F})\Big],
	\end{align*}
	where the expectation is taken over all samples of size $n$ drawn according to the distribution of the sample.
\end{definition}
\end{comment}

\begin{definition}[Covering number]
	Let $\mathcal{F}$ be a class of function from $\mathcal{X}$ to $\mathbb{R}$. For a given sequence $x=(x_1,\ldots,x_n)\in\mathcal{X}^n,$
	%\subseteq\mathbb{R}^n$,
	let  $\mathcal{F}_n|_x=\{(f(x_1),\ldots,f(x_n):f\in\mathcal{F}_n\}$ be the  subset of $\mathbb{R}^{n}$. For a positive number $\delta$, let $\mathcal{N}(\delta,\mathcal{F}_n|_x,\Vert\cdot\Vert_\infty)$ be the covering number of $\mathcal{F}_n|_x$ under the norm $\Vert\cdot\Vert_\infty$ with radius $\delta$.
	Define the uniform covering number
	$\mathcal{N}_n(\delta,\Vert\cdot\Vert_\infty,\mathcal{F}_n)$ to be the maximum
	over all $x\in\mathcal{X}$ of the covering number $\mathcal{N}(\delta,\mathcal{F}_n|_x,\Vert\cdot\Vert_\infty)$, i.e.,
	\begin{equation}
		\label{ucover}
		\mathcal{N}_n(\delta,\mathcal{F}_n,\Vert\cdot\Vert_\infty)=
		\max\{\mathcal{N}(\delta,\mathcal{F}_n|_x,\Vert\cdot\Vert_\infty):x\in\mathcal{X}^n\}.
	\end{equation}
\end{definition}

\begin{definition}[Shattering]
 Let $\mathcal{F}$ be a family of functions from a set $\mathcal{Z}$ to $\mathbb{R}$. A set $\{z_1,\ldots,Z_n\}\subset\mathcal{Z}$ is said to be shattered by $\mathcal{F}$, if there exists $t_1,\ldots,t_n\in\mathbb{R}$ such that
 \begin{align*}
 	\Big\vert\Big\{\Big[
 	\begin{array}{lr}
 		{\rm sgn}(f(z_1)-t_1)\\
 		\ldots\\
 		{\rm sgn}(f(z_n)-t_n)\\
 	\end{array}\Big]:f\in\mathcal{F}
 \Big\}\Big\vert=2^n,
 \end{align*}
where ${rm sgn}$ is the sign function returns $+1$ or $-1$ and $\vert\cdot\vert$ denotes the cardinality of a set. When they exist, the threshold values $t_1,\ldots,t_n$ are said to witness the shattering.
\end{definition}

\begin{definition}[Pseudo dimension]
	Let $\mathcal{F}$ be a family of functions mapping from $\mathcal{Z}$ to $\mathbb{R}$. Then, the pseudo dimension of $\mathcal{F}$, denoted by $\pdim(\mathcal{F})$, is the size of the largest set shattered by $\mathcal{F}$.
\end{definition}

\begin{definition}[VC dimension]
	Let $\mathcal{F}$ be a family of functions mapping from $\mathcal{Z}$ to $\mathbb{R}$. Then, the Vapnik–Chervonenkis (VC) dimension of $\mathcal{F}$, denoted by $\vdim(\mathcal{F})$, is the size of the largest set shattered by $\mathcal{F}$ with all threshold values being zero, i.e., $t_1=\ldots,=t_n=0$.
\end{definition}

\subsection{Supporting Lemmas}
%We give supporting lemmas used in the proof.
The following lemma gives an upper bound for the covering number in terms of the pseudo-dimension.

\begin{lemma}[Theorem 12.2 in \citet{anthony1999}]\label{c2p}
	Let $\mathcal{F}$ be a set of real functions from domain $\mathcal{Z}$ to the bounded interval $[0,B]$. Let $\delta>0$ and suppose that $\mathcal{F}$ has finite pseudo-dimension $\pdim(\mathcal{F})$ then
	\begin{align*}
		\mathcal{N}_n(\delta,\mathcal{F},\Vert\cdot\Vert_\infty)\le\sum_{i=1}^{\pdim(\mathcal{F})}\binom{n}{i}\Big(\frac{B}{\delta}\Big)^i,
	\end{align*}
which is less than $\{enB/(\delta\pdim(\mathcal{F}))\}^{\pdim(\mathcal{F})}$ for $n\ge\pdim(\mathcal{F})$.
\end{lemma}

%\begin{lemma}[Theorem 2.1 in \cite{li2019better}]\label{monomial}
%	 The monomial $x^n,n\in\mathbb{N}$ defined on $\mathbb{R}$ can be represented exactly by a $\sigma_2$ network.  The network realizing $x^n$ has layer at most $\lfloor\log_2n\rfloor+2$, width at most 6, hidden nodes at most $5\lfloor\log_2n\rfloor+5$ and number of nonzero weights at most  $25\lfloor\log_2n\rfloor+14$ respectively. Here $\lfloor a\rfloor $ denotes the  largest integer not exceeding $a$ for $a\in\mathbb{R}$. Besides, For any $n>2$, $x^n$ can not be represented exactly by any ReQU network with less than $\lceil\log_2n\rceil$ hidden layers.
%\end{lemma}

\section{Additional simulation results}
In this section, we include additional simulation results for the ``Linear" model.

\begin{table}[H]
	\caption{Data is generated from the ``Linear" model with training sample size $n= 512,1024$ and the number of replications $R = 100$. The averaged  $L_1$ and $L_2^2$ test errors with the corresponding standard deviation (in parentheses) are reported for the estimators trained by different methods.}
	\label{tab:1}
%	\resizebox{\textwidth}{!}{%
		\begin{tabular}{@{}c|c|cc|cc@{}}
			\toprule
			& Sample size &  \multicolumn{2}{c|}{$n=512$}   & \multicolumn{2}{c}{$n=1024$}  \\ \midrule
			$\tau$                & Method    &  $L_1$         & $L_2^2$       & $L_1$         & $L_2^2$       \\ \midrule
			\multirow{3}{*}{0.05} & DQRP   & 0.395(﻿0.219) & 0.283(﻿0.324) & 0.282(﻿0.165) & 0.138(﻿0.155) \\
			& Kernel QR  & \textbf{0.277(﻿0.213)} & \textbf{0.143(﻿0.227)} & \textbf{0.149(﻿0.167)} & \textbf{0.056(﻿0.172)} \\
			& QR Forest &  0.740(﻿0.176) & 1.479(﻿1.828) & 0.715(﻿0.066) & 1.250(﻿1.048) \\ \midrule
			\multirow{3}{*}{0.25} & DQRP  & \textbf{0.126(0.051)}  & \textbf{0.027(0.022)}  & 0.084(0.037)  & 0.012(0.009)  \\
			& Kernel QR  & 0.149(0.058)  & 0.037(0.026)  & \textbf{0.077(0.030)}  & \textbf{0.010(0.008)}  \\
			& QR Forest  & 0.278(0.049)  & 0.125(0.043)  & 0.279(0.021)  & 0.127(0.019)  \\ \midrule
			\multirow{3}{*}{0.5}  & DQRP    & \textbf{0.118(0.054)}  & \textbf{0.023(0.021)}  & 0.084(0.043)  & 0.012(0.012)  \\
			& Kernel QR  & 0.131(0.042)  & 0.029(0.018)  & \textbf{0.065(0.021)}  & \textbf{0.007(0.004)}  \\
			& QR Forest  & 0.232(0.034)  & 0.087(0.025)  & 0.230(0.016)  & 0.085(0.011)  \\ \midrule
			\multirow{3}{*}{0.75} & DQRP  & 0.179(0.099)  & 0.052(0.057)  & 0.113(0.075)  & 0.023(0.033)  \\
			& Kernel QR   & \textbf{0.146(0.044)}  & \textbf{0.035(0.020)}  & \textbf{0.072(0.027)}  & \textbf{0.009(0.007)}  \\
			& QR Forest   & 0.275(0.043)  & 0.124(0.041)  & 0.279(0.022)  & 0.127(0.020)  \\\midrule
			\multirow{3}{*}{0.95} & DQRP  & 0.386(0.211)  & 0.224(0.219)  & 0.252(0.151)  & 0.101(0.122)  \\
			& Kernel QR & \textbf{0.272(0.200)}  & \textbf{0.134(0.204)}  & \textbf{0.147(0.155)}  & \textbf{0.051(0.149)}  \\
			& QR Forest   & 0.705(0.123)  & 0.997(0.500)  & 0.731(0.083)  & 1.438(1.531)  \\ \bottomrule
		\end{tabular}%
	%}
\end{table}

\begin{table}[H]
	\caption{Data is generated from the multivariate ``Linear" model with training sample size $n= 512,2048$ and the number of replications $R = 100$. The averaged  $L_1$ and $L_2^2$ test errors with the corresponding standard deviation (in parentheses) are reported for the estimators trained by different methods.}
	\label{tab:4}
%		\resizebox{\textwidth}{!}{%
	\begin{tabular}{@{}c|c|cc|cc|cc@{}}
		\toprule
		& Sample size &  \multicolumn{2}{c|}{$n=512$}   & \multicolumn{2}{c}{$n=2048$}   \\ \midrule
		$\tau$                & Method       & $L_1$                  & $L_2^2$                & $L_1$                  & $L_2^2$                \\ \midrule
		\multirow{3}{*}{0.05} & DQRP    & \textbf{0.911(﻿0.184)} & \textbf{1.347(﻿0.233)} & 0.681(﻿0.148)          & \textbf{0.660(﻿0.229)} \\
		& Kernel QR   & 1.095(﻿0.101)          & 1.610(﻿0.251)          & 0.908(﻿0.065)          & 1.019(﻿0.124)          \\
		& QR Forest       & 0.939(0﻿.100)          & 1.355(﻿0.278)          & \textbf{0.648(﻿0.052)} & 0.679(﻿0.112)          \\ \midrule
		\multirow{3}{*}{0.25} & DQRP      & \textbf{0.537(0.077)}  & \textbf{0.551(0.152)}  & \textbf{0.328(0.047)}  & \textbf{0.204(0.053)}  \\
		& Kernel QR   & 0.586(0.044)           & 0.580(0.090)           & 0.367(0.023)           & 0.230(0.028)           \\
		& QR Forest    & 0.739(0.037)           & 0.847(0.081)           & 0.514(0.018)           & 0.413(0.028)           \\ \midrule
		\multirow{3}{*}{0.5}  & DQRP        & \textbf{0.531(0.071)}  & \textbf{0.523(0.137)}  & \textbf{0.309(0.044)}  & \textbf{0.178(0.047)}  \\
		& Kernel QR    & 0.557(0.044)           & 0.534(0.089)           & 0.350(0.022)           & 0.210(0.026)           \\
		& QR Forest         & 0.658(0.031)           & 0.680(0.060)           & 0.458(0.016)           & 0.331(0.022)           \\ \midrule
		\multirow{3}{*}{0.75} & DQRP        & \textbf{0.575(0.105)}  & 0.653(0.244)           & 0.382(0.089)           & 0.262(0.112)           \\
		& Kernel QR     & 0.584(0.047)           & \textbf{0.580(0.095)}  & \textbf{0.366(0.021)}  & \textbf{0.227(0.026)}  \\
		& QR Forest     & 0.741(0.040)           & 0.848(0.088)           & 0.515(0.019)           & 0.414(0.029)           \\ \midrule
		\multirow{3}{*}{0.95} & DQRP       & \textbf{0.971(0.213)}  & \textbf{1.419(0.480)}  & \textbf{0.635(0.163)}  & \textbf{0.596(0.256)}  \\
		& Kernel QR    & 1.088(0.098)           & 1.596(0.246)           & 0.883(0.073)           & 0.975(0.132)           \\
		& QR Forest         & 0.982(0.117)           & 1.441(0.434)           & 0.649(0.049)           & 0.683(0.109)           \\ \bottomrule
	\end{tabular}
%}
\end{table}

\begin{figure}[H]
	\centering
	\includegraphics[width=\textwidth]{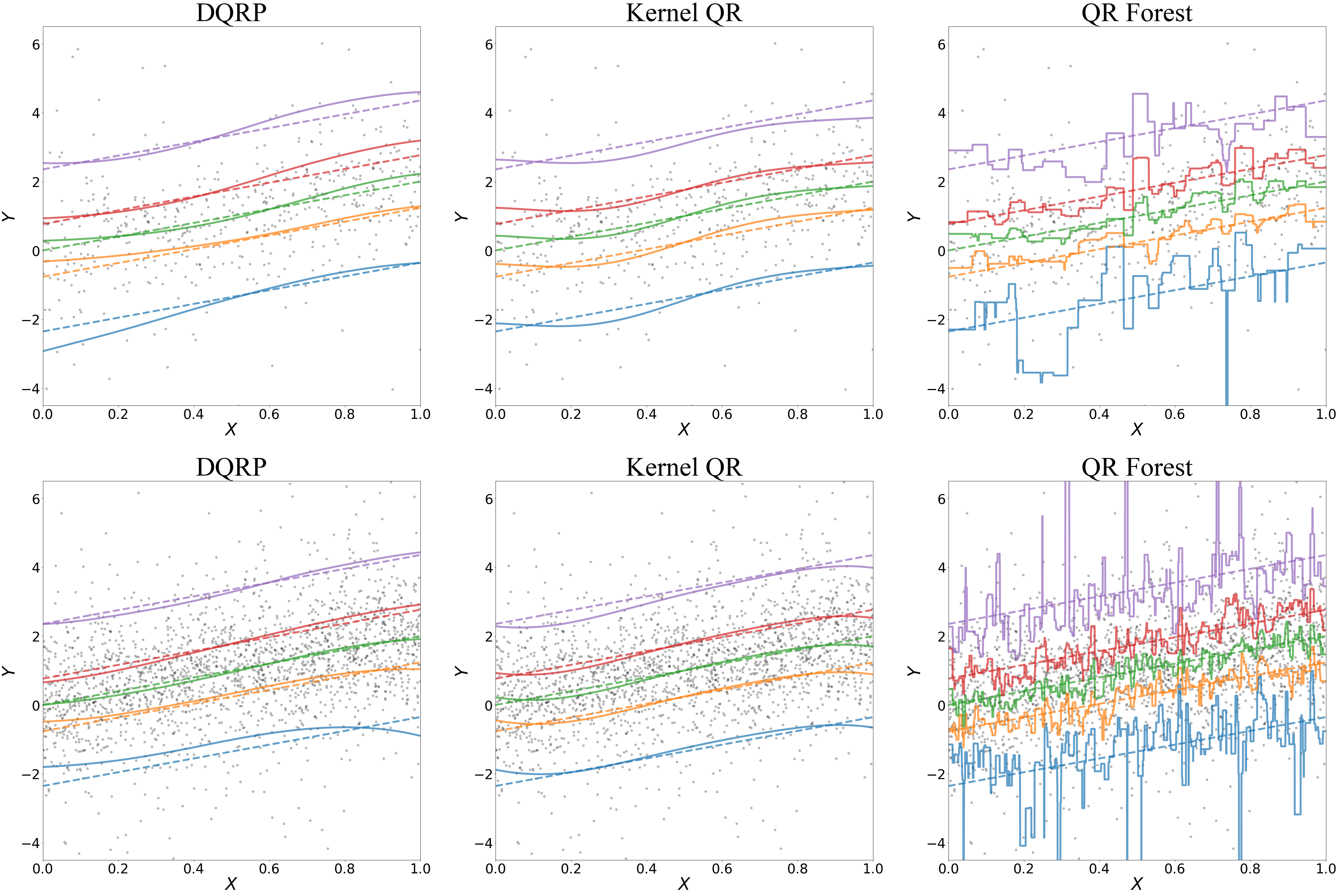}
	\caption{The fitted quantile curves by different methods under the univariate ``Linear" model when $n=512,2048$. The training data is depicted as grey dots.The target quantile functions at the quantile levels $\tau=$0.05 (blue), 0.25 (orange), 0.5 (green), 0.75 (red), 0.95 (purple) are depicted as dashed curves, and the estimated quantile functions are represented by solid curves with the same color. From the top to the bottom, the rows correspond to the sample size $n=512,2048$. From the left to the right, the columns correspond to the methods DQRP, kernel QR and QR Forest.}
	\label{fig:1}
\end{figure}

\begin{figure}[H]
	\centering
	\includegraphics[width=0.315\textwidth]{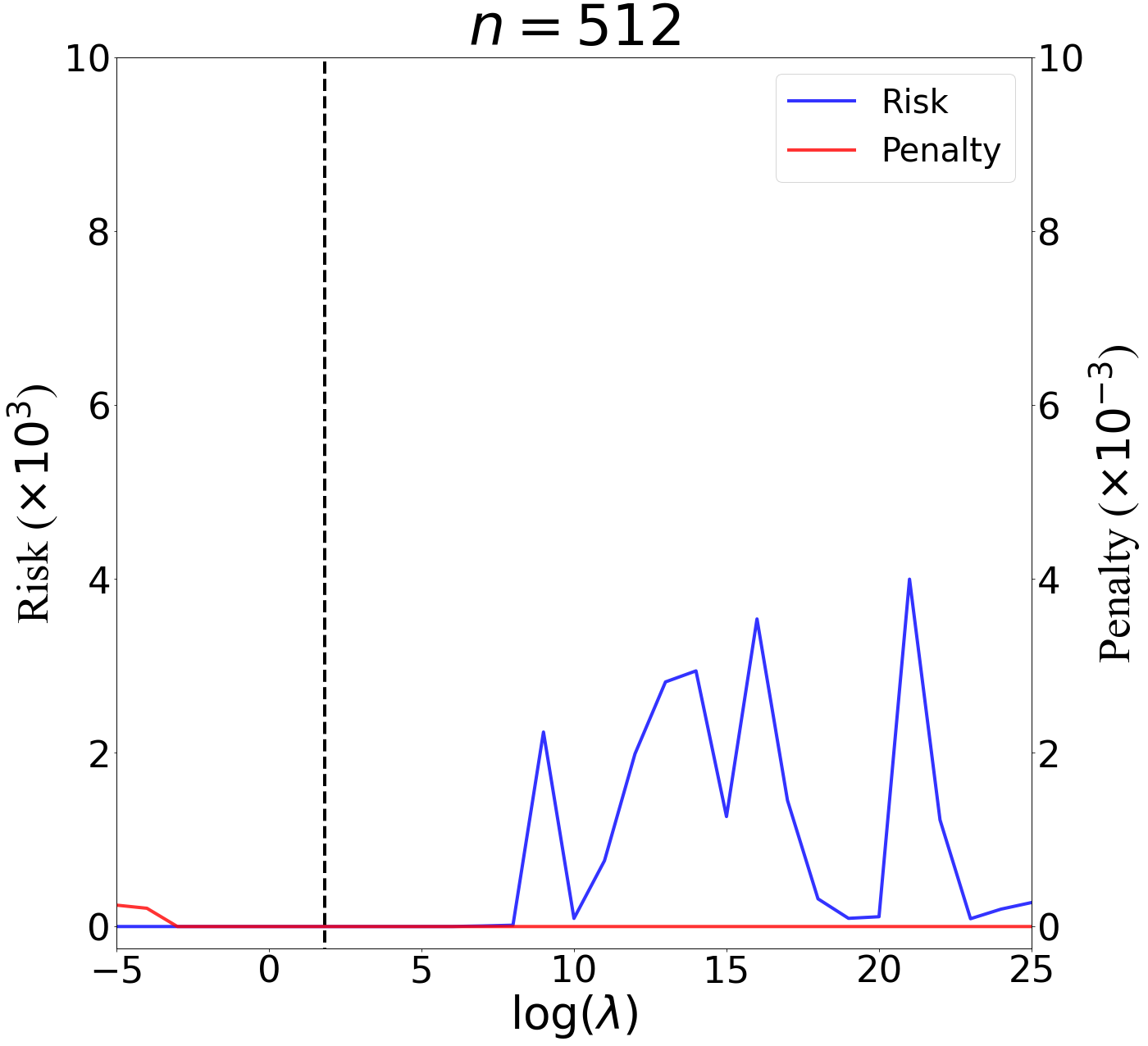}
	\includegraphics[width=0.315\textwidth]{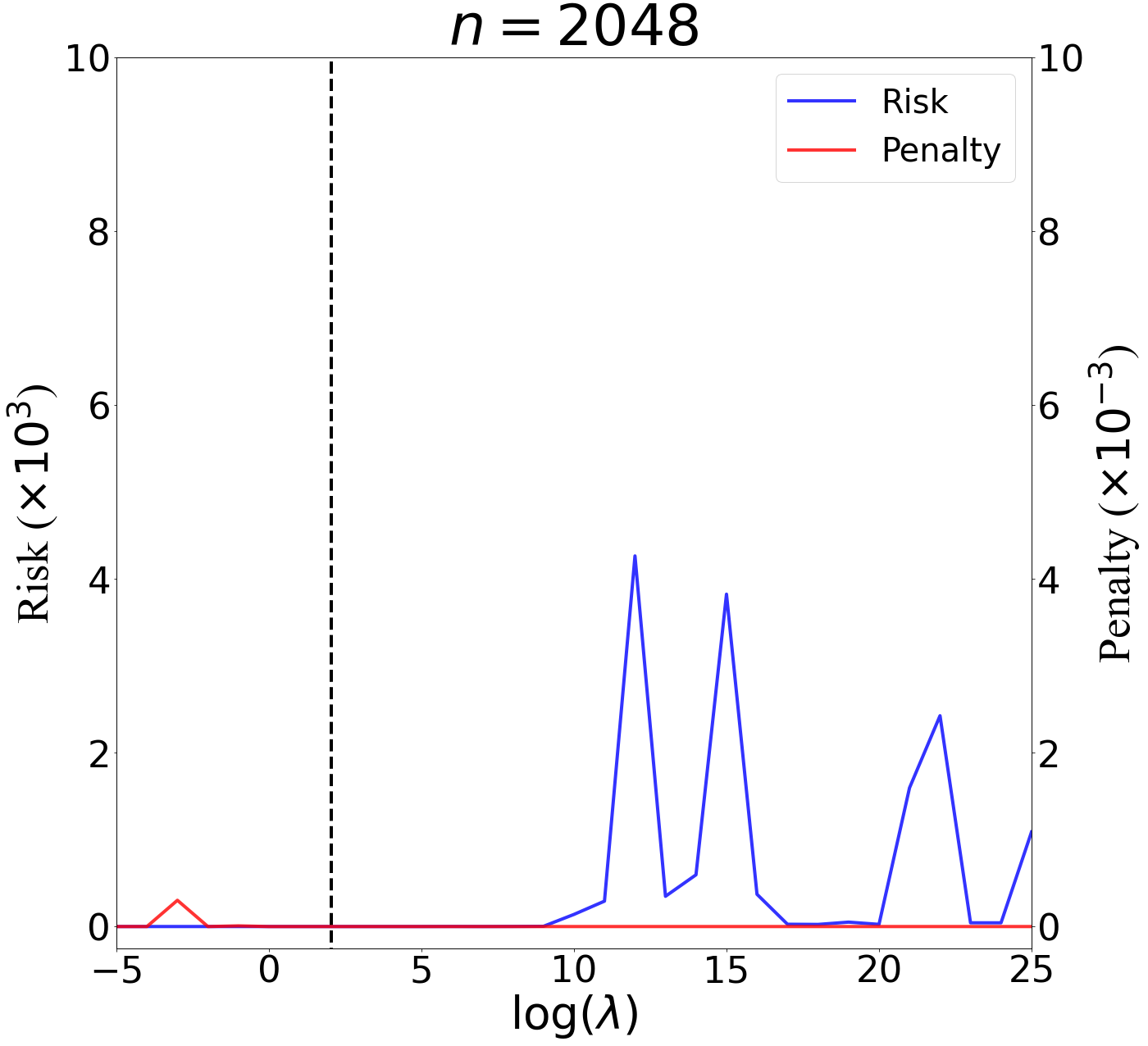}
	\caption{The value of risks and penalties under the univariate ``Wave" model when $n=512,2048$. A vertical dashed line is depicted at the value $\lambda=\log(n)$ on x-axis in each figure.}
	\label{fig:4}
\end{figure}

\begin{figure}[H]
	\centering
	\includegraphics[width=0.315\textwidth]{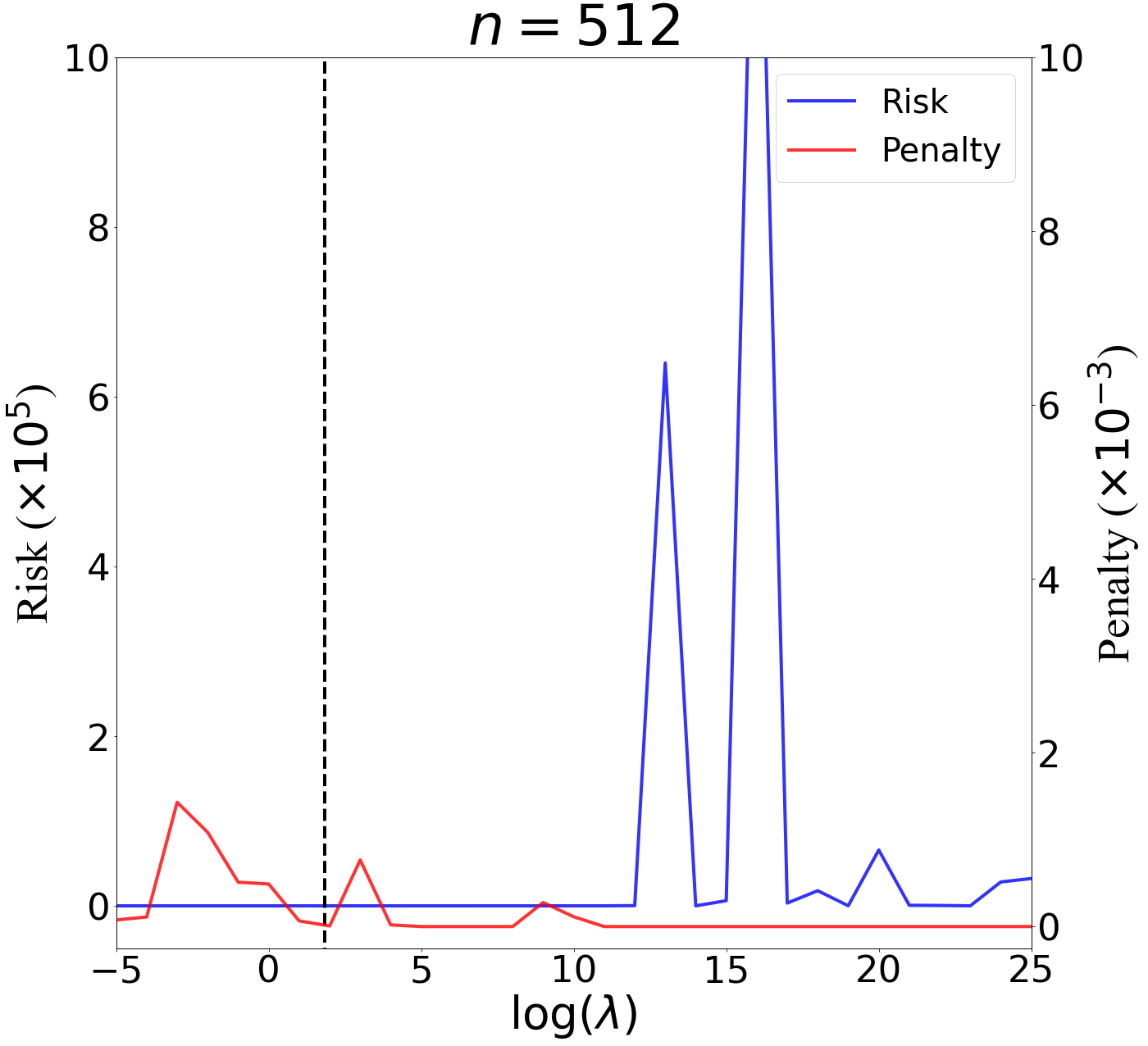}
	\includegraphics[width=0.315\textwidth]{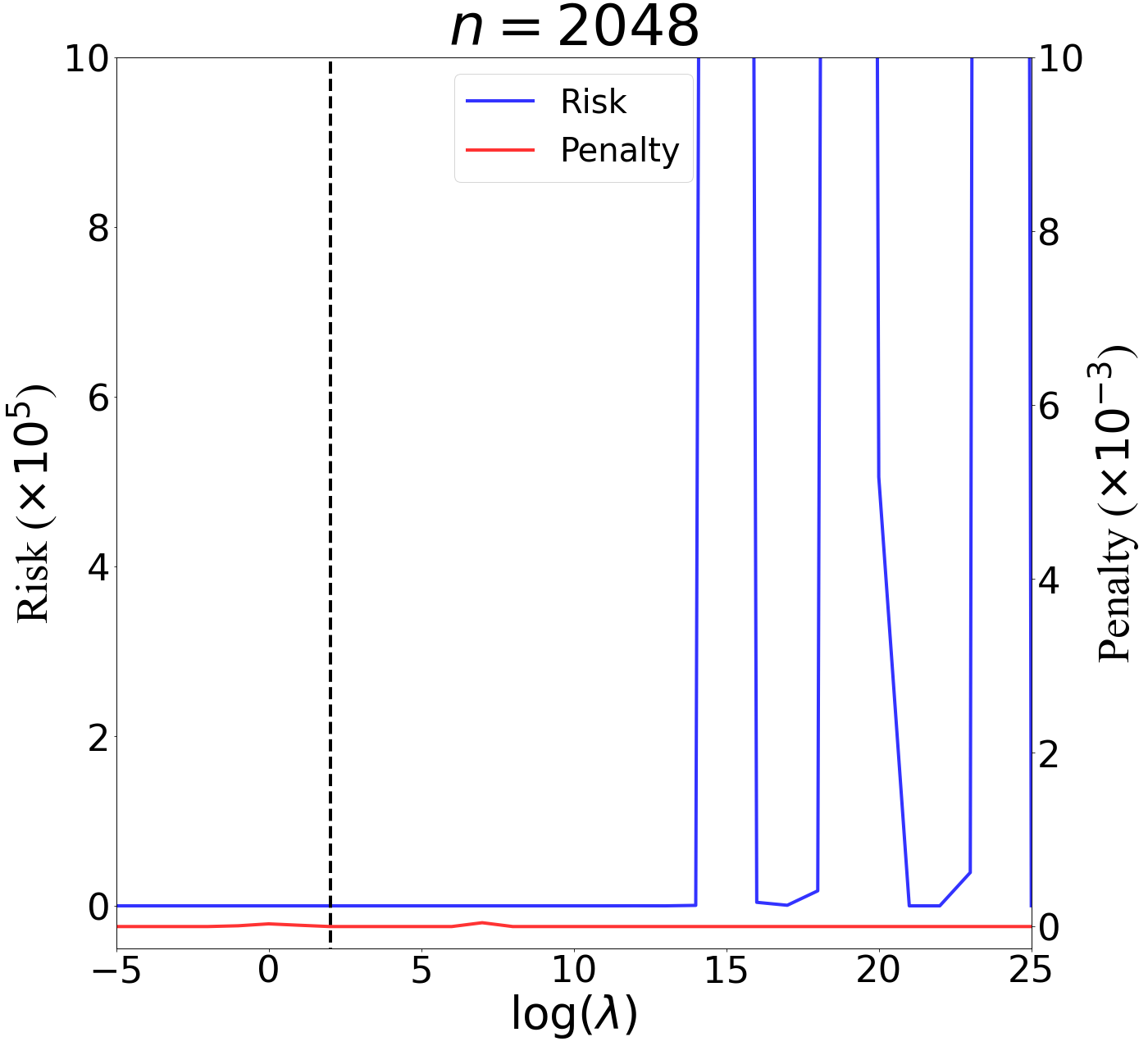}
	\caption{The value of risks and penalties under the multivariate single index model when $n=512,2048$ and $d=8$. A vertical dashed line is depicted at the value $\lambda=\log(n)$ on x-axis in each figure.}
	\label{fig:6}
\end{figure}

\end{document}